\documentclass[11pt]{article}

\date{}

\usepackage[utf8]{inputenc} 
\usepackage[T1]{fontenc}    
\usepackage{hyperref}       
\usepackage{url}            
\usepackage{booktabs}       
\usepackage{amsfonts}       
\usepackage{nicefrac}       
\usepackage{microtype}      
\usepackage{xspace}
\usepackage{natbib}

\usepackage{comment}
\usepackage{amsmath}
\usepackage{amsthm}
\usepackage{graphicx}
\usepackage{rotating}

\usepackage{amssymb}
\usepackage{mathtools}
\usepackage{xcolor}
\usepackage{bbm}
\usepackage{algorithm}
\usepackage{algorithmic}

\makeatletter

\newcommand{\tsig}{\tilde{\Sigma}}
\newcommand{\sig}{\Sigma}
\newcommand{\reg}{\textsc{\small{REGRET}}\xspace}
\newcommand{\OO}{\mathcal{O}}

\newcommand{\LQG}{\textsc{\small{LQG}}\xspace}
\newcommand{\OFU}{\textsc{\small{OFU}}\xspace}

\newcommand{\Texp}{T_{exp}}
\newcommand{\Dexp}{\mathcal{D}_{exp}}

\DeclareMathOperator{\Tr}{Tr}
\newcommand{\ta}{\tilde{A}}
\newcommand{\tb}{\tilde{B}}
\newcommand{\tc}{\tilde{C}}
\newcommand{\tp}{\tilde{P}}
\newcommand{\tl}{\tilde{L}}
\newcommand{\tk}{\tilde{K}}
\newcommand{\tth}{\tilde{\Theta}}

\newcommand{\policy}{\Pi}

\newcommand{\ofu}{\textsc{OFU}\xspace}
\newcommand{\Alg}{\textsc{\small{ExpCommit}}\xspace}

\newcommand{\E}{\mathbb E}

\newcommand{\R}{\mathbb{R}}

\newtheorem{lemma}{Lemma}[section]
\newtheorem{assumption}{Assumption}[section]
\newtheorem{theorem*}{Theorem}[section]
\newtheorem{corollary}{Corollary}[section]

\newtheorem{definition}{Definition}[section]
\newtheorem{theorem}{Theorem}[section]

\title{Regret Minimization in Partially Observable Linear Quadratic Control}

\author{%
  Sahin Lale$^1$, Kamyar Azizzadenesheli$^2$, Babak Hassibi$^1$, Anima Anandkumar$^2$\\
  $^1$~Department of Electrical Engineering\\
  $^2$~Department of Computing and Mathematical Sciences\\
  California Institute of Technology, Pasadena\\
  \texttt{\{alale,kazizzad,hassibi,anima\}@caltech.edu} \\
}

\begin{document}

\maketitle      

\begin{abstract}%
  We study the problem of regret minimization in partially observable linear quadratic control systems when the model dynamics are unknown a priori. We propose \Alg, an explore-then-commit algorithm that learns the model Markov parameters and then follows the principle of optimism in the face of uncertainty to design a controller. We propose a novel way to decompose the regret and provide an end-to-end sublinear regret upper bound for partially observable linear quadratic control. Finally, we provide stability guarantees and establish a regret upper bound of  $\tilde{\mathcal{O}}(T^{2/3})$ for \Alg, where $T$ is the time horizon of the problem.
\end{abstract}


\section{Introduction}\label{Introduction}

Controlling unknown discrete-time systems is a fundamental problem in adaptive control and reinforcement learning. In this problem, an agent interacts with an environment, with unknown dynamics, and aims to minimize the overall average regulating costs. To achieve this goal, the agent is required to explore the environment to gain a better understanding of the environment dynamics, which is often called system identification. The agent then utilizes this understanding to design a set of improved controllers that simultaneously reduces the possible future costs and also enables the agent to explore the important and unknown aspects of the system. In recent decades, this challenging problem has been extensively studied and resulted in a set of foundational steps to study the stability and asymptotic convergence to optimal controllers~\citep{lai1982least,lai1987asymptotically}. While asymptotic analyses set the ground for the design of optimal control, understanding the finite time behavior of adaptive algorithms is critical for real-world applications. 
In practice, one might prefer an algorithm that guarantees better performance on a much shorter horizon. Recent developments in the fields of statistics and machine learning along with control theory \citep{van1996weak,pena2009self,lai1982least} empowers us to not only advance the study of the asymptotic efficiency of algorithms but also to analyze their finite-time behavior~\citep{fiechter1997pac,abbasi2011regret}. 

In partially observable linear quadratic control, if the agent, a priori, is handed the system dynamics, the optimal control/policy has a closed-form in the presence of Gaussian disturbances. This control system is known as Linear–Quadratic–Gaussian (\LQG) control. In this case, the control law is obtained through estimating the latent state of the system and deploying optimal controller synthesis. This is known as the separation principle in control theory~\citep{aastrom2012introduction}. 

In partially observable environments, the agent is required to estimate the latent state of the system, which is not needed in fully observed environments where the precise state information is available. This seemingly subtle difference introduces significant challenges in controlling and learning partially observable systems when the underlying dynamics are unknown. Moreover, in the presence of an approximated model, the uncertainties in the model estimation cause an inaccurate latent state estimation. Generally, this error in the latent state estimation accumulates over time and results in sub-optimal controller synthesis, which in turn imposes a sub-optimal policy. In this work, while we study model estimation/system identification, we particularly focus on the latter challenge, \textit{i.e.}, to analyze the mentioned sub-optimality gap.

One of the general and prominent measures to assess the behavior of a learning algorithm/agent in a control system is its regret, which quantifies the additional cost encountered by the agent due to exploration and uncertainty in the model estimation, when compared with the performance of the optimal controller~\citep{lai1985asymptotically}. We aim to design an algorithm that minimizes regret and is accompanied by a sub-linear regret upper bound guarantee.

\textbf{Contributions:} In this paper, we propose \Alg, an \textit{explore-then-commit} algorithm for learning and controlling unknown discrete-time \LQG control system.  We  prove that \Alg achieves a regret upper bound of $\tilde{\mathcal{O}}({T^{2/3}})$, where $T$ is the number of total interactions. This is the first sub-linear regret bound for partially observable linear quadratic control systems.

\Alg consists of two phases, an \textit{explore} phase, and a \textit{commit} phase. During the \textit{explore} phase, the agent employs Gaussian noise as its control input, explores the system, and constructs high probability confidence sets around the estimate of the system parameters. In the \textit{commit} phase, the agent exploits these confidence sets, and designs a controller based on the principle of optimism in the face of uncertainty (\OFU). Betting on the best model, \textit{i.e.}, choosing the optimistic model, for the commit phase roots back to the early appearance of \OFU principle that is formulated as the bet on the best (BOB) principle \citep{bittanti2006adaptive}. To analyze the finite-time regret of \Alg, we provide a stability analysis of the designed controller and present a novel way of regret decomposition by deriving Bellman optimality equation for average cost per stage \LQG control. Utilizing the \OFU principle, we obtain the regret upper bound.

The Markov parameters of a linear time-invariant system are the input-output impulse response of the system and suffice to identify a realization of the underlying model~\citep{ljung1999system}. These parameters are used to design optimal controllers~\citep{skelton1994data}. At the end of the \textit{explore} phase, \Alg estimates the Markov parameters of the unknown linear system for the control design of the \textit{commit} phase. We adapt the recent development in the non-asymptotic system identification~\citep{oymak2018non} to our case, and estimate the Markov parameters up to upper confidence bounds via the classical Ho-Kalman algorithm~\citep{ho1966effective}. Selecting an optimistic controller that satisfies these confidence bounds, we derive our regret result.

\section{Preliminaries}\label{prelim}


We denote the Euclidean norm of a vector $x$ as $\|x\|_2$. For a given matrix $A$, $\| A \|_2$ denotes the spectral norm, $\| A\|_F$ denotes the Frobenius norm while $A^\top$ is its transpose, $A^{\dagger}$ is the Moore-Penrose inverse, and $\Tr(A)$ gives the trace of matrix $A$. The j-th singular value of a rank-$n$ matrix $A$ is denoted by $\sigma_j(A)$, where $\sigma_{\max}(A):=\sigma_1(A) \geq \sigma_2(A) \geq \ldots \geq \sigma_{\min}(A):=\sigma_n(A) > 0$. $I$ is the identity matrix with the appropriate dimension. In the following, $\mathcal{N}(\mu, \Sigma)$ denotes a multivariate normal distribution with mean vector $\mu$ and covariance matrix $\Sigma$. 

Suppose we have an unknown discrete time linear time-invariant system $\Theta = (A,B,C)$ characterized as,
\begin{align}
    x_{t+1}& = A x_t + B u_t + w_t \nonumber \\
    y_t& = C x_t + z_t \label{output},
\end{align}
where $x_t \in \mathbb{R}^{n}$ is the (latent) state of the system, $u_t \in \mathbb{R}^{p}$ is the control input, the observation $y_t \in \mathbb{R}^{m}$ is the output of the system, $w_{t} \sim \mathcal{N}\left(0, \sigma_{w}^{2} I\right),$ and $z_{t} \sim \mathcal{N}\left(0, \sigma_{z}^{2} I\right)$ are i.i.d. process noise and measurement noise, respectively. At each time step $t$, the system is at state $x_t$ and the agent observes $y_t$, \textit{i.e.}, an imperfect state information. Then, the agent applies a control input $u_t$ and the system evolves to $x_{t+1}$ at time step $t+1$. At each time step $t$, the agent pays a cost $c_t=y_{t}^{\top} Q y_{t}+u_{t}^{\top} R u_{t}$ where $Q$ and $R$ are positive semidefinite and positive definite matrices, respectively, with appropriate dimensions.
\begin{definition}
A linear system $\Theta = (A,B,C)$ is controllable if the controllability matrix
\[ [B \enskip AB \enskip A^2B \ldots A^{n-1}B]\]
has full row rank. Similarly, a linear system $\Theta = (A,B,C)$ is observable if the observability matrix \[[C^\top \enskip (CA)^\top \enskip (CA^2)^\top \ldots (CA^{n-1})^\top]^\top\]
has full column rank. 
\end{definition}
We assume that the underlying system is controllable and observable. Given this system model, the goal is to solve the following optimization problem 
\begin{equation} \label{opt control}
J_{\star}(\Theta) =  \lim _{T \rightarrow \infty} \min _{u=[u_{1}, \ldots, u_{T}]} \frac{1}{T} \mathbb{E}\left[\sum_{t=1}^{T} y_{t}^{\top} Q y_{t}+u_{t}^{\top} R u_{t}\right] 
\end{equation}
subject to dynamics in \eqref{output}, initial state  $x_0 = 0$, and where the information pattern available to $u_{t}$ is all output observations $y_{t'}$ for $0\leq t' \leq t$. If the system is known, the solution of \eqref{opt control}, \textit{i.e.}, the optimal control law, is given by a linear feedback control and we have $u_t = -K\hat{x}_{t|t,\Theta}$ with 
\begin{align*}
    K = \left(R+B^{\top} P B\right)^{-1} B^{\top} P A,
\end{align*}
where $P$ is the unique positive semidefinite solution to the following discrete-time algebraic Riccati equation (DARE):
\begin{equation} \label{ARE}
    P = A^\top  P  A  +  C^\top  Q C  -  A^\top  P B  \left(  R  +  B^\top  P B \right)^{-1}  B^\top  P  A,  
\end{equation}
and $\hat{x}_{t|t,\Theta}$ is the minimum mean square error (MMSE) estimate of the underlying state using system parameters $\Theta$, where $\hat{x}_{0|-1,\Theta}=0$. This estimate is obtained efficiently via the Kalman filter:
\begin{align}\label{eq:optKalmanStateEstimate}
     &\hat{x}_{t|t,\Theta} \!=\! \left( I \!-\! LC\right)\hat{x}_{t|t-1,\Theta} \!+\! Ly_t, \enskip \hat{x}_{t|t-1,\Theta} \!=\! (A\hat{x}_{t-1|t-1,\Theta} \!+\! Bu_{t-1}), \enskip
     L \!=\! \sig  C^\top \! \left( C \sig  C^\top \!+\! \sigma_z^2 I \right)^{-1}\!\!, 
\end{align}
where $\sig$ is the unique positive semidefinite solution to the following DARE: 
\begin{equation}\label{DARE}
    \sig = A \sig   A^\top - A \sig  C^\top \left( C \sig  C^\top + \sigma_z^2 I \right)^{-1} C \sig  A^\top + \sigma_w^2 I.
\end{equation}
$\sig$ has the interpretation of being the steady state error covariance matrix of state estimation under $\Theta$. For a positive number $H$, we have the following definitions.
\begin{definition}\label{MarkovDef}
The Markov parameters matrix 
\begin{align*}
    G = [0_{m \times p} \enskip CB \enskip CAB \enskip   \ldots   \enskip CA^{H-2}B] \in   \R^{m \times Hp}
\end{align*}
denotes the matrix of length  $H$ impulse response of the system $\Theta$. 
\end{definition}
\begin{definition} 
For $d_1 + d_2 + 1 = H$, the Hankel matrix $\mathbf{H} \in \R^{md_1 \times p(d_2+1)} $ is a $d_1 \times\left( d_2+1\right)$  block matrix such that $(i,j)$th block of $\mathbf{H}$ is $(i+j)$ block in $G$, e.g., $C A^{i+j-2} B$ for $(i,j)>1$.
\end{definition}
We set the value of H in Section \ref{SysId} in order to guarantee the identifiability of the system parameters. We are interested in the setting where the matrices $A$, $B$, and $C$, therefore the Markov parameters $G$ of the system, are unknown. The cost matrices $Q$ and $R$ impose the desired outcome and are given. The agent interacts with this environment for $T$ time steps and aims to minimize its cumulative cost $\sum_{t=1}^T c_t$. At time $T$, we evaluate the agent's performance by its regret
\begin{equation*}
    \reg(T) = \sum_{t=0}^T (c_t - J_*(\Theta)),
\end{equation*}
which is the difference between the cost incurred by the agent and the expected cost of the optimal controller, which knows the system $\Theta$. We consider the following problem setup.
\begin{assumption} \label{Stable}
The system is order $n$ and stable, \textit{i.e.} $\rho(A) < 1$, where $\rho(\cdot)$ denotes the spectral radius of a matrix which is the largest absolute value of its eigenvalues. Define $\Phi(A) = \sup _{\tau \geq 0} \frac{\left\|A^{\tau}\right\|}{\rho(A)^{\tau}}$. We assume that $\Phi(A) < \infty$ for the given system.  
\end{assumption}
\noindent \textbf{Remark:}
Note that the stability of the system ensures the system identification, prevents explosive behavior in the underlying system, and is the main topic of study in the recent literature \citep{oymak2018non, sarkar2019finite, tsiamis2019finite}. The assumption regarding $\Phi(A)$ is a mild condition, \textit{e.g.} if $A$ is diagonalizable, $\Phi(A)$ is finite. 

\begin{assumption} \label{AssumContObs}
The unknown system $\Theta =(A,B,C)$ is a member of a set $\mathcal{S}$, defined as,  
\begin{equation*}
    \mathcal{S} \subseteq \mathcal{S}_0 \cap \{ \Theta' =  (A' ,B' ,C') |~\rho(A')<1, \Tr(G_{\Theta'}^\top G_{\Theta'}) \leq \kappa^2 \},
\end{equation*}
for some $\kappa \geq 0$, where $G_{\Theta'} \in \R^{m\times Hp}$ is the Markov parameters matrix of system $\Theta'$ and 
\begin{equation*}
    \mathcal{S}_0 = \{ \Theta' = (A',B',C') |  \enskip (A',B')\text{ is controllable and } 
    (A',C') \text{ is observable} \}.
\end{equation*}
\end{assumption}

\begin{assumption}\label{Stabilizable set}
The set $\mathcal{S}$ consists of systems that are contractible, \textit{i.e.,}
\begin{align*}
    \rho \coloneqq \sup_{\Theta' = (A',B',C') \in \mathcal{S}} \left\|A' - B'K(\Theta') \right\| < 1, 
\end{align*}
where $K(\Theta')$ is the optimal control gain of $\Theta'$, and
\begin{align*}
    \upsilon \coloneqq \sup_{\Theta' = (A',B',C') \in \mathcal{S}} \left\|A' - A'L(\Theta')C'\right\| < 1.
\end{align*}
where $L(\Theta')$ is the optimal Kalman gain of $\Theta'$. Moreover, there exists finite numbers $D,  \Gamma, \zeta$ such that $D = \sup_{\Theta' \in \mathcal{S}} \|P(\Theta') \|$, $\Gamma = \sup_{\Theta' \in \mathcal{S}} \|K(\Theta') \|$ and $\zeta = \sup_{\Theta' \in \mathcal{S}} \|L(\Theta') \|$. 
\end{assumption}

\section{Explore, then Commit}
In this section, we study the components of \Alg and provide the required theoretical analysis. \Alg, as outlined in Alg.~\ref{algo}, is given $T$, the total number of interactions; and operates in two phases: the first $\Texp$ time steps are dedicated to exploration and in the remaining $T-\Texp$, \Alg acts according to the optimal controller of the optimistic model, learned from the experiences in the first phase. 

\begin{algorithm}[htb!] 
\caption{\Alg}
  \begin{algorithmic}[1]
 \STATE \textbf{Input:} $T$, $\Texp$, $H$, $S>0$, $\delta > 0$, $n$, $Q$, $R$ \\
 ------ \textsc{\small{Exploration}} ----------------------------------------- \\
\STATE  Deploy $u_t \sim \mathcal{N}(0,\sigma_u^2 I)$ for the first $\Texp$ time steps and store $\Dexp = \lbrace y_t,u_t \rbrace_{t=1}^{\Texp}$
\STATE Calculate $\hat{G}_t$ using $\Dexp$ 
\STATE Deploy HO-KALMAN($H,\tilde{G}_t,n$) to obtain $\hat{A}, \hat{B}, \hat{C}$
\STATE Construct the confidence intervals $\mathcal{C}_A, \mathcal{C}_B, \mathcal{C}_C$ such that $(A,B,C) \in (\mathcal{C}_A \times \mathcal{C}_B \times \mathcal{C}_C)$ w.h.p
\STATE Find a $\tth = (\ta,\tb,\tc) \in (\mathcal{C}_A \times \mathcal{C}_B \times \mathcal{C}_C) \cap \mathcal{S} $ s.t.\\
$ \qquad J(\tth) \leq \inf_{\Theta' \in (\mathcal{C}_A \times \mathcal{C}_B \times \mathcal{C}_C) \cap \mathcal{S}} J(\Theta') + T^{-1/3}$
\STATE Construct the optimal control law based on $\tilde{A}, \tilde{B}, \tilde{C}$ \\
------ \textsc{\small{Commit Optimistically(Exploitation)}} ------ \\
\FOR{$t = \Texp+1, \ldots, T$}
\STATE Execute the constructed control law
\ENDFOR
  \end{algorithmic}
 \label{algo} 
\end{algorithm}

In the exploration phase, \Alg excites the system with $u_t \sim \mathcal{N}(0,\sigma_u^2 I)$ for $1\leq t\leq T_{exp}$. Using the generated input-output trajectory $\Dexp = \{y_t, u_t \}^{T_{exp}}_{t=1}$, \Alg constructs $N$ subsequences of length $H$, each, where $T_{exp}= H + N - 1$ and $N\geq 1$. For all $t$, such that $H\leq t \leq \Texp$, define,
\begin{align}
    \Bar{u}_t = [u_t^{\top} \quad u_{t-1}^{\top} \quad \ldots \quad u_{t-H+1}^{\top}]^{\top} \in \mathbb{R}^{Hp}, \quad
    \Bar{w}_t = [w_t^{\top} \quad w_{t-1}^{\top} \quad \ldots \quad w_{t-H+1}^{\top}]^{\top} \in \mathbb{R}^{Hn}. \nonumber
\end{align}
Using these definitions, we write (\ref{output}) as 
\begin{equation} \label{second}
    y_t = G \Bar{u}_t + F \Bar{w}_t + z_t + e_{t}
\end{equation}
where $e_t = CA^{H-1}x_{t-H+1}$ and $F = [0_{m \times n} \enskip C \enskip CA \enskip   \ldots   \enskip CA^{H-2}] \in \mathbb{R}^{m \times Hn}$ is the process noise to output Markov parameters matrix. Let's define,
\begin{align}
  W & = [\Bar{w}_H,\enskip \Bar{w}_{H+1},\enskip   \ldots   \enskip, \Bar{w}_{T_{exp}}]^{\top} \in \mathbb{R}^{N \times Hn}, \quad 
  Z = [z_H,\enskip  z_{H+1},\enskip   \ldots   \enskip, z_{T_{exp}}]^{\top} \in \mathbb{R}^{N \times m}, \label{mat def} \\
  E &= [e_H,\enskip e_{H+1},\enskip   \ldots   , e_{\Texp}]^{\top} \in \mathbb{R}^{N \times m}, \quad \quad \enskip 
  Y = [y_H,\enskip  y_{H+1}, \enskip   \ldots   \enskip, y_{T_{exp}}]^{\top} \in \mathbb{R}^{N\times m}, \nonumber\\
  U &= [\Bar{u}_H,\enskip \Bar{u}_{H+1}, \enskip   \ldots   \enskip, \Bar{u}_{T_{exp}} ]^\top \in \mathbb{R}^{N \times Hp}. \nonumber 
\end{align}
Using these, we restate (\ref{second}) as the following linear system, 
$Y = UG^{\top} + E + Z + WF^{\top}$. Then, we estimate the unknown $G$ by solving the following least square problem,
\begin{equation} \label{leastsquares}
    \hat{G} = \arg \min_X ||Y - UX^{\top}||_F^2  = \big[(U^{\top}U)^{-1}U^{\top}Y\big]^{\top}  .
\end{equation}
\Alg deploys this estimate $\hat{G}$ and computes the underlying system parameters $\Theta$ up to their confidence levels, consisting of a set of plausible models. The algorithm exploits this construction and employs the optimal control law of the most optimistic model (the model with the minimum promising average cost in the set of plausible models) for the remaining $T-T_{exp}$ time steps,  \textit{i.e.,} the commit phase. In the remainder of this section, we analyze the main components and phases of \Alg. In the next section, we conclude these analyses with the final regret bound. 

For some probability $\delta$, defined  later for the confidence set constructions, we consider the following terms:
\begin{align}
R_z = 4\sigma_z &\left( \sqrt{H p+m} + \sqrt{\log\left( \frac{1}{\delta}\right)}\right), \quad 
R_w={\sigma_w\|F\|\max\{\sqrt{N_w},N_w/\sqrt{N}\}},\label{eq:Rs}\\
& \quad \quad  R_e=c'\sigma_e  \sqrt{ \log\left(\frac{H}{\delta}\right)  \max \left \{1, \frac{H \left(3m + \log\left(\frac{H}{\delta} \right) \right)}{N\left(1-\rho(A)^{H}\right)} \right\}  }.\nonumber
\end{align}
which capture the noise variance and the dimensionality of the problem. Here $N_w=cH (p+n)\log^2(2H (p+n))\log^2(2\Texp(p+n))$, $c,c'>0$ are absolute constants and $\sigma_e$ is a constant assigned later in \eqref{eq:sigmae}.  Recalling the Definition \ref{MarkovDef} and the Assumption \ref{Stabilizable set}, we define the following time intervals,
\begin{align}
    &T_{G}  =  \frac{(R_w + R_e +  R_z)^2}{\sigma_u^2},\quad T_{B}  =  T_{G}\left(\frac{7n}{\sqrt{\sigma_n(\mathbf{H})}}\right)^2, \quad T_{A}  =  T_{G} \left( \frac{31n\|\mathbf{H}\| + 7n\sigma_n(\mathbf{H})}{\sigma_n^2(\mathbf{H})}\right)^2, \nonumber \\
    &T_{M}  \!=\!  T_B \left(\frac{ 2\zeta \rho }{1-\rho} \right)^2, \!\enskip \! T_{N} \!=\! T_G \left(\frac{4\sqrt{n}}{\sigma_n(\mathbf{H})}\right)^2, \!  \enskip \!T_L \!=\! \max\{T_A, T_B\}\!\!  \left(\frac{\!c_1\! \left(2\|C\| \!+\! 1\right)\Phi(A)^2 \!+\! c_2 (2\Phi(A)\!+\!1)\! }{1 - \upsilon^2} \right)^2 \!\! \nonumber \\
    &T_\alpha \!=\! T_B \left( \frac{\Gamma\left(1+\zeta(1+\| C\|)\right)}{\sigma - \upsilon}\right)^2\!\!, \enskip T_{\beta} \!=\! \max\{T_A, T_B\}\!\! \left( \frac{\Gamma\|B\|(1\!+\!\zeta \!+\! \zeta \|C\|)(\Phi(A)\zeta \!+\! (1\!+\!\Gamma)(1\!+\!\zeta) ) }{(1 - \sigma)^2} \right)^2\!\!\!, \nonumber \\
    &T_\gamma = \max\{T_A, T_B\} \left( \frac{1 + \Gamma(1+\zeta\| B\|)}{\sigma - \rho} \right)^2 
    \label{times}
\end{align}
where $c_1$ and $c_2$ are problem dependent constants and $1> \sigma > \max\{\rho, \upsilon \}$. These define the requirements for theoretical guarantees of the algorithm. Finally, combining these, we are ready to state $T_0$, the minimum exploration period for \Alg,
\begin{equation}\label{explorationlowermain}
    \Texp > T_0 = \max \left\{T_A, T_B, T_G, T_L, T_M, T_N, T_\alpha, T_{\beta}, T_\gamma \right\} + H.
\end{equation}
The following theorem gives the regret upper bound of \Alg which is the first sublinear end-to-end regret result for \LQG control.
\begin{theorem}[Regret Upper Bound]\label{total regret main text}
Given a \LQG $\Theta = (A,B,C)$, and regulating parameters $Q$ and $R$, with high probability, the regret of the \Alg after $T\geq T_0^{3/2}$  time steps with $\Texp = T^{2/3}$ is 
\begin{equation}
    \reg(T) = \tilde{\OO}\left(T^{2/3}\right).
\end{equation}
\end{theorem}
The formal version of the theorem, as well as the detailed analysis of the deriving components of the result, are provided in the Appendix~\ref{SuppRegret} and \ref{SuppRegretTotal}. In Section~\ref{LearnMarkov}, we state the confidence set around the Markov parameters matrix estimation $\hat{G}$ and  in Section~\ref{SysId} we provide the confidence sets over the estimated model dynamics $\tth = (\hat{A},\hat{B},\hat{C})$. Finally, in Section~\ref{ControlDesign}, we provide the controller design using these confidence sets and additional guarantees that these confidence sets bring. 

\subsection{Learning Markov Parameters}
\label{LearnMarkov}
We deploy the proposed regression problem in \eqref{leastsquares} to estimate $G$. As provided in \eqref{second}, $e_t$ is dependent on previous noise and control inputs, therefore, special care is required to analyze such regression. In order to characterize the impact of $e_t$, we define $\sigma_e$, the ``effective standard deviation'', as follows
\begin{equation}\label{eq:sigmae}
    \sigma_{e} = \Phi(A) \|CA^{H-1}\| \sqrt{\frac{H \|\Gamma_\infty\|}{1-\rho(A)^{2H}}} 
\end{equation}
where $\Gamma_\infty$ is the steady state covariance matrix of $x_t$, $\Gamma_{\infty}\!=\!\sum_{i=0}^{\infty} \sigma_{w}^{2} A^{i}\left(A^{\top}\right)^{i}+\sigma_{u}^{2} A^{i} BB^{\top}\left(A^{\top}\right)^{i}.$ 

\begin{theorem}
[Markov Parameter Estimation~\citep{oymak2018non}]
\label{error spectral} For $\delta \in (0,1)$, suppose $N\geq cH p \log(\frac{1}{\delta})$ for some absolute constant $c>0$. Suppose Assumption \ref{Stable} holds. Run the system with control input $u_t \sim \mathcal{N}(0,\sigma_u^2 I)$ for $T_{exp}= N+H-1$ time steps to collect $N$ samples of $H$ input-output pairs. Let, $\hat{G}$ be the least squares estimate of the Markov parameters matrix. Then, the following holds for $\hat{G}$ with probability at least $1-7\delta$,
\begin{equation}
\|\hat{G}-G\|\leq \frac{R_w+R_e+R_z}{\sigma_u\sqrt{\Texp-H+1}},\label{the bound}
\end{equation}
\end{theorem}
The proof can be found in Appendix \ref{Proof Pieces}.
This result is obtained by a set of small tweaks in the bounds in~\cite{oymak2018non} to make them appropriate for the current problem of study. The main idea is to decompose $\hat{G}- G$ into parts using \eqref{leastsquares} and bound each part separately.

\subsection{System Identification}
\label{SysId}
After estimating the Markov parameters matrix $\hat{G}$, we are ready to come up with a balanced realization of $\Theta$ from $\hat{G}$. 
\Alg deploys the Ho-Kalman algorithm for this task. The details of the Ho-Kalman algorithm are given in Appendix \ref{kalmanhosupp} and Algorithm~\ref{kalmanho}. The Ho-Kalman algorithm, as
a subroutine to \Alg, receives $\hat{G}$ as an input and computes an order $n$ system $\hat{\Theta} = (\hat{A}, \hat{B}, \hat{C})$. Note that only the order $n$ input-output response of the system is uniquely identifiable~\citep{ljung1999system}. In other words, the system parameters $\Theta$ (even with the correct Markov parameters matrix $G$) are recovered up to similarity transformation. More generally, for any invertible $\mathbf{T}\in \R^{n \times n}$, the system $A' = \mathbf{T}^{-1}A\mathbf{T}, B' = \mathbf{T}^{-1}B, C' = C \mathbf{T} $ gives the same Markov parameters matrix $G$, equivalently, the same input-output impulse response. In this work, we are interested to recover ${\Theta}$, up to a similarity transformation.

For $H\geq2n+1$, using $\hat{G} = [\hat{G}_1 \ldots \hat{G}_H] \in \R^{m \times Hp}$, where $\hat{G}_i$ is the $i$th $m \times p$ block of $\hat{G}$ for all $1 \leq i \leq H$, the algorithm constructs a $(n \times n+1)$ block Hankel matrix $\mathbf{\hat{H}} \in \R^{nm \times (n+1)p}$ such that $(i,j)$th block of Hankel matrix is $\hat{G}_{(i+j)}$. It is worth noting that if the input to the algorithm was $G$ then corresponding Hankel matrix, $\mathbf{H}$ is rank $n$, more importantly,
\begin{align*}
     \mathbf{H}    =   [C^\top~(CA) ^\top    \ldots (CA^{n-1}) ^\top] ^\top [B~AB\ldots A^{n}B]=\mathbf{O} [\mathbf{C} ~~A^nB]=\mathbf{O} [B ~~A\mathbf{C}]
\end{align*}
%
where $\mathbf{O}$ and $\mathbf{C}$ are observability and controllability matrices respectively. Essentially, the Ho-Kalman algorithm estimates these matrices using $\hat{G}$. In order to estimate $\mathbf{O}$ and $\mathbf{C}$, the algorithm constructs $\mathbf{\hat{H}}^-$, the first $np$ columns of $\mathbf{\hat{H}}$ and calculates $\mathbf{\hat{N}}$, the best rank-$n$ approximation of $\mathbf{\hat{H}}^-$. Therefore, the singular value decomposition of $\mathbf{\hat{N}}$ provides us with the estimates of $\mathbf{O},\mathbf{C}$, \textit{i.e.}, $\mathbf{\hat{N}} = \mathbf{U}\mathbf{\Sigma}^{1/2}~ \mathbf{\Sigma}^{1/2}\mathbf{V} =\mathbf{\hat{O}}\mathbf{\hat{C}} $. From these estimates, the algorithm recovers $\hat{B}$ as the first $n\times p$ block of $\mathbf{\hat{C}}$, $\hat{C}$ as the first $m \times n$ block of $\mathbf{\hat{O}}$, and $\hat{A}$ as $\mathbf{\hat{O}}^\dagger \mathbf{\hat{H}}^+ \mathbf{\hat{C}}^\dagger$ where $\mathbf{\hat{H}}^+$ is the submatrix of $\mathbf{\hat{H}}$, obtained by discarding the left-most $nm \times p$ block. 

Note that if we feed $G$ to the Ho-Kalman algorithm, the $\mathbf{H}^-$ is the first $np$ columns of $\mathbf{H}$, it is rank-$n$, and $\mathbf{N} \!=\! \mathbf{H}^-$. Using the outputs of the Ho-Kalman algorithm, \textit{i.e.,} $(\hat{A}, \hat{B}, \hat{C})$, \Alg constructs confidence sets centered around these outputs that contain a similarity transformation of the system parameters $\Theta \!=\!\! (A, B, C)$ with high probability. Theorem \ref{ConfidenceSets} states the construction of confidence sets and it is a slight modification of Corollary 5.4 of~\citet{oymak2018non}.

\begin{theorem}[Confidence Set Construction] \label{ConfidenceSets}
Let $H\geq 2n+1$. For $\delta \in (0,1)$, suppose $N\geq cH p \log(\frac{1}{\delta})$ for some absolute constant $c>0$. Let $\mathbf{H}$ be the Hankel matrix obtained from $G$. Let $\bar{A}, \bar{B}, \bar{C}$ be the system parameters that Ho-Kalman algorithm provides for $G$. Let $\hat{A}, \hat{B}, \hat{C}$ be the system parameters obtained from the Ho-Kalman algorithm with using the least squares estimate of the Markov parameter matrix $\hat{G}$ after the exploration of $\Texp$ time steps. Suppose Assumptions \ref{Stable} and \ref{AssumContObs} hold, thus $\mathbf{H}$ is rank-$n$. Then, there exists a unitary matrix $\mathbf{T} \in \R^{n \times n}$ such that, with probability at least $1-7\delta$, $\bar{\Theta}=(\bar{A}, \bar{B}, \bar{C}) \in (\mathcal{C}_A \times \mathcal{C}_B \times \mathcal{C}_C) $ where
\begin{align}
    &\quad \mathcal{C}_A = \left \{A' \in \R^{n \times n} : \|\hat{A} - \mathbf{T}^\top A' \mathbf{T} \| \leq \beta_A \right\}, \enskip 
    \mathcal{C}_B = \left \{B' \in \R^{n \times p} : \|\hat{B} - \mathbf{T}^\top B' \| \leq  \beta_B \right\}, \nonumber \\
    &\quad \mathcal{C}_C = \left\{C' \in \R^{m \times n} : \|\hat{C} -  C'  \mathbf{T} \| \leq \beta_C \right\}, \text{ for } \\
    &\beta_A \!=\! \left( \frac{31n\|\mathbf{H}\| + 7n\sigma_n(\mathbf{H})}{\sigma_n^2(\mathbf{H})} \right) \frac{R_w+R_e+R_z}{\sigma_u\sqrt{\Texp-H+1}}, \enskip 
    \beta_B \!=\! \beta_C  \!=\!  \frac{7n(R_w\!+\!R_e\!+\!R_z)}{\sigma_u \sqrt{\sigma_n(\mathbf{H}) (\Texp\!-\!H\!+\!1)}}. \nonumber
\end{align}
\end{theorem}
The proof is in the Appendix \ref{SuppConfSet}. Theorem~\ref{ConfidenceSets} translates the results in Theorem~\ref{error spectral} to system parameters.
\subsection{Designing an Optimistic controller}
\label{ControlDesign}
After constructing the confidence sets, the algorithm deploys the \OFU and chooses the system parameters $\tth = (\ta, \tb, \tc) \in (\mathcal{C}_A \times \mathcal{C}_B \times \mathcal{C}_C) \cap \mathcal{S} $ such that 
\begin{equation} \label{optimistic}
    J(\tth) \leq \inf_{\Theta' \in (\mathcal{C}_A \times \mathcal{C}_B \times \mathcal{C}_C) \cap \mathcal{S}} J(\Theta')+ T^{-1/3}.
\end{equation}
Through control synthesis, the algorithm designs the optimal feedback controller $(\tp, \tk, \tl)$ for the system $\tth$. The algorithm uses this optimistic controller to control the underlying system $\Theta$ for the remaining $T-\Texp$ time steps. Before the main regret analysis, we provide the following lemma on $\|\tsig  - \sig \|$ where $\tsig $ is the solution to DARE given in \eqref{DARE} for the system $\tth$ and $\sig $ is the solution to \eqref{DARE} for the underlying system $\Theta$. Under the Assumption \ref{Stabilizable set}, we obtain the following.

\begin{lemma} \label{StabilityCov}
For $\delta \in (0,1)$, suppose $N\geq cH p \log(\frac{1}{\delta})$ for some absolute constant $c>0$. Suppose Assumption \ref{Stabilizable set} holds. Then, there exists a similarity transformation $\mathbf{S} \in \R^{n \times n}$ such that, with probability at least $1-7\delta$, 
\begin{align*}
    \|\tsig - \mathbf{S}^{-1}\sig \mathbf{S}  \| &\leq \Delta \sig \coloneqq \frac{\Phi(A)^2 (8\|C\|+ 4)\|\Sigma \|^2 + \sigma_z^2 (8\Phi(A) + 4)  \| \Sigma\| }{\sigma_z^2(1 - \upsilon^2)} \max\left\{\beta_A, \beta_C \right\} \\
    \|\tl - \mathbf{S}^{-1} L\| &\leq k_1 \Delta \sig + k_2 \beta_C,
\end{align*}
for some problem dependent constants $k_1$ and $k_2$.
\end{lemma}
The proof is in Appendix \ref{Ldifference}. In order to derive the results in the Lemma~\ref{StabilityCov}, we utilize a fixed point argument on the DARE of the steady state error covariance matrix of state estimation. We construct an operator that has $\tsig - \mathbf{S}^{-1} \sig \mathbf{S}$ as the fixed point and then argue that the norm of it is bounded as shown in Lemma \ref{StabilityCov} for the given exploration duration. We combine the bound with the definition of the Kalman gain and under the Assumption \ref{Stabilizable set}, we obtain the bound on $\|\tl - \mathbf{S}^{-1} L \|$.
Later, in Lemma \ref{Boundedness}, we deploy these results to show the boundedness of MMSE estimate of the underlying state using optimistic system parameters $\tth$ and the output of the system, $\|\hat{x}_{t|t,\tth} \|$ and $\|y_t \|$, respectively.  
Lemma \ref{StabilityCov} is  a crucial concentration result and it is repeatedly used in the regret analysis of \Alg along with Theorem \ref{ConfidenceSets}.

\section{Regret Analysis of \Alg}
In this section, we provide the regret analysis of \Alg that leads to the guarantee in Theorem  \ref{total regret main text}. At first, we discuss the regret due to exploring the system with $u_t \sim \mathcal{N}(0,\sigma_u^2 I)$ in \textit{explore} phase for $1\leq t\leq T_{exp}$. Then, we provide the regret of deploying the optimistic controller in \textit{commit} phase for $T_{exp}+1\leq t\leq T$.
\subsection{Regret of the Explore Phase}
\begin{lemma}
\label{exploration regret}
Suppose Assumptions \ref{AssumContObs} and \ref{Stabilizable set} hold. Given a \LQG $\Theta = (A,B,C)$, the regret of deploying $u_t \sim \mathcal{N}(0,\sigma_u^2 I)$ for $1\leq t\leq \Texp$ is upper bounded as follows with high probability
\begin{align}\label{eq:informalREG}
    \reg(\Texp) = c_* \Texp +  \tilde{\OO}\left(\sqrt{\Texp}\right)
\end{align}
where $c_*$ is a problem dependent constant.
\end{lemma}
This lemma might feel intuitive to many readers. One of the main reasons we provide Lemma \ref{exploration regret} is the importance and contribution of $\tilde{\OO}\left(\sqrt{\Texp}\right)$ terms in \eqref{eq:informalREG} to the final regret bound. The proof and the precise expressions are in Appendix \ref{SuppRegretExplore}. The Lemma~\ref{exploration regret} shows that the exploration phase in \Alg results in a linear regret upper bound. In order to compensate for the regret accumulated during the \textit{explore} phase, \Alg is required to efficiently exploit the information gathered during this phase. The algorithm fulfills this by following an optimistic controller during the \textit{commit} phase. 

\subsection{Regret of the Commit phase}
In the following, we show that with a proper choice of exploration duration, the designed controller by \Alg 
results in stable performance during the \textit{commit} phase. After exploring the system, \Alg commits to an optimistic controller and does not change the controller during the \textit{commit} phase.
Therefore, extra care is required in determining the duration of the \textit{explore} phase, since the tightness of confidence sets are critical in obtaining a desired performance during the \textit{commit} phase. 

During this phase, despite the fact that the environment evolves based on $\Theta$, the agent behaves optimally with respect to the optimistic model $\tth = (\ta,\tb,\tc)$, as if the environment evolves based on $\tth$. Therefore, at time step $t$, the agent observes output $y_t$ from $\Theta$ and updates its state estimation $\hat{x}_{t|t,\tth}$ using \eqref{eq:optKalmanStateEstimate} and model $\tth = (\ta,\tb,\tc)$. 
The agent computes the optimal control $u_t$, using the optimal feedback controller for the system $\tth$. 
We show that by exploring the system $\Theta$ for $\Texp > T_0$ time steps, \Alg is guaranteed to maintain bounded state estimation, $\hat{x}_{t|t,\tth}$ and thus bounded output $y_{t+1}$ 
for $T-\Texp$ time steps of interaction in the \textit{commit} phase.  

\begin{lemma} \label{Boundedness}
Suppose Assumption \ref{Stabilizable set} holds and \Alg explores the system $\Theta$ for  $\Texp$ time steps. For $\delta \in (0,1)$, suppose $N\geq cH p \log(\frac{1}{\delta})$ for some absolute constant $c>0$. For \Alg, define the following three events in the probability space $\Omega$, 
\begin{align*}
    \mathcal{E} &= \left\{\omega \in \Omega :  \Theta \in (\mathcal{C}_A \times \mathcal{C}_B \times \mathcal{C}_C) \right\}, \quad 
    \mathcal{F} = \left\{\omega \in \Omega : \forall t \leq T, \| \hat{x}_{t|t,\tth}\| \leq \tilde{\mathcal{X}} \right\}, \\
    \mathcal{G} &= \Bigg\{\omega \in \Omega: \forall t \leq T, \|y_t \| \leq \rho \| C\| \tilde{\mathcal{X}}  + \| C\| \bar{\Delta} + \left(\| C\| \|\Sigma\|^{1/2} + \sigma_z \right) \sqrt{2m\log(2mT/\delta)} \Bigg\}, 
    \end{align*}
where 
\begin{align*}
    \tilde{\mathcal{X}} &= \frac{2\zeta \left(\|C\| \bar{\Delta} + \left(\| C\| \|\Sigma\|^{1/2} + \sigma_z\right) \sqrt{2n\log(2nT/\delta)}\right) }{1-\rho} 
\end{align*}
for $\bar{\Delta} =  poly(\Phi(A), \beta_A, \beta_B, \|B\|, \|C\|, \rho, \zeta, \Gamma, m, \sigma, \|\Sigma \| )$. Then, $\Pr\left( \mathcal{E} \cap \mathcal{F} \cap \mathcal{G}\right) \!\geq\! 1\!-\!9\delta.$ 
\end{lemma}
The proof and the exact expression for $\bar{\Delta}$ are given in Appendix~\ref{SuppBounded}. Under the event $\mathcal{E}$, the underlying system parameters $\Theta$ is contained in the confidence sets constructed in Theorem \ref{ConfidenceSets}. 
The event $\mathcal{F}$ indicates that the state estimation used by the agent, despite using $\tth$, is bounded above. Finally, the event $\mathcal{G}$ indicates that the output of the system is bounded above despite a controller that is optimal for the optimistic system $\tth$ is deployed by \Alg. More generally, the Lemma~\ref{Boundedness} shows that the state estimation and the output obtained by using optimal feedback controller $(\tp, \tk, \tl)$ of $\tth$ chosen from the confidence sets, is logarithmic in $T$ as long as the system is explored for $\Texp > T_0$ time steps. 

Under the events described in the Lemma~\ref{Boundedness}, the agent maintains a stable performance. Under this stability guarantee, we analyze the regret of \Alg during the \textit{commit} phase. In order to provide a regret decomposition for the optimistic controller, we first derive the Bellman optimality equation for the average cost per step \LQG control problem. 

For infinite state and control space average cost per step problems, \textit{e.g.} the \LQG control system $\Theta = \left(A,B,C\right)$ with regulating parameters $Q$ and $R$, the optimal average cost per stage $J_*(\Theta)$ and the differential(relative) cost satisfy Bellman optimality equation~\citep{bertsekas1995dynamic}. In the following lemma, we identify the correct differential cost for \LQG and obtain the Bellman optimality equation.

\begin{lemma}[Bellman Optimality Equation for \LQG] \label{LQGBellman}
Given state estimation $\hat{x}_{t|t-1} \in \R^{n}$ and an observation $y_t \in \R^{m}$ pair at time $t$, Bellman optimality equation of average cost per stage control of \LQG system $\Theta = (A,B,C)$ with regulating parameters $Q$ and $R$ is 
\begin{align}
    &J_*(\Theta)  +  \hat{x}_{t|t}^\top  \left( P  -  C^\top Q C\right)  \hat{x}_{t|t}  +  y_t^\top Q y_t =  \min_u   \bigg\{ y_t^\top Q y_t + u^\top   R u \label{bellman} \\
    &\qquad \qquad \qquad \qquad \qquad  \qquad \qquad \qquad +  \mathbb{E}\bigg[ \hat{x}_{t+1|t+1}^{u \top}  \left( P  -   C^\top Q C  \right)  \hat{x}_{t+1|t+1}^{u}    +  y_{t+1}^{u \top} Q y_{t+1}^u  \bigg] \nonumber
    \bigg\}
\end{align}
where $P$ is the unique solution to DARE of $\Theta$, $\hat{x}_{t|t}= (I-LC)\hat{x}_{t|t-1} + L y_t$, $y_{t+1}^u = C(Ax_t + Bu + w_t) + z_{t+1}$, and $\hat{x}_{t+1|t+1}^{u} = \left( I - LC\right)(A\hat{x}_{t|t} + Bu) + Ly_{t+1}^u$.
The equality is achieved by the optimal controller of $\Theta$.
\end{lemma}
The proof of this result is given in Appendix \ref{SuppBelmanOptimality}.
Note that we can also write the same Bellman equation \eqref{bellman} for the optimistic system $\tth$. To obtain the regret expression of \Alg, we use the optimistic nature of the controller, and the characterization of the difference between $J_*(\tth) - J_*(\Theta)$. 
%
We then decompose regret into several terms as provided in Appendix~\ref{SuppRegret}. The following expresses the regret upper bound of the \textit{commit} phase of \Alg. 

\begin{theorem}[The regret of optimistic controller] \label{exploitation regret}
Given a \LQG control system $\Theta = (A,B,C)$ with regulating parameters $Q$ and $R$, suppose the Assumptions \ref{Stable}-\ref{Stabilizable set} hold. After a \textit{explore} phase with a duration of $\Texp$, if \Alg interacts with the system $\Theta$ for the remaining $T-\Texp$ steps using optimistic controller, with high probability the regret of \Alg in the commit phase is bounded as follows,
\begin{equation}
    \reg(T) = \tilde{\OO}\left( \frac{T-\Texp}{\sqrt{\Texp}} + T^{2/3} \right).
\end{equation}
\end{theorem}
The proof is provided in Appendix \ref{SuppRegretTotal}, where we analyze individual terms in the regret decomposition. The analysis builds upon Lemmas~\ref{StabilityCov} and~\ref{Boundedness}. For each term in the decomposition of the regret, we use the fact that the error in the system parameter estimates is $\tilde{\OO}\left( 1/\sqrt{\Texp}\right)$ and the input and the output of the system during commit phase is well-controlled with high probability via Lemma \ref{Boundedness}. For $\Texp>T_0$, via combining Lemma \ref{exploration regret} and Theorem \ref{exploitation regret}, we have the following total regret expression for \Alg,
\begin{equation}\label{eq:lastminut}
    \reg(T) = \tilde{\OO}\left( \Texp + \frac{T-\Texp}{\sqrt{\Texp}} + T^{2/3}\right)
\end{equation}
 after the total $T$ interactions, resulting in the trade-off between exploration and exploitation. On one hand, one can increase the duration of exploration for a given $T$ to obtain tighter confidence sets and thus better control performance but the linear regret obtained during exploration will dominate the total regret term. On the other hand, one can decrease the exploration duration for a given $T$ but then the optimistically chosen controller wouldn't give good performance on the underlying system. 

Based on this trade-off, for the total interaction time $T$ such that 
$T^{2/3} > T_0$, setting $\Texp = T^{2/3}$ provide us with the optimal performance for the algorithm. Therefore, by substituting $\Texp = T^{2/3}$ in the \eqref{eq:lastminut}, the regret of \Alg is upper bounded with $\tilde{\OO}\left(T^{2/3}\right)$ as stated in Theorem \ref{total regret main text}.

\section{Related Works}
In control, the design of optimal controllers for given partially observable linear systems, especially with quadratic cost and Gaussian perturbation, has played a significant role in the development of Kalman filter and controlling complex systems~\citep{aastrom2012introduction,bertsekas1995dynamic,hassibi1999indefinite}. In a variety of real-world problems, the underlying system is unknown to the agent, therefore learning and identifying the unknown system is crucial. When the environment is fully observable, mainly asymptotic behavior of adaptive control methods are understood~\citep{lai1982least,chen1987optimal,campi1998adaptive}. These methods mainly rely on pure exploration to learn the model parameters except~\citet{campi1998adaptive}, which is based on the \ofu principle. The \ofu principle has been deployed to propose efficient algorithms for both fully and partially observable Markov decision process~\citep{jaksch2010near,azizzadenesheli2016reinforcement}. An interesting work ~\citet{abbasi2011regret} deploys the statistical developments in the self normalized analysis of linear models~\citep{pena2009self,abbasi2011improved} and provide the first regret upper bound of $\tilde{\mathcal{O}}(\sqrt{T})$ for the fully observable case. Further, development has been made to improve these results from either the statistical and computational point of view where either \ofu or pure exploration have been used~\citep{faradonbeh2017optimism,abeille2017thompson,abeille2018improved,ouyang2017learning,dean2018regret}. Moreover, methods involving pure exploration or greedy have been studied by ~\citep{abbasi2019model,mania2019certainty, faradonbeh2018input, cohen2019learning}. While the variety of mentioned works study the full adaptive case, the simplicity and importance of \textit{explore-then-commit}~\citep{garivier2016explore} based methods maintain their significance.

Mainly, the mentioned studies in partially observable domains, except an interesting section on partially observable setting  of~\citet{mania2019certainty}, assume the state of the system is observable. Partial observability of linear systems introduces a set of new challenges in both learning the model parameters as well as analyzing the regret. For system identification, \textit{i.e.}, parameter estimation, ~\citet{chen1992integrated,juang1993identification,phan1994system,oymak2018non,sarkar2019finite} study the \LQG{}s and show how to learn the model parameters through samples generated by deploying a Gaussian control. In this paper, we employ these analyses, adapt them, and provide a detailed regret study. In the case of robust control, when bounded energy adversarial noise is considered,~\citet{hassibi1999indefinite} proposes a study of $\mathcal{H}_2$ and $\mathcal{H}_\infty$ in control synthesis. In such a setting, when we are interested in predicting the next observation, \textit{i.e.,} the system output, a series of novel approaches propose to relax this problem to online convex optimizations and provide guarantees in cumulative prediction error~\citep{hazan2017learning,arora2018towards,hazan2018spectral}.

\section{Conclusion and Future Work}
In this work, we study learning and controlling an unknown \LQG system. We propose \Alg, an exploration-exploitation algorithm that learns the model dynamics through interaction and designs a controller to minimize the average cost. \Alg consists of two phases, \textit{explore} phase and \textit{commit} phase. During the \textit{explore} phase, the agent deploys a Gaussian control to explore the system. Using the experiences gathered during this phase, the \Alg estimates the Markov parameters of the system and employs the Ho-Kalman method to estimate the underlying model parameters. The agent then constructs high probability confidence sets on the model parameters. During the \textit{commit} phase, the agent computes an optimistic model within the confidence sets and determines its corresponding optimal controller. The agent uses this controller to control the system. We illustrate that by deploying the optimistic controller for the remainder of the time, \textit{i.e.,} the commit time, the regret of the \Alg is upper bounded by $\tilde{\OO}(T^{2/3})$ for a total of $T$ interactions with the system. This result is the first sublinear end-to-end regret upper bound for controlling an unknown partially observable linear quadratic system. 


In future work, we plan to extend the theoretical tools developed in this work to the more interactive paradigm of adaptive control, where we are interested in agents that explore the system while simultaneously, exploit the estimates to minimize the overall costs. We plan to design a set of efficient adaptive control algorithms resulting in regret upper bounds of $\tilde{\OO}(\sqrt{T})$. 

\newpage
\section*{Acknowledgements}
S. Lale is supported in part by DARPA PAI. K. Azizzadenesheli is supported in part by Raytheon and Amazon Web Service. B. Hassibi is supported in part by the National Science Foundation under grants CNS-0932428, CCF-1018927, CCF-1423663 and CCF-1409204, by a grant from Qualcomm Inc., by NASA’s Jet Propulsion Laboratory through the President and Director’s Fund, and by King Abdullah University of Science and Technology. A. Anandkumar is supported in part by Bren endowed chair, DARPA PAIHR00111890035 and LwLL grants, Raytheon, Microsoft, Google, and Adobe faculty fellowships.

\bibliography{main}

\begin{thebibliography}{43}
\providecommand{\natexlab}[1]{#1}
\providecommand{\url}[1]{\texttt{#1}}
\expandafter\ifx\csname urlstyle\endcsname\relax
  \providecommand{\doi}[1]{doi: #1}\else
  \providecommand{\doi}{doi: \begingroup \urlstyle{rm}\Url}\fi

\bibitem[Abbasi-Yadkori and Szepesv{\'a}ri(2011)]{abbasi2011regret}
Yasin Abbasi-Yadkori and Csaba Szepesv{\'a}ri.
\newblock Regret bounds for the adaptive control of linear quadratic systems.
\newblock In \emph{Proceedings of the 24th Annual Conference on Learning
  Theory}, pages 1--26, 2011.

\bibitem[Abbasi-Yadkori et~al.(2011)Abbasi-Yadkori, P{\'a}l, and
  Szepesv{\'a}ri]{abbasi2011improved}
Yasin Abbasi-Yadkori, D{\'a}vid P{\'a}l, and Csaba Szepesv{\'a}ri.
\newblock Improved algorithms for linear stochastic bandits.
\newblock In \emph{Advances in Neural Information Processing Systems}, pages
  2312--2320, 2011.

\bibitem[Abbasi-Yadkori et~al.(2019)Abbasi-Yadkori, Lazic, and
  Szepesv{\'a}ri]{abbasi2019model}
Yasin Abbasi-Yadkori, Nevena Lazic, and Csaba Szepesv{\'a}ri.
\newblock Model-free linear quadratic control via reduction to expert
  prediction.
\newblock In \emph{The 22nd International Conference on Artificial Intelligence
  and Statistics}, pages 3108--3117, 2019.

\bibitem[Abeille and Lazaric(2017)]{abeille2017thompson}
Marc Abeille and Alessandro Lazaric.
\newblock Thompson sampling for linear-quadratic control problems.
\newblock \emph{arXiv preprint arXiv:1703.08972}, 2017.

\bibitem[Abeille and Lazaric(2018)]{abeille2018improved}
Marc Abeille and Alessandro Lazaric.
\newblock Improved regret bounds for thompson sampling in linear quadratic
  control problems.
\newblock In \emph{International Conference on Machine Learning}, pages 1--9,
  2018.

\bibitem[Arora et~al.(2018)Arora, Hazan, Lee, Singh, Zhang, and
  Zhang]{arora2018towards}
Sanjeev Arora, Elad Hazan, Holden Lee, Karan Singh, Cyril Zhang, and Yi~Zhang.
\newblock Towards provable control for unknown linear dynamical systems.
\newblock 2018.

\bibitem[{\AA}str{\"o}m(2012)]{aastrom2012introduction}
Karl~J {\AA}str{\"o}m.
\newblock \emph{Introduction to stochastic control theory}.
\newblock Courier Corporation, 2012.

\bibitem[Azizzadenesheli et~al.(2016)Azizzadenesheli, Lazaric, and
  Anandkumar]{azizzadenesheli2016reinforcement}
Kamyar Azizzadenesheli, Alessandro Lazaric, and Animashree Anandkumar.
\newblock Reinforcement learning of pomdps using spectral methods.
\newblock \emph{arXiv preprint arXiv:1602.07764}, 2016.

\bibitem[Bertsekas(1995)]{bertsekas1995dynamic}
Dimitri~P Bertsekas.
\newblock \emph{Dynamic programming and optimal control}, volume~2.
\newblock Athena scientific Belmont, MA, 1995.

\bibitem[Bittanti et~al.(2006)Bittanti, Campi, et~al.]{bittanti2006adaptive}
Sergio Bittanti, Marco~C Campi, et~al.
\newblock Adaptive control of linear time invariant systems: the “bet on the
  best” principle.
\newblock \emph{Communications in Information \& Systems}, 6\penalty0
  (4):\penalty0 299--320, 2006.

\bibitem[Campi and Kumar(1998)]{campi1998adaptive}
Marco~C Campi and PR~Kumar.
\newblock Adaptive linear quadratic gaussian control: the cost-biased approach
  revisited.
\newblock \emph{SIAM Journal on Control and Optimization}, 36\penalty0
  (6):\penalty0 1890--1907, 1998.

\bibitem[Chen et~al.(1992)Chen, Huang, Phan, and Juang]{chen1992integrated}
Chung-Wen Chen, Jen-Kuang Huang, Minh Phan, and Jer-Nan Juang.
\newblock Integrated system identification and state estimation for control
  offlexible space structures.
\newblock \emph{Journal of Guidance, Control, and Dynamics}, 15\penalty0
  (1):\penalty0 88--95, 1992.

\bibitem[Chen and Guo(1987)]{chen1987optimal}
Han-Fu Chen and Lei Guo.
\newblock Optimal adaptive control and consistent parameter estimates for armax
  model with quadratic cost.
\newblock \emph{SIAM Journal on Control and Optimization}, 25\penalty0
  (4):\penalty0 845--867, 1987.

\bibitem[Cohen et~al.(2019)Cohen, Koren, and Mansour]{cohen2019learning}
Alon Cohen, Tomer Koren, and Yishay Mansour.
\newblock Learning linear-quadratic regulators efficiently with only $
  \sqrt{T}$ regret.
\newblock \emph{arXiv preprint arXiv:1902.06223}, 2019.

\bibitem[Dean et~al.(2018)Dean, Mania, Matni, Recht, and Tu]{dean2018regret}
Sarah Dean, Horia Mania, Nikolai Matni, Benjamin Recht, and Stephen Tu.
\newblock Regret bounds for robust adaptive control of the linear quadratic
  regulator.
\newblock In \emph{Advances in Neural Information Processing Systems}, pages
  4188--4197, 2018.

\bibitem[Faradonbeh et~al.(2017)Faradonbeh, Tewari, and
  Michailidis]{faradonbeh2017optimism}
Mohamad Kazem~Shirani Faradonbeh, Ambuj Tewari, and George Michailidis.
\newblock Optimism-based adaptive regulation of linear-quadratic systems.
\newblock \emph{arXiv preprint arXiv:1711.07230}, 2017.

\bibitem[Faradonbeh et~al.(2018)Faradonbeh, Tewari, and
  Michailidis]{faradonbeh2018input}
Mohamad Kazem~Shirani Faradonbeh, Ambuj Tewari, and George Michailidis.
\newblock Input perturbations for adaptive regulation and learning.
\newblock \emph{arXiv preprint arXiv:1811.04258}, 2018.

\bibitem[Fiechter(1997)]{fiechter1997pac}
Claude-Nicolas Fiechter.
\newblock Pac adaptive control of linear systems.
\newblock In \emph{Annual Workshop on Computational Learning Theory:
  Proceedings of the tenth annual conference on Computational learning theory},
  volume~6, pages 72--80. Citeseer, 1997.

\bibitem[Garivier et~al.(2016)Garivier, Lattimore, and
  Kaufmann]{garivier2016explore}
Aur{\'e}lien Garivier, Tor Lattimore, and Emilie Kaufmann.
\newblock On explore-then-commit strategies.
\newblock In \emph{Advances in Neural Information Processing Systems}, pages
  784--792, 2016.

\bibitem[Hassibi et~al.(1999)Hassibi, Sayed, and
  Kailath]{hassibi1999indefinite}
Babak Hassibi, Ali~H Sayed, and Thomas Kailath.
\newblock \emph{Indefinite-Quadratic Estimation and Control: A Unified Approach
  to H2 and H-infinity Theories}, volume~16.
\newblock SIAM, 1999.

\bibitem[Hazan et~al.(2017)Hazan, Singh, and Zhang]{hazan2017learning}
Elad Hazan, Karan Singh, and Cyril Zhang.
\newblock Learning linear dynamical systems via spectral filtering.
\newblock In \emph{Advances in Neural Information Processing Systems}, pages
  6702--6712, 2017.

\bibitem[Hazan et~al.(2018)Hazan, Lee, Singh, Zhang, and
  Zhang]{hazan2018spectral}
Elad Hazan, Holden Lee, Karan Singh, Cyril Zhang, and Yi~Zhang.
\newblock Spectral filtering for general linear dynamical systems.
\newblock In \emph{Advances in Neural Information Processing Systems}, pages
  4634--4643, 2018.

\bibitem[Ho and K{\'a}lm{\'a}n(1966)]{ho1966effective}
BL~Ho and Rudolf~E K{\'a}lm{\'a}n.
\newblock Effective construction of linear state-variable models from
  input/output functions.
\newblock \emph{at-Automatisierungstechnik}, 14\penalty0 (1-12):\penalty0
  545--548, 1966.

\bibitem[Jaksch et~al.(2010)Jaksch, Ortner, and Auer]{jaksch2010near}
Thomas Jaksch, Ronald Ortner, and Peter Auer.
\newblock Near-optimal regret bounds for reinforcement learning.
\newblock \emph{Journal of Machine Learning Research}, 11\penalty0
  (Apr):\penalty0 1563--1600, 2010.

\bibitem[Juang et~al.(1993)Juang, Phan, Horta, and
  Longman]{juang1993identification}
Jer-Nan Juang, Minh Phan, Lucas~G Horta, and Richard~W Longman.
\newblock Identification of observer/kalman filter markov parameters-theory and
  experiments.
\newblock \emph{Journal of Guidance, Control, and Dynamics}, 16\penalty0
  (2):\penalty0 320--329, 1993.

\bibitem[Krahmer et~al.(2014)Krahmer, Mendelson, and
  Rauhut]{krahmer2014suprema}
Felix Krahmer, Shahar Mendelson, and Holger Rauhut.
\newblock Suprema of chaos processes and the restricted isometry property.
\newblock \emph{Communications on Pure and Applied Mathematics}, 67\penalty0
  (11):\penalty0 1877--1904, 2014.

\bibitem[Lai and Robbins(1985)]{lai1985asymptotically}
Tze~Leung Lai and Herbert Robbins.
\newblock Asymptotically efficient adaptive allocation rules.
\newblock \emph{Advances in applied mathematics}, 6\penalty0 (1):\penalty0
  4--22, 1985.

\bibitem[Lai and Wei(1987)]{lai1987asymptotically}
Tze~Leung Lai and Ching-Zong Wei.
\newblock Asymptotically efficient self-tuning regulators.
\newblock \emph{SIAM Journal on Control and Optimization}, 25\penalty0
  (2):\penalty0 466--481, 1987.

\bibitem[Lai et~al.(1982)Lai, Wei, et~al.]{lai1982least}
Tze~Leung Lai, Ching~Zong Wei, et~al.
\newblock Least squares estimates in stochastic regression models with
  applications to identification and control of dynamic systems.
\newblock \emph{The Annals of Statistics}, 10\penalty0 (1):\penalty0 154--166,
  1982.

\bibitem[Ledoux and Talagrand(2013)]{ledoux2013probability}
Michel Ledoux and Michel Talagrand.
\newblock \emph{Probability in Banach Spaces: isoperimetry and processes}.
\newblock Springer Science \& Business Media, 2013.

\bibitem[Ljung(1999)]{ljung1999system}
Lennart Ljung.
\newblock System identification.
\newblock \emph{Wiley Encyclopedia of Electrical and Electronics Engineering},
  pages 1--19, 1999.

\bibitem[Mania et~al.(2019)Mania, Tu, and Recht]{mania2019certainty}
Horia Mania, Stephen Tu, and Benjamin Recht.
\newblock Certainty equivalent control of lqr is efficient.
\newblock \emph{arXiv preprint arXiv:1902.07826}, 2019.

\bibitem[Ouyang et~al.(2017)Ouyang, Gagrani, and Jain]{ouyang2017learning}
Yi~Ouyang, Mukul Gagrani, and Rahul Jain.
\newblock Learning-based control of unknown linear systems with thompson
  sampling.
\newblock \emph{arXiv preprint arXiv:1709.04047}, 2017.

\bibitem[Oymak and Ozay(2018)]{oymak2018non}
Samet Oymak and Necmiye Ozay.
\newblock Non-asymptotic identification of lti systems from a single
  trajectory.
\newblock \emph{arXiv preprint arXiv:1806.05722}, 2018.

\bibitem[Pe{\~n}a et~al.(2009)Pe{\~n}a, Lai, and Shao]{pena2009self}
Victor~H Pe{\~n}a, Tze~Leung Lai, and Qi-Man Shao.
\newblock \emph{Self-normalized processes: Limit theory and Statistical
  Applications}.
\newblock Springer Science \& Business Media, 2009.

\bibitem[Phan et~al.(1994)Phan, Juang, Horta, and Longman]{phan1994system}
Minh Phan, Jer-Nan Juang, Lucas~G Horta, and Richard~W Longman.
\newblock System identification from closed-loop data with known output
  feedback dynamics.
\newblock \emph{Journal of guidance, control, and dynamics}, 17\penalty0
  (4):\penalty0 661--669, 1994.

\bibitem[Rudelson et~al.(2013)Rudelson, Vershynin, et~al.]{rudelson2013hanson}
Mark Rudelson, Roman Vershynin, et~al.
\newblock Hanson-wright inequality and sub-gaussian concentration.
\newblock \emph{Electronic Communications in Probability}, 18, 2013.

\bibitem[Sarkar et~al.(2019)Sarkar, Rakhlin, and Dahleh]{sarkar2019finite}
Tuhin Sarkar, Alexander Rakhlin, and Munther~A Dahleh.
\newblock Finite-time system identification for partially observed lti systems
  of unknown order.
\newblock \emph{arXiv preprint arXiv:1902.01848}, 2019.

\bibitem[Simchowitz et~al.(2018)Simchowitz, Mania, Tu, Jordan, and
  Recht]{simchowitz2018learning}
Max Simchowitz, Horia Mania, Stephen Tu, Michael~I Jordan, and Benjamin Recht.
\newblock Learning without mixing: Towards a sharp analysis of linear system
  identification.
\newblock \emph{arXiv preprint arXiv:1802.08334}, 2018.

\bibitem[Skelton and Shi(1994)]{skelton1994data}
Robert~E Skelton and Guojun Shi.
\newblock The data-based lqg control problem.
\newblock In \emph{Proceedings of 1994 33rd IEEE Conference on Decision and
  Control}, volume~2, pages 1447--1452. IEEE, 1994.

\bibitem[Tsiamis and Pappas(2019)]{tsiamis2019finite}
Anastasios Tsiamis and George~J Pappas.
\newblock Finite sample analysis of stochastic system identification.
\newblock \emph{arXiv preprint arXiv:1903.09122}, 2019.

\bibitem[Van Der~Vaart and Wellner(1996)]{van1996weak}
Aad~W Van Der~Vaart and Jon~A Wellner.
\newblock Weak convergence.
\newblock In \emph{Weak convergence and empirical processes}, pages 16--28.
  Springer, 1996.

\bibitem[Vershynin(2010)]{vershynin2010introduction}
Roman Vershynin.
\newblock Introduction to the non-asymptotic analysis of random matrices.
\newblock \emph{arXiv preprint arXiv:1011.3027}, 2010.

\end{thebibliography}
\bibliographystyle{plainnat}
\newpage
\appendix

\begin{center}
{\huge Appendix}
\end{center}
In the following, we first adapt the results in~\citet{oymak2018non} for the estimation error bound on $\hat{G}$. This analysis is provided in the Appendix~\ref{Proof Pieces}. 
An extra care in the probability expressions are required to fully adapt the results in~\citet{oymak2018non}. For this very reason, as well as sake of completeness, we provide the proof pieces that lead up to the Theorem \ref{error spectral}.


In the Appendix~\ref{SuppConfSet}, we describe the Ho-Kalman algorithm as well as the construction of the confidence sets which are used to compute the optimistic controller by deploying \OFU principle. The construction of confidence sets follow the same analysis of \citet{oymak2018non}, but we apply small changes in the derivation to obtain confidence sets required for \Alg.

In the Appendix~\ref{Ldifference}, we show that the steady state error covariance matrix of state estimation and Kalman gain of optimistically chosen system $\tth$ is concentrated around steady state error covariance matrix of state estimation and Kalman gain of $\Theta$.

In the Appendix~\ref{SuppRegretExplore} we provide an upper bound on the regret obtained during \textit{explore} phase. In the Appendix~\ref{SuppBounded}, we show that with the given exploration duration \eqref{explorationlowermain}, the state estimation obtained via optimistic parameters and the output of underlying system to the optimistic control input is bounded with high probability. In the Appendix~\ref{SuppBelmanOptimality}, we derive the Bellman optimality equation for \LQG and in the Appendix~\ref{SuppRegret} we provide regret decomposition for the optimistic controller, the optimal controller of the optimistic model chosen from the confidence sets. 

Finally, in the Appendix~\ref{SuppRegretTotal} we express the derivation of the regret upper bound of \Alg. The Appendix~\ref{Technical} contains technical theorems and lemmas used in the proofs.

\section{Markov Parameters Estimation, Proof of Theorem~\ref{error spectral}} \label{Proof Pieces}
Suppose the system $\Theta$ is stable (i.e.~$\rho(A)<1$) and $N\geq cH p \log(\frac{1}{\delta})$. We run the system with control input $u_t \sim \mathcal{N}(0,\sigma_u^2 I)$ for $T_{exp}= N+H-1$ time steps to collect $N$ samples of $H$ input-output pairs. From the definition of $\hat{G}$ in~\eqref{leastsquares} we have that 
\begin{equation*}
    (\hat{G}-G)^\top = (U^\top U)^{-1} U^\top \Gamma = (U^\top U)^{-1} \left(U^\top E + U^\top Z + U^\top W F^\top  \right)
\end{equation*}
Thus the spectral norm of $\hat{G}-G$ is bounded as 
\begin{equation}
    \|\hat{G}-G \| \leq \|(U^\top U)^{-1} \| \left(\| U^\top E \| +  \| U^\top Z \| + \| U^\top W \| \| F \| \right)
\end{equation}

\subsection{Characterization of \texorpdfstring{$U^\top U$}{Data Matrix}}

In the following, we characterize $U^\top U$ and provide a bound on $\|U\|$. 

\begin{lemma}[Characterization of Data Matrix~\citep{oymak2018non}]
 \label{lem cond} 
 Let $U\in\R^{N\times H p}$ be the input data matrix as described in \eqref{mat def}. For $\delta \in (0,1)$,  suppose the sample size obeys $N\geq c H p \log(\frac{1}{\delta})$ for sufficiently large constant $c>0$. Then, with probability at least $1-\delta$,
 
\[
2N\sigma_u^2\succeq U^\top U\succeq  N{\sigma_u^{2}/2}.
\]
\end{lemma}

\begin{proof} Let $r(v):\R^d\rightarrow\R^d$ be the circulant shift operator which maps a vector $v\in\R^d$ to its single entry circular rotation to the right i.e. $r(v)=[v_d~v_1~\dots~v_{d-1}]\in\R^d$. Let $\mathbf{C}\in\R^{T_{exp} p\times T_{exp} p}$ be a circulant matrix given below 
\small
\begin{equation*}
    \mathbf{C} = \left[
    \begin{array}{cccccc}{u_{T_{exp},1} \ldots u_{T_{exp},p}  } & {u_{T_{exp}-1,1} \ldots u_{T_{exp}-1,p} } & {\dots} & {u_{H,1} \ldots u_{H,p}  } & {\dots} & {u_{1,1} {\ldots} u_{1,p}  } \\ {u_{1,p} \ldots u_{T_{exp},p-1}  } & {u_{T_{exp},p} \ldots u_{T_{exp}-1,p-1} } & {\dots} & {u_{H-1,p} \ldots u_{H,p-1}  } & {\dots} & {u_{2,p} {\ldots} u_{1,p-1}  } \\ {\ddots} & {\ddots} & {\ddots} & {\ddots} & {\ddots} & {\ddots} \\ {u_{1,1}  {\ldots} u_{1,p}  } & {u_{T_{exp},1} \ldots u_{T_{exp},p} } & {\dots} & {u_{H+1,1} \ldots u_{H+1,p} } & {\dots} & {u_{2,1} u_{2,2} {\ldots} u_{2,p} } \\ {\ddots} & {\ddots} & {\ddots} & {\ddots} & {\ddots} & {\ddots} \\
    {u_{T_{exp},2}  \ldots u_{T_{exp}-1,1}  } & {u_{T_{exp}-1,2} \ldots u_{T_{exp}-2,1} } & {\dots} & {u_{H,2} \ldots u_{H-1,1} } & {\dots} & {u_{1,2}  {\ldots} u_{T_{exp},1} }\end{array} \right]
\end{equation*}
\normalsize

The first row is given by
\[
c_1=[u_{T_{exp}}^\top~u_{T_{exp}-1}^\top~\dots~u_2^\top~u_1^\top].
\]
Notice that the $i$th row of $\mathbf{C}$ is $c_i=r^{i-1}(c_1)$ for $1 \leq i\leq T_{exp}p$. Therefore, looking at the rightmost $Hp$ columns, one can observe that rightmost columns of $c_1$ gives 
\[
\left[ u_{H,1} u_{H,2} \ldots u_{H,p}  \dots u_{1,1} u_{1,2} \ldots u_{1,p}  \right] = \left[ u_H^\top u_{H-1}^\top \ldots u_1^\top \right] = \Bar{u}_H^\top
\]
Similarly, the rightmost $Hp$ columns of each $1+ip$th row of $\mathbf{C}$ gives $\Bar{u}_{H+i}$. From this one can deduce that $U$ is a submatrix of $\mathbf{C}$. Applying Theorem~\ref{circ thm}, setting $N_0 = c H p \log^2(2H p)\log^2(2T_{exp} p)$ and adjusting for variance $\sigma_u^2$, with probability at least $1-(2 T_{exp} p)^{-\log^2(2H p)\log(2T_{exp} p)}$, we have
\[
 2\sigma_u^2 I \succeq N^{-1} U^\top U\succeq  \sigma_u^{2}/2 I \implies 2N\sigma_u^2\succeq U^\top U\succeq  N{\sigma_u^{2}/2}
\]
whenever $N\geq N_0$.
\end{proof}

\subsection{Upper bound on \texorpdfstring{$\|U^\top Z \|$}{UZ}}

\begin{lemma}[Bound on  $\|U^\top Z\|$~\citep{oymak2018non}] \label{cor a2}
Let $U\in\R^{N\times H p}$ be the data matrix and let $Z \in\R^{N\times m}$ be the measurement noise matrix as described in (\ref{mat def}).  For $\delta \in (0,1/3)$, suppose $N\geq cH p \log(\frac{1}{\delta})$ for some absolute constant $c>0$. With probability at least $1-3\delta$,
\[
\|U^\top Z\|\leq 2\sigma_u\sigma_z\sqrt{N}\left(\sqrt{(Hp+m)} + \log(2Hp)\log(2\Texp p)\right).
\]
\end{lemma}
\begin{proof} Using $T_{exp}\geq N$, Lemma~\ref{lem cond} yields that 
\begin{equation}
    \Pr(\|U\|\leq \sqrt{2N}\sigma_u)\geq 1-\exp(-\log^2(2H p)\log^2(2T_{exp}p)).
\end{equation}

Suppose U have singular value decomposition of $U = V_1 \Sigma V_2^\top$ where $V_1 \in \R^{N \times Hp}$. Notice that $V_1^\top Z \in \R^{Hp \times m}$ has i.i.d. $\mathcal{N}(0,1)$ entries. Recall the following theorem and lemma:

Using Theorem \ref{Gordon}, we have $\mathbb{E}\left[\|V_1^\top Z\| \right] \leq \sqrt{Hp} + \sqrt{m} \leq \sqrt{2(Hp + m)}$. Since spectral norm is 1-Lipschitz function, Lemma~\ref{gausss_lip} implies that, with probability at least $1-2\exp(-t^2 /2 )-\exp(-\log^2(2H p)\log^2(2T_{exp}p))$,

\begin{equation*}
    \|U^\top Z \| = \|V_2 \Sigma V_1^\top Z \| = \| \Sigma V_1^\top Z \|  \leq \sqrt{2N}\sigma_u \sigma_z \left( \sqrt{2 (Hp + m)} + t \right)
\end{equation*}

Setting it for  $t=\sqrt{2}\log(2H p)\log(2T_{exp}p)$ and $\delta$ results the statement in the main Lemma.
\end{proof}

\subsection{Upper bound on \texorpdfstring{$\|U^\top W \|$}{UW}}

The main body of analysis of $\|U^\top W \|$, with an extra involvement from the associative Theorem~\ref{circ thm} is similar to Lemma~\ref{lem cond}.

\begin{lemma}[Bound on $\|U^\top W\|$~\citep{oymak2018non}] \label{lemmaUW}
Let $U\in\R^{N\times H p}$ be the data matrix and let $W \in\R^{N\times m}$ be the process noise matrix as described in \eqref{mat def}. Let $N_0 = c'H (p+n) \log^2(2H (p+n))\log^2(2\Texp(p+n))$ for some absolute constant $c'>0$. For $\delta \in (0,1)$, suppose $N\geq cH p \log(\frac{1}{\delta})$ for some absolute constant $c>0$. With probability at least $1-\delta$,
\[
\|U^\top W \| \leq \sigma_w \sigma_u \max \{\sqrt{N_0 N}, N_0 \}
\]

\end{lemma}
\begin{proof}
First, we define $m_t=[\sigma_u^{-1}u_t^\top~\sigma_w^{-1}w_t^\top]^\top \!\! \in \! \R^{p+n}$ and $\bar{m}_i=[m_{i}^\top ~m_{i-1}^\top ~\dots~m_{i-H+1}^\top]^\top \!\! \in \! \R^{H q}$. We also define the matrix $M = [\bar{m}_{H},~\dots~,\bar{m}_{H+N-1}]^\top \in\R^{N\times H (p+n)}$.
\\

\resizebox{0.96\linewidth}{!}{
    $M = \left[
    \begin{array}{cccccccc}{\sigma_u^{-1} u_{H}^\top} \enskip {  \sigma_w^{-1} w_{H}^\top  } & {\sigma_u^{-1} u_{H-1}^\top } \enskip { \sigma_w^{-1} w_{H-1}^\top} & {\dots} & {\sigma_u^{-1} u_{2}^\top } \enskip { \sigma_w^{-1} w_{2}^\top} & {\dots} & {\sigma_u^{-1} u_{1}^\top  } \enskip {\sigma_w^{-1} w_{1}^\top  } \\ {\sigma_u^{-1} u_{H+1}^\top } \enskip  { \sigma_w^{-1} w_{H+1}^\top  } & {\sigma_u^{-1} u_{H}^\top } \enskip { \sigma_w^{-1} w_{H}^\top} & {\dots} & {\sigma_u^{-1} u_{3}^\top } \enskip { \sigma_w^{-1} w_{3}^\top} & {\dots} & {\sigma_u^{-1} u_{2}^\top } \enskip { \sigma_w^{-1} w_{2}^\top} \\ {\ddots} & {\ddots} & {\ddots} & {\ddots}  \\ {\sigma_u^{-1} u_{H+N-1}^\top} \enskip {  \sigma_w^{-1} w_{H+N-1}^\top  } & {\sigma_u^{-1} u_{H+N-2}^\top } \enskip { \sigma_w^{-1} w_{H+N-2}^\top} & {\dots} & {\sigma_u^{-1} u_{N+1}^\top } \enskip { \sigma_w^{-1} w_{N+1}^\top} & {\dots} & {\sigma_u^{-1} u_{N}^\top  } \enskip {\sigma_w^{-1} w_{N}^\top  } \end{array} \right] 
     $} \\
Observe that by construction, $\sigma_u^{-1}U,\sigma_w^{-1}W$ are submatrices of $M$. In particular, $(\sigma_u\sigma_w)^{-1}U^\top W$ is an $H p\times H n$ size off-diagonal submatrix of $M^\top M$. This is due to the facts that i) $\sigma_u^{-1}U$ is a submatrix of $M$ characterized by the column indices 
\[
\{(i-1)(p+n)+j \Big | 1\leq i\leq H,~1\leq j\leq p\},
\]
and ii) $\sigma_w^{-1}W$ lies at the complementary columns. Since $(\sigma_u\sigma_w)^{-1}U^\top W$ is an off-diagonal submatrix of $M^\top M$, it is also a submatrix of $M^\top M - I$ and spectral norm of a submatrix is upper bounded by the norm of the original matrix. Thus 
\begin{equation}
    (\sigma_u\sigma_w)^{-1}\| U^\top W\|\leq \|M^\top M-N I\|.\label{uw bound}
\end{equation}
Finally, we will embed $M$ in a circulant matrix. Let $r(v):\R^d\rightarrow\R^d$ be the circulant shift operator which maps a vector $v\in\R^d$ to its single entry circular rotation to the right i.e. $r(v)=[v_d~v_1~\dots~v_{d-1}]\in\R^d$. Let $\mathbf{C}\in\R^{T_{exp}(p+n)\times T_{exp}(p+n) }$ be a circulant matrix given below  

\resizebox{0.94\linewidth}{!}{
    $\mathbf{C} = \left[
    \begin{array}{cccccc}{m_{T_{exp} ,1} \ldots m_{T_{exp},p+n}  } & {m_{T_{exp}-1,1} \ldots m_{T_{exp}-1,p+n} } & {\dots} & {m_{H,1} \ldots m_{H,p+n}  } & {\dots} & {m_{1,1} {\ldots} m_{1,p+n}  } \\ {m_{1,p+n} \ldots m_{T_{exp},p+n-1}  } & {m_{T_{exp},p+n} \ldots m_{T_{exp}-1,p+n-1} } & {\dots} & {m_{H-1,p+n} \ldots m_{H,p+n-1}  } & {\dots} & {m_{2,p+n} {\ldots} m_{1,p+n-1}  } \\ {\ddots} & {\ddots} & {\ddots} & {\ddots} & {\ddots} & {\ddots} \\ {m_{1,1} {\ldots} m_{1,p+n}  } & {m_{T_{exp},1} \ldots m_{T_{exp},p+n} } & {\dots} & {m_{H+1,1} \ldots m_{H+1,p+n} } & {\dots} & {m_{2,1}  {\ldots} m_{2,p+n} } \\ {\ddots} & {\ddots} & {\ddots} & {\ddots} & {\ddots} & {\ddots} \\
    {m_{T_{exp},2}  \ldots m_{T_{exp}-1,1}  } & {m_{T_{exp}-1,2}  \ldots m_{T_{exp}-2,1} } & {\dots} & {m_{H,2}  \ldots m_{H-1,1} } & {\dots} & {m_{1,2} {\ldots} m_{T_{exp},1} }\end{array} \right]
    $}
    
The first row is given by
\[
c_1=[m_{T_{exp}}^\top~m_{T_{exp}-1}^\top~\dots~m_2^\top~m_1^\top].
\]
Notice that the $i$th row of $\mathbf{C}$ is $c_i=r^{i-1}(c_1)$ for $1 \leq i\leq T_{exp}(p+n)$. Therefore, looking at the rightmost $H(p+n)$ columns, one can observe that 
rightmost columns of $c_1$ gives 
\[
\left[ m_{H,1} m_{H,2} \ldots m_{H,p+n}  \dots m_{1,1} m_{1,2} \ldots m_{1,p+n}  \right] = \left[ m_H^\top m_{H-1}^\top \ldots m_1^\top \right] = \Bar{m}_H^\top
\]
Similarly, rightmost $H(p+n)$ columns of each $1+i(p+n)$th row of $\mathbf{C}$ gives $\Bar{m}_{H+i}$. From this one can deduce that $M$ is a submatrix of $\mathbf{C}$.

Applying Theorem~\ref{circ thm}, setting $N_0 = c' H (p+n) \log^2(2H (p+n))\log^2(2T_{exp} (p+n))$ and adjusting for variance $\sigma_u^2$, with probability at least $1-(2 T_{exp} (p+n))^{-\log^2(2H (p+n))\log(2T_{exp} (p+n))}$, we have that 
\[
\|\frac{1}{N} M^\top M - I \| \leq \max \{\sqrt{\frac{N_0}{N}}, \frac{N_0}{N} \}
\]
which implies that $\| U^\top W \| \leq \sigma_w \sigma_u \max \{ \sqrt{N_0 N}, N_0 \}$
via inequality \eqref{uw bound}.

\end{proof}

\subsection{Upper bound on \texorpdfstring{$\|U^\top E \|$}{UE}}
\begin{lemma}[Inner product of a fixed vector and state vector \citep{oymak2018non}]\label{thm e bound}

Let $\E_t\in\R^{N_t\times m}$ be the matrix composed of the rows $e_{t+iH}=CA^{H-1}x_{t+1+iH}$. Define 
\[
\gamma=\frac{\|\Gamma_\infty\| \Phi(A)^2 \|CA^{H-1}\|^2}{1-\rho(A)^{2H}}.
\]
Given a unit length vector $a \in \R^m$, for all $\tau\geq 2$ and for some absolute constant $c>0$, we have that
\[
\Pr(\|{E_ta}\|_2^2\geq \tau N_t\gamma)\leq  2exp(-c\tau N_t(1-\rho(A)^{H})).
\]
\end{lemma} 
\begin{proof}
In the proof, the authors consider an alternative way to rewrite $\| E_t a\|_2$ to apply Hanson-Wright Theorem~\citep{rudelson2013hanson}. For further details of the proof please refer to Lemma D.6 of~\citep{oymak2018non}. 
\end{proof}

\begin{theorem}[Bound on Decomposition elements $U_t^\top E_t$~\citep{oymak2018non}]\label{decomposed bound}
Let \[
U_t=[ \bar{u}_{t+H},~ \bar{u}_{t+2H},~\dots,~\bar{u}_{t+N_tH}]^\top \quad E_t=[ e_{t+H},~ e_{t+2H},~\dots,~e_{t+N_tH}]^\top. \] 
Define $\gamma=\frac{\|\Gamma_\infty\| \Phi(A)^2 \|CA^{H-1}\|^2}{1-\rho(A)^{2H}}$. For $\delta \in (0,1)$, $U_t^\top E_t$ obeys 
\[
\|U_t^\top E_t\| \leq  c_0\sigma_u\sqrt{\tau \log\left(\frac{1}{\delta}\right)  N_t \gamma},
\] with probability at least $1-\delta-2\exp(-c\tau N_t(1-\rho(A)^H)+3m)$ for $\tau\geq 1$ and for some absolute constants $c_0$ and $c$. 
\end{theorem}

\begin{proof}
For these matrices, define the filtrations $\mathcal{F}_i = \sigma\left( \{u_j, w_j\}^{t+iH}_{j=1} \}\right)$ for $1\leq i \leq N_t$. Based on this definition, $e_{t+iH} \in \mathcal{F}_{i-1}$ and $u_{t+iH} \in \mathcal{F}_{i} $ and independent of $\mathcal{F}_{i-1}$. This gives a formulation to use Lemma \eqref{sig sub}. Combining Lemma~\ref{sig sub} and Lemma~\ref{thm e bound} within the Covering Lemma~\eqref{cover bound}, one can derive the result. For the details of the proof please refer to Theorem D.2 of~\citep{oymak2018non}. In order to get the presented result, one needs to pick $t = c' \sigma_u \sqrt{\tau \log^2(2Hp)\log^2(2\Texp p)  N_t \gamma}$  in using Lemma~\ref{sig sub}, which will translate to given bound with the stated probability. 
\end{proof}

\begin{corollary} \label{preciseboundUtEt}
For $\delta \in (0,1)$, suppose $N\geq cH p \log(\frac{1}{\delta})$ for some absolute constant $c>0$. With probability at least $1-\frac{3\delta}{H}$, we have that 
\[ 
\|U_t^\top E_t\| \leq c' \sigma_u \sqrt{ \gamma \log\left(\frac{H}{\delta}\right)  \max \left \{\frac{N}{H}, \frac{3m + \log\left(\frac{H}{\delta} \right)}{1-\rho(A)^{H}} \right\}  } 
\]
for some constant $c' > 0$. 
\end{corollary}
\begin{proof}
The given choice of $N$ implies that $N_t \geq \frac{N}{2H}$. Picking $\tau = \max \left\{ \frac{2H \left(3m + \log \left( \frac{1}{\delta}\right) \right)}{cN(1-\rho(A)^H)}, 1 \right\}$ to use in Theorem~\ref{decomposed bound} gives the statement of corollary. 
\end{proof}

\begin{theorem}[Bound on $\|U^\top E \|$~\citep{oymak2018non}] \label{ub eb prod} 
For $\delta \in (0,1/3)$, suppose $N\geq cH p \log(\frac{1}{\delta})$ for some absolute constant $c>0$. Let $U\in\R^{N\times H p}$ be the data matrix and let $E \in\R^{N\times m}$ be the unknown state contribution matrix as described in \eqref{mat def}. Define $\gamma=\frac{\|\Gamma_\infty\| \Phi(A)^2 \|CA^{H-1}\|^2}{1-\rho(A)^{2H}}$. Then, with probability at least $1-3\delta$, 
\[
\|U^\top E\|\leq c \sigma_u H \sqrt{ \gamma \log\left(\frac{H}{\delta}\right)  \max \left \{\frac{N}{H}, \frac{3m + \log\left(\frac{H}{\delta} \right)}{1-\rho(A)^{H}} \right\}  } .
\]
\end{theorem}
\begin{proof}
First decompose $U^\top E=\sum_{t=H}^{T_{exp}} \bar{u}_t e_t^\top$ into sum of $H$ smaller products. We form the matrices 
\[
U_t=[ \bar{u}_{t+H},~ \bar{u}_{t+2H},~\dots,~\bar{u}_{t+N_tH}]^\top \in \R^{N_t \times Hp} \quad E_t=[ e_{t+H},~ e_{t+2H},~\dots,~e_{t+N_tH}]^\top \in \R^{N_t \times m}. \] 
Then, $U^\top E$ is decomposed as
\begin{equation}\label{sum spect}
U^\top E=\sum_{t=0}^{H-1}U_t^\top E_t \implies \|U^\top E\|\leq \sum_{t=0}^{H-1}\|U_t^\top E_t\|.
\end{equation}
Using the union bound with Corollary~\ref{preciseboundUtEt}, stated bound is obtained. 
\end{proof}

The final statement of the Theorem~\ref{error spectral} follows by combining Lemma~\ref{lem cond}, Theorem~\ref{ub eb prod}, Lemma~\ref{cor a2} and Lemma~\ref{lemmaUW} using union bound considering the probability of events happening. In order to get the expression of $R_e$ notice that $H \gamma = \sigma_e^2$.

\section{Confidence Set Construction for the System Parameters}
\label{SuppConfSet}
After estimating $\hat{G}$, we construct the high probability confidence sets for the unknown system parameters. We exploit these sets when we deploy the \OFU principle for the controller synthesis. \Alg uses Ho-Kalman method~\citep{ho1966effective} to estimate the system parameters. In this section, we first describe estimation components of the Ho-Kalman method, the  Algorithm~\ref{kalmanho}. We then derive the confidence sets around the system parameter estimates in \Alg using the results developed in the Theorem~\ref{error spectral}.

\subsection[Ho-Kalman Algorithm]{Ho-Kalman Algorithm~\citep{ho1966effective}} \label{kalmanhosupp}
The Ho-Kalman algorithm, Algorithm \ref{kalmanho},  takes the Markov parameter matrix estimate $\hat{G}$, $H$, the systems order $n$ and dimensions $d_1, d_2$, as the input. It is worth restating that the dimension of latent state, $n$, is the order of the system for observable and controllable dynamics. With the assumption that $H\geq2n+1$, we pick $d_1 \geq n$ and $d_2 \geq n$ such $d_1+d_2+1 = H$. This guarantees that the system identification problem is well-conditioned, \textit{i.e.} $\mathbf{H}$ is rank-$n$.

\begin{algorithm}[tbh] 
 \caption{Ho-Kalman Algorithm}
  \begin{algorithmic}[1] 
  \STATE {\bfseries Input:} $\hat{G}$, $H$, system order $n$, $d_1, d_2$ such that $d_1 + d_2 + 1 = H$ \\
  \STATE Form the Hankel Matrix $\mathbf{\hat{H}} \in \R^{md_1 \times p(d_2+1)}$ from $\hat{G}$
    \STATE Set $\mathbf{\hat{H}}^- \in \R^{m d_1 \times p d_2} \enskip \text{ as the first $pd_2$ columns of }\mathbf{\hat{H}}$ 
    \STATE Using SVD obtain $\mathbf{\hat{N}} \in \R^{m d_1 \times p d_2}$ \enskip, the rank-$n$ approximation of $\mathbf{\hat{H}}^-$
    \STATE Obtain  $\mathbf{U},\mathbf{\Sigma},\mathbf{V} = \text{SVD}(\mathbf{\hat{N}})$
    \STATE Construct $\mathbf{\hat{O}} = \mathbf{U}\mathbf{\Sigma}^{1/2} \in \R^{md_1 \times n}$
    \STATE Construct $\mathbf{\hat{C}} = \mathbf{\Sigma}^{1/2}\mathbf{V} \in \R^{n \times pd_2}$
    \STATE Obtain $\hat{C}\in \R^{m\times n}$, the first $m$ rows of $\mathbf{\hat{O}}$
    \STATE Obtain $\hat{B}\in \R^{n\times p}$, the first $p$ columns of $\mathbf{\hat{C}}$
    \STATE Obtain $\mathbf{\hat{H}}^+ \in \R^{m d_1 \times p d_2} \enskip \text{, the last $pd_2$ columns of}(\mathbf{\hat{H}})$ 
    \STATE Obtain $\hat{A} = \mathbf{\hat{O}}^\dagger \mathbf{\hat{H}}^+ \mathbf{\hat{C}}^\dagger \in \R^{n\times n}$
  \end{algorithmic}
 \label{kalmanho}  
\end{algorithm}

Recall that system parameters can be learned up to similarity transformation, \textit{i.e.} for any invertible $\mathbf{T}\in \R^{n \times n}$, $A' = \mathbf{T}^{-1}A\mathbf{T}, B' = \mathbf{T}^{-1}B, C' = C \mathbf{T} $ gives the same Markov parameters as $G$ so it's a valid realization. Note that the similarity transformations have bounded norms due to Assumption \ref{AssumContObs}. Given $\hat{G} = [\hat{G}_1 \ldots \hat{G}_H] \in \R^{m \times Hp}$, where $\hat{G}_i$ is the $i$'th $m \times p$ block of $\hat{G}$, for all $1 \leq i \leq H$, the algorithm constructs $d_1 \times (d_2+1)$ Hankel matrix $\mathbf{\hat{H}}$ such that $(i,j)$th block of Hankel matrix is $\hat{G}_{(i+j)}$. Notice that if the input to the algorithm was $G$ then constructed Hankel matrix, $\mathbf{H}$ would be rank $n$, where 
\begin{align*}
     \mathbf{H} =  [C^\top~(CA) ^\top    \ldots (CA^{n-1}) ^\top] ^\top [B~AB\ldots A^{n}B]=\mathbf{O} [\mathbf{C} ~~A^nB]=\mathbf{O} [B ~~A\mathbf{C}]
\end{align*}
The matrices $\mathbf{O}$ and $\mathbf{C}$ are observability and controllability matrices respectively. Essentially, the Ho-Kalman algorithm estimates these matrices using $\hat{G}$. In order to obtain these estimates, the algorithm constructs $\mathbf{\hat{H}}^-$ by taking first $pd_2$ columns of $\mathbf{\hat{H}}$ and calculates $\mathbf{\hat{N}}$ which is the best rank-$n$ approximation of $\mathbf{\hat{H}}^-$. Singular value decomposition of $\mathbf{\hat{N}}$ gives the estimates of $\mathbf{O},\mathbf{C}$, \textit{i.e.} $\mathbf{\hat{N}} = \mathbf{U}\mathbf{\Sigma}^{1/2}~ \mathbf{\Sigma}^{1/2}\mathbf{V} =\mathbf{\hat{O}}\mathbf{\hat{C}} $. From these estimates, the algorithm recovers $\hat{B}$ as the first $n\times p$ block of $\mathbf{\hat{C}}$, $\hat{C}$ as the first $m \times n$ block of $\mathbf{\hat{O}}$ and $\hat{A}$ as $\mathbf{\hat{O}}^\dagger \mathbf{\hat{H}}^+ \mathbf{\hat{C}}^\dagger$ where $\mathbf{\hat{H}}^+$ is the submatrix of $\mathbf{\hat{H}}$ obtained by discarding left-most $md_1 \times p$ block.    

\subsection{Confidence sets around \texorpdfstring{$\hat{A}, \hat{B}, \hat{C}$}{estimations}}

The results in this section are adopted and modified versions of those in~\citep{oymak2018non}. Except the Lemma \ref{hokalmanstability lemma}, which is directly from the mentioned work, we make a small change to the presentation of it so that we can observe the dependency on the duration of exploration period and the construction of confidence set around the estimates. In this section, for completeness we provide this lemma that is used in the proof of theorem. The proof uses a simple singular value perturbation arguments. For more details of the proof, please refer to~\citep{oymak2018non}.

\begin{lemma}[\citep{oymak2018non}] \label{hokalmanstability lemma}
$\mathbf{H}$, $\mathbf{\hat{H}}$ and $\mathbf{N}, \mathbf{\hat{N}}$ satisfies the following perturbation bounds,

\begin{align*} 
\max \left\{\left\|\mathbf{H}^{+}-\mathbf{\hat{H}}^{+}\right\|,\left\|\mathbf{H}^{-}-\mathbf{\hat{H}}^{-}\right\|\right\} \leq \|\mathbf{H}-\mathbf{\hat{H}}\| &\leq \sqrt{\min \left\{d_{1}, d_{2}+1\right\}}\|\hat{G} - G\| \\ \|\mathbf{N}-\mathbf{\hat{N}}\| \leq 2\left\|\mathbf{H}^{-}-\mathbf{\hat{H}}^{-}\right\| &\leq 2 \sqrt{\min \left\{d_{1}, d_{2}\right\}}\|\hat{G} - G\|
\end{align*}
\end{lemma}

Denote the $n$th largest singular value of $\mathbf{N}$, \textit{i.e.} smallest nonzero singular value, as $\sigma_{n}(\mathbf{N})$. It is worth noting that if $\sigma_{n}(\mathbf{N})$ is large enough, the order of the system can be estimated via singular value thresholding of $\mathbf{\hat{H}}^{-}$. 

Given the Lemma~\ref{hokalmanstability lemma}, we can show the following theorem:

\begin{theorem}
\label{Sup:ConfidenceSets}
Suppose $\mathbf{H}$ is the rank-$n$ Hankel matrix obtained from $G$. Let $\bar{A}, \bar{B}, \bar{C}$ be the system parameters that Ho-Kalman algorithm provides for $G$. Suppose the system is order $n$ and it is observable and controllable. Define the rank-$n$ matrix $\mathbf{N}$ such that it is the submatrix of $\mathbf{H}$ obtained by discarding the last block column of $\mathbf{H}$. Suppose $\sigma_{n}(\mathbf{N}) > 0$ and $\| \mathbf{\hat{N}} - \mathbf{N} \| \leq \frac{\sigma_{n}(\mathbf{N})}{2} $. Then, there exists a unitary matrix $\mathbf{T} \in \R^{n \times n}$ such that, $\bar{\Theta}=(\bar{A}, \bar{B}, \bar{C}) \in (\mathcal{C}_A \times \mathcal{C}_B \times \mathcal{C}_C) $ for
\begin{align}
    \mathcal{C}_A &= \left \{A' \in \R^{n \times n} : \|\hat{A} - \mathbf{T}^\top A' \mathbf{T} \| \leq \left( \frac{31n\|\mathbf{H}\|}{\sigma_n^2(\mathbf{H})} + \frac{13n}{2\sigma_n(\mathbf{H})}  \right) \|\hat{G} - G \| \right \} \\
    \mathcal{C}_B &= \left \{B' \in \R^{n \times p} : \|\hat{B} - \mathbf{T}^\top B' \| \leq  \frac{7n}{\sqrt{\sigma_n(\mathbf{H})}} \|\hat{G} - G \| \right\} \\
    \mathcal{C}_C &= \left\{C' \in \R^{m \times n} : \|\hat{C} -   C' \mathbf{T} \| \leq \frac{7n}{\sqrt{\sigma_n(\mathbf{H})}} \|\hat{G} - G \| \right\} 
\end{align}
where $\hat{A}, \hat{B}, \hat{C}$ obtained from Ho-Kalman algorithm with using the least squares estimate of the Markov parameter matrix $\hat{G}$.  
\end{theorem}

\begin{proof}
The proof is similar to proof of Theorem 4.3 in \cite{oymak2018non}. Difference in the presentation arises due to providing different characterization of the dependence on $\|\mathbf{N} - \mathbf{\hat{N}} \|$ and centering the confidence ball over the estimations rather than the output of Ho-Kalman algorithm with the input of $G$. In~\citet{oymak2018non}, from the inequality
\[
\|\bar{B} - \mathbf{T}^\top \hat{B} \|_F^2 \leq \frac{2n \| \mathbf{N} - \mathbf{\hat{N}}\|^2}{(\sqrt{2}-1)\left(\sigma_n(\mathbf{N}) - \| \mathbf{N} - \mathbf{\hat{N}}\| \right)},
\]
the authors use the assumption $\| \mathbf{N} - \mathbf{\hat{N}}\| \leq \frac{\sigma_{n}(\mathbf{N})}{2}$ to cancel out numerator and denominator. In this presentation, we define $T_N$ such that $\| \mathbf{N} - \mathbf{\hat{N}}\| \leq \frac{\sigma_{n}(\mathbf{N})}{2}$ holds with high probability as long as $\Texp \geq T_N$. In the definition of $T_N$, we use $\sigma_n(H)$, due to the fact that singular values of submatrices by column partitioning are interlaced, \textit{i.e.} $\sigma_n(\mathbf{N}) = \sigma_n(\mathbf{H}^-) \geq \sigma_n(\mathbf{H})$. Then, we define the denominator based on $\sigma_{n}(\mathbf{N})$ and again use the fact $\sigma_n(\mathbf{N}) = \sigma_n(\mathbf{H}^-) \geq \sigma_n(\mathbf{H})$. Following the proof steps provided in~\citet{oymak2018non} and combining Lemma~\ref{hokalmanstability lemma} with the fact that one can choose either $d_1=n$ or $d_2=n$ while $d_1,d_2 \geq n$, we obtain the presented theorem. 
\end{proof}

\noindent \textbf{Proof of Theorem \ref{ConfidenceSets}:}
\\
This result is obtained by combining Theorem~\ref{error spectral} and Theorem~\ref{Sup:ConfidenceSets}
\null\hfill$\square$

\section[Proof of Lemma \ref{StabilityCov}]{Upper bound on $\|\tsig - \mathbf{S}^{-1} \sig \mathbf{S} \|$ and $\|\tl - \mathbf{S}^{-1} L \|$, Proof of Lemma \ref{StabilityCov}} \label{Ldifference}
In this section, we provide the concentration results on $\|\tsig - \mathbf{S}^{-1}\sig\mathbf{S} \|$ and $\|\tl - \mathbf{S}^{-1}L \|$.
$\mathbf{S} \in \R^{n \times n} $ is a similarity transformation that is composed of two similarity transformations. The first one takes the system $\Theta$ and transforms to $\bar{\Theta}$, the output of Ho-Kalman algorithm. The second similarity transformation is the unitary matrix that is proven to exist in Theorem \ref{ConfidenceSets}. We deploy the fixed point argument from~\citep{mania2019certainty} to bound $\|\tsig - \mathbf{S}^{-1} \sig \mathbf{S} \|$. Then we utilize the resulting bounds to come up with the concentration of $\|\tl - \mathbf{S}^{-1} L \|$. \\




\noindent \textbf{Proof of Lemma \ref{StabilityCov}:}
\\
For the simplicity of the presentation of the proof, without loss of generality, let $A = \mathbf{S}^{-1} A \mathbf{S}$, $C = C \mathbf{S}$, \textit{i.e.} assume that $\mathbf{S} = I$. Given parameters $(A,C,\sigma_w^2I, \sigma_z^2 I)$, define $F(X,A,C)$ such that 
\begin{align*}
    F(X,A,C) &= X - A X A^\top + A X C^\top \left(C X C^\top + \sigma_z^2I \right)^{-1} C X A^\top - \sigma_w^2 I \\
    &= X - A (I + \sigma_z^{-2} X C^\top C)^{-1} X A^\top - \sigma_w^2 I 
\end{align*}
where last equality follows from matrix inversion lemma.  Moreover, notice that solving algebraic Riccati equation for steady state error covariance matrix of state estimation for $(A,C,\sigma_w^2I, \sigma_z^2 I)$ is equivalent to finding the unique positive definite solution to $X$ such that $F(X,A,C) = 0$. The solution for the underlying system $\Theta$, $F(X,A,C) = 0$, is denoted as $\Sigma$ and the solution for the optimistic system $\tth$ chosen from the set $(\mathcal{C}_A \times \mathcal{C}_B \times \mathcal{C}_C) \cap \mathcal{S}$, $F(X,\ta,\tc) = 0$, is denoted as $\tsig$. Denote $D_{\Sigma} = \tsig - \Sigma$ and $M = A(I-LC)$. Recall that $L = \Sigma C^\top \left( C \Sigma C^\top + \sigma_z^2 I \right)^{-1}$. For any matrix $X$ such that $I + (\Sigma + X)(\sigma_z^{-2} C^\top C)$ is invertible we have
\begin{equation} \label{fixedidentity}
    F(\Sigma + X, A,C) = X - M X M^\top + M X (\sigma_z^{-2} C^\top C) [I + (\Sigma + X) (\sigma_z^{-2} C^\top C)]^{-1} X M^\top.
\end{equation}

One can verify the identity by adding $F(\Sigma, A,C) = 0$ to the right hand side of \eqref{fixedidentity} and use the identity that $M = A(I-LC) = A(I + \sigma_z^{-2}\Sigma C^\top C)^{-1} = A (I - \Sigma C^\top(C\Sigma C^\top + \sigma_z^2 I)^{-1}C)$. Define two operators $\mathcal{T}(X)$, $\mathcal{H}(X)$ such that $\mathcal{T}(X) = X - M X M^\top $ and $\mathcal{H}(X) = M X (\sigma_z^{-2} C^\top C) [I + (\Sigma + X) (\sigma_z^{-2} C^\top C)]^{-1} X M^\top$. Thus,
\begin{equation*}
    F(\Sigma + X, A,C) = \mathcal{T}(X) + \mathcal{H}(X)
\end{equation*}
Notice that since \eqref{fixedidentity} is satisfied for any $X$ such that $I + (\Sigma + X)(\sigma_z^{-2} C^\top C)$ is invertible,

\begin{equation} \label{fixeddiff}
    F(\Sigma + X, A,C) - F(\Sigma + X, \ta,\tc) = \mathcal{T}(X) + \mathcal{H}(X)
\end{equation}
has a unique solution $X = D_{\Sigma}$ where $\Sigma + D_{\Sigma} \succeq 0$.

Recall that $M$ is stable. Therefore, the linear map $\mathcal{T} : X \mapsto X - M X M^\top$ has non-zero eigenvalues, \textit{i.e} $\mathcal{T}$ is invertible. Using this, define the following operator,

\begin{equation*}
    \Psi(X) = \mathcal{T}^{-1} \left( F(\Sigma + X, A,C) - F(\Sigma + X, \ta,\tc)  - \mathcal{H}(X)\right). 
\end{equation*}

Notice that solving for $X$ in \eqref{fixeddiff} is equivalent to solving for $X$ that satisfies $\Sigma + X \succeq 0$ and $\Psi(X) = X$. This shows that $\Psi(X)$ has a unique fixed point X that is $D_{\Sigma}$. Consider the set 

\begin{equation}
    \mathcal{S}_{\Sigma,\beta} = \{X : \|X\|\leq \beta, X = X^\top, \Sigma + X \succeq 0 \}. 
\end{equation}

Let $X \in \mathcal{S}_{\Sigma,\beta}$ for $\beta < \sigma_n(\Sigma)/2$. First of all, recalling Assumption~\ref{Stabilizable set}, notice that operator norm of $\mathcal{T}^{-1}$ is upper bounded as $\| \mathcal{T}^{-1}\| \leq \frac{1}{1-\upsilon^2}$. Using Lemma~\ref{normwoodbury}, we get $\|\mathcal{H}(X)\| \leq \sigma_z^{-2}\upsilon^2 \|X\|^2 \|C\|^2 \leq \sigma_z^{-2}\upsilon^2 \beta^2 \|C\|^2$.  Now consider $ F(\Sigma + X, A,C) - F(\Sigma + X, \ta,\tc)$:
\begin{align}
     &F(\Sigma + X, \ta,\tc) - F(\Sigma + X, A,C)  \\
     &= A (I + \sigma_z^{-2} (\Sigma + X) C^\top C)^{-1} (\Sigma + X) A^\top \!-\! \ta (I + \sigma_z^{-2} (\Sigma + X) \tc^\top \tc)^{-1} (\Sigma + X) \ta^\top \nonumber \\
     &= A(I + \sigma_z^{-2} (\Sigma + X) \tc^\top \tc)^{-1} (\Sigma + X) \sigma_z^{-2}(C^\top C \!-\! \tc^\top \tc) (I \!+\! \sigma_z^{-2} (\Sigma \!+\! X) C^\top C)^{-1}(\Sigma \!+\! X) A^\top \nonumber \\
     &- (\ta \!-\! A)(I \!+\! \sigma_z^{-2} (\Sigma \!+\! X) \tc^\top \tc)^{-1}(\Sigma \!+\! X) A^\top \!-\! A (I + \sigma_z^{-2} (\Sigma \!+\! X) \tc^\top \tc)^{-1} (\Sigma \!+\! X) (\ta \!-\! A)^\top \nonumber \\
     &- (\ta - A) (I + \sigma_z^{-2} (\Sigma + X) \tc^\top \tc)^{-1} (\Sigma + X) (\ta - A)^\top \label{Fparts}
\end{align}
Using, Lemma~\ref{normwoodbury} and the fact that $X \in \mathcal{S}_{\Sigma,\beta}$,
\begin{align}
    &\|F(\Sigma\!+\!\!X, \ta,\tc)\!-\!F(\Sigma\!+\!\!X,A,C) \|\!\\
    &\leq\! \sigma_z^{-2}\Phi(A)^2\|\Sigma + X\|^2 \| C^\top C - \tc^\top \tc\|\!+\!2\Phi(A) \|\Sigma\!+\!X\| \|\ta - A\| \!+\! \|\Sigma + X\| \|\ta -A \|^2  \nonumber \\
    &\leq\!\sigma_z^{-2}\Phi(A)^2(\beta\! +\! \| \Sigma\| )^2(2\|C\| \|\tc \!-\! C \| + \|\tc \!- \!C \|^2  )\! +\! (\beta \!+\! \| \Sigma\|)(2\Phi(A) \|\ta \!-\! A\| + \|\ta\! -\! A\|^2) \label{boundFdif}
\end{align}
This gives us the following,
\begin{align*}
   \| \Psi(X)\|\! &\leq\! \frac{\sigma_z^{-2}\Phi(A)^2\left(\beta + \| \Sigma\| \right)^2\left(2\|C\| \|\tc - C \| + \|\tc \!- \!C \|^2  \right)}{1-\upsilon^2} \\
   &\qquad \qquad\qquad\qquad\qquad+ \frac{\left(\beta + \| \Sigma\|\right) \left(2\Phi(A) \|\ta - A\| + \|\ta\! -\! A\|^2\right) + \sigma_z^{-2}\upsilon^2 \beta^2 \|C\|^2 }{1-\upsilon^2}  
\end{align*}

Again using Lemma~\ref{normwoodbury} and the definition of $\mathcal{H}(X)$, for $X_1, X_2 \in \mathcal{S}_{\Sigma,\beta}$
\begin{equation*}
    \|\mathcal{H}(X_1) - \mathcal{H}(X_2) \| \leq \upsilon^2\left( (\sigma_z^{-2}\|C\|^2\beta)^2 + 2(\sigma_z^{-2}\|C\|^2\beta) \right) \|X_1 -X_2 \| 
\end{equation*}

Next we bound
\[
\|\mathcal{D}(X_1,X_2)\| = \|F(\Sigma+X_1, \ta,\tc)-F(\Sigma+X_1,A,C) - F(\Sigma+X_2, \ta,\tc)+F(\Sigma+X_2,A,C)\|.
\]
Notice that using Lemma~\ref{normwoodbury}, we have $\|(I + \sigma_z^{-2} (\Sigma + X) \tc^\top \tc)^{-1} \| , \|(I + \sigma_z^{-2} (\Sigma + X) C^\top C)^{-1} \| \leq \frac{2(\|\Sigma \| + \beta )}{\sigma_n(\Sigma)}$ from the choice of $\beta$. Let $V_1 = (I + \sigma_z^{-2} (\Sigma + X_1) C^\top C)^{-1} (\Sigma + X_1) $ and $\tilde{V}_1 = (I + \sigma_z^{-2} (\Sigma + X_1) \tc^\top \tc)^{-1} (\Sigma + X_1) $. Define similarly $V_2$ and $\tilde{V}_2$. Note that from Lemma~\ref{normwoodbury}, $\|V_1\|, \| V_2\|, \| \tilde{V}_1\|, \| \tilde{V}_2\| \leq \|\Sigma\| + \beta$. Using these, we bound $\| \mathcal{D}(X_1,X_2)\| $ as follows 
\begin{align}
  &\Big\|\mathcal{D}(X_1,X_2) \Big\| \nonumber \\
  &= \Big \|A\tilde{V}_1 \sigma_z^{-2}(C^\top C \!-\! \tc^\top \tc) V_1 A^\top \!\!-\! A \tilde{V}_2 \sigma_z^{-2}(C^\top C - \tc^\top \tc) V_2 A^\top \!\!-\! (\ta \!-\! A)\tilde{V}_1 A^\top \!+\! (\ta \!-\! A)\tilde{V}_2 A^\top \nonumber \\
  & \qquad - A \tilde{V}_1 (\ta - A)^\top + A \tilde{V}_2 (\ta - A)^\top -(\ta - A) \tilde{V}_1 (\ta - A)^\top + (\ta - A) \tilde{V}_2 (\ta - A)^\top \Big \| \nonumber \\
  &\leq \Phi(A)^2 \|(\tilde{V}_1 - \tilde{V}_2)\sigma_z^{-2}(C^\top C - \tc^\top \tc)V_1  \| + \Phi(A)^2 \|\tilde{V}_2 \sigma_z^{-2}(C^\top C - \tc^\top \tc) (V_1 -V_2) \| \nonumber \\
  & \qquad+  \|\tilde{V}_1 - \tilde{V}_2 \| \left(2\Phi(A)\|\ta - A \| + \| \ta - A\|^2 \right) \nonumber \\
  &\leq \sigma_z^{-2} \Phi(A)^2 (2\|C\| \|\tc \!-\! C \| + \|\tc \!- \!C \|^2 ) \left( \|\tilde{V}_1 - \tilde{V}_2\|  \| V_1  \|  + \|\tilde{V}_2 \| \|V_1 -V_2 \|  \right) \nonumber \\
  &\qquad + \|\tilde{V}_1 - \tilde{V}_2 \| \left(2\Phi(A)\|\ta - A \| + \| \ta - A\|^2 \right) \label{boundVs}
\end{align}
We need to consider $\|\tilde{V}_1 - \tilde{V}_2 \| $ and $\|V_1 -V_2 \|$: 
\begin{align*}
    \|\tilde{V}_1 - \tilde{V}_2 \| &\leq \| (I + \sigma_z^{-2} (\Sigma + X_1) \tc^\top \tc)^{-1} (X_1 - X_2) \| \\
    &+ \left \|\left( (I + \sigma_z^{-2} (\Sigma + X_1) \tc^\top \tc)^{-1} -  (I + \sigma_z^{-2} (\Sigma + X_2) \tc^\top \tc)^{-1} \right) (\Sigma + X_2) \right \| \\
    &\leq \|X_1 -X_2\| \frac{2(\|\Sigma \| + \beta )}{\sigma_n(\Sigma)} + \sigma_z^{-2} \frac{4(\|\Sigma \| + \beta )^3}{\sigma^2_n(\Sigma)} (\|C\| + \|\tc -C \|)^2  \|X_1 - X_2 \| \\
    \|V_1 - V_2 \| &\leq \|X_1 -X_2\| \frac{2(\|\Sigma \| + \beta )}{\sigma_n(\Sigma)} + \sigma_z^{-2} \frac{4(\|\Sigma \| + \beta )^3}{\sigma^2_n(\Sigma)} \|C\|^2  \|X_1 - X_2 \|
\end{align*}
Combining these with \eqref{boundVs}, we get 
\small
\begin{align*}
    &\Big\|\mathcal{D}(X_1,X_2) \Big\| \\
    &\leq\! \Bigg[ (2\|C\| \|\tc \!-\! C \|\!+\! \|\tc \!- \!C \|^2 ) \Phi(A)^2 \Bigg(\frac{4\sigma_z^{-2} (\|\Sigma \|\! +\! \beta )^2}{\sigma_n(\Sigma)}\! +\! \frac{8\sigma_z^{-4}(\|\Sigma \| \!+\! \beta )^4}{\sigma^2_n(\Sigma)} ((\|C\| \!+\! \|\tc -C \|)^2\!+\!\|C\|^2)  \!\Bigg)  \\
    &+\left(2\Phi(A)\|\ta - A \| + \| \ta - A\|^2 \right) \left( \frac{2(\|\Sigma \| + \beta )}{\sigma_n(\Sigma)} +  \frac{4\sigma_z^{-2}(\|\Sigma \| + \beta )^3}{\sigma^2_n(\Sigma)} (\|C\| + \|\tc -C \|)^2  \right) \Bigg] \|X_1 - X_2 \| \! 
\end{align*}
\normalsize 
Therefore we have the following inequality for $\Psi(X_1) - \Psi(X_2)$: 
\small
\begin{align}
    &\|\Psi(X_1)\!-\!\Psi(X_2) \| \nonumber \\
    &\leq  \Bigg[ (2\|C\| \|\tc \!-\! C \|\!+\! \|\tc \!- \!C \|^2 )\Phi(A)^2 \Bigg(\frac{4\sigma_z^{-2} (\|\Sigma \|\! +\! \beta )^2}{\sigma_n(\Sigma)}\! +\! \frac{8\sigma_z^{-4}(\|\Sigma \| \!+\! \beta )^4}{\sigma^2_n(\Sigma)} ((\|C\| \!+\! \|\tc -C \|)^2\!+\!\|C\|^2)  \!\Bigg) \nonumber \\
    &+ \left(2\Phi(A)\|\ta - A \| + \| \ta - A\|^2 \right) \left( \frac{2(\|\Sigma \| + \beta )}{\sigma_n(\Sigma)} + \sigma_z^{-2} \frac{4(\|\Sigma \| + \beta )^3}{\sigma^2_n(\Sigma)} (\|C\| + \|\tc -C \|)^2  \right) \nonumber \\
    &+ \upsilon^2\left( (\sigma_z^{-2}\|C\|^2\beta)^2 + 2(\sigma_z^{-2}\|C\|^2\beta) \right) \Bigg] \frac{\|X_1 - X_2 \| }{1-\upsilon^2} \label{contraction}
\end{align}
\normalsize
Denote $\epsilon$ such that $\epsilon \coloneqq \max\{\|C - \tc \|, \|A - \ta \|\}$. The choice of $\Texp$ (due to $T_A$, $T_B$) guarantees that $\epsilon < 1$. In order to show that $D_{\Sigma}$ is the unique fixed point of $\Psi$ in $\mathcal{S}_{\Sigma,\beta}$, one needs to show that $\Psi$ maps $\mathcal{S}_{\Sigma,\beta}$ to itself and it's contraction. To this end, we need to have $\epsilon$ and $\beta$ that gives $\|\Psi(X)\| \leq \beta$ and $\|\Psi(X_1)\!-\!\Psi(X_2) \| < \|X_1 - X_2 \| $. Let $\beta = 2k^* \epsilon < \frac{\sigma_n(\Sigma)}{2}$ where 
\begin{equation*}
    k^* = \frac{\sigma_z^{-2}\Phi(A)^2 (2\|C\|+ 1)\|\Sigma \|^2 + (2\Phi(A) + 1) \| \Sigma\| }{1 - \upsilon^2}
\end{equation*}
One can verify that this gives $\|\Psi(X)\| \leq \beta$. In order to get contraction, the coefficient of $\|X_1 - X_2\|$ in \eqref{contraction} must be less than 1. This requires 
\begin{equation*}
    \epsilon < \frac{1-\upsilon^2}{(2\|C\|+1)\Phi(A)^2 c_1 + (2\Phi(A)+1)c_2 + 6k^*c_3}, 
\end{equation*}
for 
\begin{align*}
    c_1 &= \Bigg(\frac{4\sigma_z^{-2} (\|\Sigma \|\! +\! \sigma_n(\Sigma)/2 )^2}{\sigma_n(\Sigma)}\! +\! \frac{8\sigma_z^{-4}(\|\Sigma \| \!+\! \sigma_n(\Sigma)/2 )^4}{\sigma^2_n(\Sigma)} (2\|C\|^2 + 2\|C\| + 1)  \!\Bigg) \\
    c_2 &= \Bigg( \frac{2(\|\Sigma \| + \sigma_n(\Sigma)/2 )}{\sigma_n(\Sigma)} +  \frac{4\sigma_z^{-2}(\|\Sigma \| + \sigma_n(\Sigma)/2 )^3}{\sigma^2_n(\Sigma)} (\|C\|^2 + 2\|C\| + 1)  \Bigg) \\
    c_3 &= \upsilon^2 \sigma_z^{-2} \|C\|^2. 
\end{align*}
From the choice of $\Texp$ (Due to $T_L$), $\epsilon$ satisfies the stated bound. Thus, $\Psi$ has a unique fixed point in $\mathcal{S}_{\Sigma,2k\epsilon}$, \textit{i.e.} $\|\tsig - \Sigma \| \leq 2k^* \max\{\|C - \tc \|, \|A - \ta \| \}$. Bringing back the similarity transformations, this gives us the following bound
\begin{align*}
    \|\tsig - \mathbf{S}^{-1} \sig \mathbf{S}  \| &\leq 2k^* \max\{\|\tc - C \mathbf{S} \|, \|\ta - \mathbf{S}^{-1}A\mathbf{S} \| \} \\
    & \leq 4k^* \max\left\{\beta_A, \beta_C \right\} \coloneqq \Delta \sig
\end{align*}
since $\|\ta - \mathbf{S}^{-1}A\mathbf{S} \| \leq 2\beta_A $ and $\|\tc - C\mathbf{S} \| \leq 2\beta_C $. \\

We know prove the second part of Lemma \ref{StabilityCov} for $\|\tl - \mathbf{S}^{-1}L\|$. Again, considering for $\mathbf{S} = I$, 
and using the definition of $L$ and $\tl$, we get 
\begin{align}
    L-\tl &= \Sigma C^\top(C \Sigma C^\top + \sigma_z^2 I)^{-1} - \tsig \tc^\top(\tc \tsig  \tc^\top + \sigma_z^2 I)^{-1} \nonumber \\
    &= (\Sigma C^\top - \tsig \tc^\top) (C \Sigma C^\top + \sigma_z^2 I)^{-1} + \tsig \tc^\top \left((C \Sigma C^\top + \sigma_z^2 I)^{-1} - (\tc \tsig  \tc^\top + \sigma_z^2 I)^{-1} \right) \nonumber  \\
    &= \left( (\Sigma - \tsig )C^\top + (\tsig  - \Sigma)(C^\top - \tc^\top) +  \Sigma(C^\top - \tc^\top)\right)(C \Sigma C^\top + \sigma_z^2 I)^{-1} \nonumber \\
    &+ \tsig \tc^\top (C \Sigma C^\top + \sigma_z^2 I)^{-1} \left((C \Sigma C^\top \!+\! \sigma_z^2 I)^{-1} + (\tc\tsig \tc^\top \!\!\!-\! C\Sigma C^\top )^{-1}\right)^{-1}\!\!\! (C \Sigma C^\top \!+\! \sigma_z^2 I)^{-1} \label{Lwoodbury}
\end{align}
where \eqref{Lwoodbury} follows from Matrix inversion lemma. We will bound each term individually: 
\begin{align*}
    \|\tl - L \| &\leq \sigma_z^{-2}\left(\|C\|\| \Sigma - \tsig \| + \|  \Sigma - \tsig \|\|\tc - C\| + \| \Sigma\| \|\tc - C\|\right)   \\
    &\qquad \qquad + \frac{\sigma_z^{-4}\left(\|C\|\| \Sigma - \tsig \| + \|  \Sigma - \tsig \|\|\tc - C\| + \| \Sigma\| \|\tc - C\| + \|C\| \| \Sigma\| \right)}{\sigma_n \left( (C \Sigma C^\top + \sigma_z^2 I)^{-1} + (\tc\tsig \tc^\top - C\Sigma C^\top )^{-1} \right)}
\end{align*}
From Weyl's inequality we have 
\begin{align*}
    &\frac{1}{\sigma_n \left( (C \Sigma C^\top + \sigma_z^2 I)^{-1} + (\tc\tsig \tc^\top - C\Sigma C^\top )^{-1} \right)} \\
    &\qquad \qquad \qquad \qquad  \qquad \leq \frac{1}{\sigma_n \left( (C \Sigma C^\top + \sigma_z^2 I)^{-1} \right)  +\sigma_n \left( (\tc\tsig \tc^\top - C\Sigma C^\top )^{-1} \right)} \\
    &\qquad \qquad \qquad \qquad  \qquad = \frac{1}{\frac{1}{\|C \Sigma C^\top + \sigma_z^2 I \|} + \frac{1}{\|\tc\tsig \tc^\top - C\Sigma C^\top \|}} = \frac{\|\tc\tsig \tc^\top - C\Sigma C^\top\|}{1 + \frac{\|\tc\tsig \tc^\top - C\Sigma C^\top\|}{\|C \Sigma C^\top + \sigma_z^2 I \|}} \\
    &\qquad \qquad \qquad \qquad  \qquad \leq \| \tc\tsig \tc^\top - C\Sigma C^\top \| 
\end{align*}
where the last inequality follows from the choice of $\Texp$ which provides that $\|\tc\tsig \tc^\top - C\Sigma C^\top \| \leq \|C \Sigma C^\top + \sigma_z^2 I \| $. Thus we get
\begin{align*}
    &\|\tl - L \| \\
    &\leq \left(\sigma_z^{-2} + \sigma_z^{-4}\| \tc\tsig \tc^\top - C\Sigma C^\top \|\right) \left(\|C\|\| \Sigma - \tsig \| + \|  \Sigma - \tsig \|\|\tc - C\| + \| \Sigma\| \|\tc - C\| \right) \\
    &\qquad \qquad+ \sigma_z^{-4} \| \tc\tsig \tc^\top - C\Sigma C^\top \| \|C\| \|\Sigma\| \\
    &\leq\! \left(\sigma_z^{-2} \!+\! \sigma_z^{-4} \left( 4\beta_C^2\|\sig \| \!+\! 4\beta_C\|C\|\|\sig \| \!+\! \|C\|^2 \Delta \sig \!+\! 4\beta_C \|C\| \Delta \sig \!+\! 4 \beta_C^2 \Delta \sig \right) \right)\left(\|C\|\Delta \sig \!+\! 2 \beta_C \Delta \sig \!+\! 2 \beta_C \| \Sigma\| \right)\\
    &\qquad \qquad+ \sigma_z^{-4} \left( 4\beta_C^2\|\sig \| \!+\! 4\beta_C\|C\|\|\sig \| \!+\! \|C\|^2 \Delta \sig \!+\! 4\beta_C \|C\| \Delta \sig \!+\! 4 \beta_C^2 \Delta \sig \right) \|C\| \|\Sigma\| \\
    &\leq\! \left(\sigma_z^{-2} \!+\! \sigma_z^{-4} \left( 4\beta_C\|\sig \| \!+\! 4\beta_C\|C\|\|\sig \| \!+\! \|C\|^2 \Delta \sig \!+\! 4 \|C\| \Delta \sig \!+\! 4 \Delta \sig \right) \right)\left(\|C\|\Delta \sig \!+\! 2 \Delta \sig \!+\! 2 \beta_C \| \Sigma\| \right)\\
    &\qquad \qquad+ \sigma_z^{-4} \left( 4\beta_C\|\sig \| \!+\! 4\beta_C\|C\|\|\sig \| \!+\! \|C\|^2 \Delta \sig \!+\! 4 \|C\| \Delta \sig \!+\! 4 \Delta \sig \right) \|C\| \|\Sigma\|
    \\
    &\leq\! \Delta \sig \left(\sigma_z^{-2}(\|C\|+2) + \sigma_z^{-4} \|C\|^3 \| \sig\| + 10 \sigma_z^{-4} \|C\|^2 \| \sig \| +  24 \sigma_z^{-4} \|C\| \| \sig \| + 16 \sigma_z^{-4} \| \sig \| \right) \\
    &\qquad \qquad + \beta_C \left( 2 \sigma_z^{-2} \| \sig\| +
    8\sigma_z^{-4} \| \sig\|^2
    + 12\sigma_z^{-4} \|C\| \| \sig\|^2 + 4\sigma_z^{-4} \|C\|^2 \| \sig\|^2   \right)\\
    &\qquad \qquad +\! \Delta \sig^2 \left( \sigma_z^{-4} \|C\|^3 + 6\sigma_z^{-4} \|C\|^2 + 12\sigma_z^{-4} \|C\| + 8\sigma_z^{-4} \right).
\end{align*}
The choice of $T_L$ gives,
\begin{align*}
    &\|\tl\!-\! L\| \\
    &\leq\! \Delta \sig \bigg(\sigma_z^{-2}(\|C\|\!+\!2) \!+\! \sigma_z^{-4} \|C\|^3 (1\!+\!\| \sig\| ) \!+\! 10 \sigma_z^{-4} \|C\|^2 (0.6 \!+\! \| \sig \|) \!+\!  24 \sigma_z^{-4} \|C\| (0.5\!+\!\| \sig \|) \!+\! 16 \sigma_z^{-4} (0.5 \!+\! \| \sig \|)  \bigg) \\ 
    &\qquad + \beta_C \bigg( 2 \sigma_z^{-2} \| \sig\| +
    8\sigma_z^{-4} \| \sig\|^2
    + 12\sigma_z^{-4} \|C\| \| \sig\|^2 + 4\sigma_z^{-4} \|C\|^2 \| \sig\|^2   \bigg)
\end{align*}

\null\hfill$\square$

\section{Regret of Exploration Phase}
\label{SuppRegretExplore}
\begin{lemma}
Suppose Assumptions \ref{Stable} and \ref{AssumContObs}  hold. For any $0 < \delta < 1$, with probability as least $1-\delta$,  the regret of controlling \LQG system $\Theta$ with $u_t \sim \mathcal{N}(0,\sigma_u^2)$ for $1\leq t\leq \Texp$, $\textit{i.e.}$ pure exploration, is upper bounded as follows:
\begin{align}
    \reg(\Texp) \!&\leq\! \Texp \! \left( \frac{(\sigma_w^2 \!+\! \sigma_u^2 \|B\|^2) \|A\|^2}{1-\|A\|^2} \Tr(C^\top Q C) \!+\! \sigma_u^2 \Tr(R) \!-\! \Tr(C^\top Q C \bar{\Sigma} \!+\! P(\Sigma \!-\! \bar{\Sigma})) \! \right) \nonumber \\
   &\qquad \qquad \qquad + 2\sqrt{\Texp}\left(\|C^\top Q C \|X_{exp}^2 + \| Q\|Z^2 + \|R\|U_{exp}^2\right)\sqrt{2 \log\frac{2}{\delta}}
\end{align}
where $X_{exp} \coloneqq \frac{(\sigma_w^2 + \sigma_u^2 \|B\|^2)  \|A\|^2}{1-\|A\|^2}\sqrt{2n\log(12n\Texp/\delta)}$, $Z \coloneqq \sigma_z^2 \sqrt{2m\log(12m\Texp/\delta)}$,\\ $U_{exp} \coloneqq \sigma_u^2 \sqrt{2p\log(12p\Texp/\delta)}$, while $\Sigma$ and $\bar{\Sigma}$ are the solutions to the following algebraic Riccati equation of system $\Theta$:
\begin{equation*}
    \Sigma = A \bar{\Sigma} A^\top + \sigma_w^2 I, \qquad \bar{\Sigma} = \Sigma - \Sigma C^\top \left( C \Sigma C^\top + \sigma_z^2 I \right)^{-1} C \Sigma.
\end{equation*}

\end{lemma}

\begin{proof}
For all $1\leq t \leq \Texp$, $\Sigma(x_t) \preccurlyeq \mathbf{\Gamma_\infty}$, where $\mathbf{\Gamma_\infty}$ is the steady state covariance matrix of $x_t$ such that,
\begin{equation*}
    \mathbf{\Gamma_\infty} = \sum_{i=0}^\infty \sigma_w^2 A^i(A^\top)^i + \sigma_u^2 A^iBB^\top(A^\top)^i.
\end{equation*}

From the Assumption \ref{Stable}, for a finite $\Phi(A)$,  $\|A^\tau \| \leq \Phi(A) \rho(A)^\tau$ for all $\tau \geq 0$. Thus, $\|\mathbf{\Gamma_\infty}\| \leq (\sigma_w^2 + \sigma_u^2 \|B\|^2) \frac{\Phi(A)^2 \rho(A)^2}{1-\rho(A)^2}$. Notice that each $x_t$ is component-wise $\sqrt{\|\mathbf{\Gamma_\infty}\|}$-sub-Gaussian random variable.

Combining this with Lemma~\ref{subgauss lemma} and using union bound, we can deduce that with probability $1-\delta/2$, for all $1\leq t \leq \Texp$, 
\begin{align}
    \label{exploration norms first}
    \| x_t \| &\leq X_{exp} \coloneqq \frac{(\sigma_w + \sigma_u \|B\|)  \Phi(A) \rho(A)}{\sqrt{1-\rho(A)^2}}\sqrt{2n\log(12n\Texp/\delta)} , \\
    \| z_t \| &\leq Z \coloneqq  \sigma_z \sqrt{2m\log(12m\Texp/\delta)} , \\ 
    \| u_t \| &\leq U_{exp} \coloneqq \sigma_u \sqrt{2p\log(12p\Texp/\delta)}  \label{exploration norms last}.
\end{align}
Let $\Omega = 2(\|C^\top Q C \|X_{exp}^2 + \| Q\|Z^2 + \|R\|U_{exp}^2)$. Define $\mathcal{X}_t = x_t^\top C^\top Q C x_t + z_t^\top Q z_t + u_t^\top R u_t - \mathbb{E}[x_t^\top C^\top Q C x_t + z_t^\top Q z_t + u_t^\top R u_t]$ and its truncated version $\tilde{\mathcal{X}}_t =\mathbbm{1}_{\mathcal{X}_t \leq \Omega} \mathcal{X}_t$. Define $S = \sum_{t=1}^{\Texp} \mathcal{X}_t$ and $\tilde{S} = \sum_{t=1}^{\Texp} \tilde{\mathcal{X}}_t$.  By Lemma~\ref{basicprob}, 

\begin{align}
    \Pr\left( S > \Omega \sqrt{2\Texp \log\frac{2}{\delta}}\right) \leq \Pr \left(\max_{1\leq t\leq \Texp} \mathcal{X}_t \geq \Omega \right) + \Pr\left( \tilde{S} > \Omega \sqrt{2\Texp \log\frac{2}{\delta}} \right).
\end{align}
From equations \eqref{exploration norms first}-\eqref{exploration norms last} and Theorem~\ref{azuma}, each term on the right hand side is bounded by $\delta/2$. Thus, with probability $1-\delta$,
\begin{align}
    &\sum_{t=1}^{\Texp} y_t^\top Q y_t \!+\! u_t^\top R u_t \!-\! \mathbb{E}[y_t^\top Q y_t \!+\! u_t^\top R u_t] \!\leq\! \Omega \sqrt{2\Texp \log\frac{2}{\delta}} \\
    &\sum_{t=1}^{\Texp} y_t^\top Q y_t \!+\! u_t^\top R u_t \!\leq\! \Texp\left((\sigma_w^2 \!+\! \sigma_u^2 \|B\|^2) \frac{\Phi(A)^2 \rho(A)^2}{1-\rho(A)^2} \Tr(C^\top Q C) \!+\! \sigma_u^2 \Tr(R) \!+\! \sigma_z^2 \Tr(Q) \right) \\
    &\qquad\qquad\qquad\qquad\qquad+ 2\left(\|C^\top Q C \|X_{exp}^2 + \| Q\|Z^2 + \|R\|U_{exp}^2\right)\sqrt{2\Texp \log\frac{2}{\delta}}
\end{align}
Recall that cost obtained in $\Texp$ by the optimal controller of $\Theta$ is
\[\Texp\left(\Tr(C^\top Q C \bar{\Sigma}) + \Tr(P(\Sigma - \bar{\Sigma})) + \sigma_z^2 \Tr(Q) \right).\] 
Thus the regret obtained from $\Texp$ length exploration is upper bounded as described in the statement of lemma. 
\end{proof}

\section{Upper Bound on \texorpdfstring{$\|\hat{x}_{t|t,\tth}\|$}{State Estimation} and \texorpdfstring{$\|y_t\|$}{Output}, Proof of Lemma \ref{Boundedness}}
\label{SuppBounded}
We know that the behavior of the underlying system is the same as any system that is obtained via similarity transformation of it. Therefore, without loss of generality, we assume that $\mathbf{S} = I$. Also note that, the effect of initial state after $\Texp$ length exploration is omitted in the presentation since it is well-controlled during the  exploration, \textit{i.e.}, \eqref{exploration norms first}, and it exponentially decays after the exploration due to achieved stable closed loop dynamics that is explained below. First define $e_t = y_t - C\hat{x}_{t|t-1,\Theta}$, where $e_t \sim \mathcal{N}\left(0,C \sig  C^\top + \sigma_z^2 I \right)$. Note that $e_t$ is the zero mean white innovation process, in the innovation form of the system characterization.


\begin{proof}
Assume that the event $\mathcal{E}$ holds. Observe that $\hat{x}_{t|t,\tth}$ has the following dynamics, 
\begin{align}
    \hat{x}_{t|t,\tth} &= (I - \tl\tc)(\ta - \tb\tk) \hat{x}_{t-1|t-1,\tth} + \tl y_t \nonumber\\
    &= (I - \tl\tc)(\ta - \tb\tk) \hat{x}_{t-1|t-1,\tth} \nonumber \\
    &\quad \qquad \qquad +\tl \left(Cx_t - C\hat{x}_{t|t-1,\tth} + C\hat{x}_{t|t-1,\tth}  + z_t\right) \nonumber \\
    &= (I - \tl\tc)(\ta - \tb\tk) \hat{x}_{t-1|t-1,\tth} \nonumber \\
    &\quad \qquad \qquad +\tl \left(Cx_t - C\hat{x}_{t|t-1,\tth} + C(\ta - \tb\tk) \hat{x}_{t-1|t-1,\tth} + z_t\right) \nonumber \\
    &= \left(\ta-\tb\tk - \tl\left(\tc\ta -\tc\tb\tk - C\ta + C\tb\tk \right) \right)\hat{x}_{t-1|t-1,\tth} \nonumber  \\
    &\quad \qquad \qquad + \tl C(x_{t} - \hat{x}_{t|t-1,\Theta} + \hat{x}_{t|t-1,\Theta} - \hat{x}_{t|t-1,\tth}) + \tl z_t \nonumber \\
    &= \left(\ta-\tb\tk - \tl\left(\tc\ta -\tc\tb\tk - C\ta + C\tb\tk \right) \right)\hat{x}_{t-1|t-1,\tth} \nonumber \\
    &\quad \qquad \qquad + \tl C(x_{t} - \hat{x}_{t|t-1,\Theta}) + \tl C (\hat{x}_{t|t-1,\Theta} - \hat{x}_{t|t-1,\tth}) + \tl z_t \label{estimation propagate} 
\end{align}

Thus, it propagates according to the linear system given in equation \eqref{estimation propagate} with closed loop dynamics $\mathbf{M} = \left(\ta-\tb\tk - \tl\left(\tc\ta -\tc\tb\tk - CA + CB\tk \right)\right)$ driven by the process $\tl C(x_{t} - \hat{x}_{t|t-1,\Theta}) + \tl C (\hat{x}_{t|t-1,\Theta} - \hat{x}_{t|t-1,\tth}) + \tl z_t$. With the Assumption~\ref{Stabilizable set}, for the given $\Texp$ (due to $T_M$), we have $\|\mathbf{M} \| < \frac{1+\rho}{2} < 1$. Notice that $\tl C (x_{t} - \hat{x}_{t|t-1,\Theta}) + \tl z_t$ is $\zeta (\| C\| \|\Sigma\|^{1/2} + \sigma_z)$-sub-Gaussian, thus it's $\ell_2$-norm can be bounded using Lemma \ref{subgauss lemma}:
\begin{equation}
    \|\tl_t C (x_{t} - \hat{x}_{t|t-1,\Theta}) + \tl_t z_t\| \leq \zeta \left(\| C\| \|\Sigma\|^{1/2} + \sigma_z\right) \sqrt{2n\log(2nT/\delta)} \label{excitationbound}
\end{equation}
for all $T \geq t \geq \Texp$ with probability at least $1-\delta$. Next, we will consider $\hat{x}_{t|t-1,\Theta} - \hat{x}_{t|t-1,\tth}$, \textit{i.e.} how much the state estimation using the optimistic model deviates from the state estimation using the true model. Let $\Delta_t = \hat{x}_{t|t-1,\Theta} - \hat{x}_{t|t-1,\tth}$. Consider the following decompositions, 

\begin{align*}
    \hat{x}_{t+1|t,\Theta} &= A \hat{x}_{t|t,\Theta} - B \tk \hat{x}_{t|t,\tth}\\
    &= A \hat{x}_{t|t,\Theta} - B \tk \hat{x}_{t|t,\Theta} - B \tk (\hat{x}_{t|t,\tth} - \hat{x}_{t|t,\Theta}) \\
    &= (A - B\tk) \hat{x}_{t|t,\Theta} - B \tk (\hat{x}_{t|t,\tth} - \hat{x}_{t|t,\Theta}), \\
    \hat{x}_{t+1|t,\tth} &= \ta \hat{x}_{t|t,\tth} - \tb \tk \hat{x}_{t|t,\tth}\\
    &= (\ta + A - A - \tb \tk + B \tk - B \tk) (\hat{x}_{t|t,\tth} - \hat{x}_{t|t,\Theta} + \hat{x}_{t|t,\Theta} ) \\
    &= \underbrace{(\ta \!-\! A \!-\! \tb \tk \!+\! B \tk)}_{\delta_{\tth}} \hat{x}_{t|t,\Theta} +  \delta_{\tth} (\hat{x}_{t|t,\tth} \!-\! \hat{x}_{t|t,\Theta}) + (A \!-\! B \tk) (\hat{x}_{t|t,\tth} \!-\! \hat{x}_{t|t,\Theta}) + (A \!-\! B \tk) \hat{x}_{t|t,\Theta}.
\end{align*}
Thus, we get
\begin{align}
    \hat{x}_{t+1|t,\Theta} - \hat{x}_{t+1|t,\tth} &= (A - B\tk) \hat{x}_{t|t,\Theta} - B \tk (\hat{x}_{t|t,\tth} - \hat{x}_{t|t,\Theta}) \!-\! \delta_{\tth} \hat{x}_{t|t,\Theta} -  \delta_{\tth} (\hat{x}_{t|t,\tth} \!-\! \hat{x}_{t|t,\Theta}) \nonumber \\ &\quad - (A \!-\! B \tk) (\hat{x}_{t|t,\tth} \!-\! \hat{x}_{t|t,\Theta}) - (A \!-\! B \tk) \hat{x}_{t|t,\Theta}, \nonumber \\
    \Delta_{t+1}&= A (\hat{x}_{t|t,\Theta} - \hat{x}_{t|t,\tth} ) - \delta \hat{x}_{t|t,\Theta} + \delta (\hat{x}_{t|t,\Theta} - \hat{x}_{t|t,\tth} ) \nonumber \\
    &= (A + \delta_{\tth}) (\hat{x}_{t|t,\Theta} - \hat{x}_{t|t,\tth} ) - \delta_{\tth} \hat{x}_{t|t,\Theta}. \label{lastline}
\end{align}
Recalling the definition of innovation term $e_t$ in the beginning of the section, we can decompose $\hat{x}_{t|t,\Theta} - \hat{x}_{t|t,\tth}$ as follows,
\begin{align*}
    \hat{x}_{t|t,\Theta} - \hat{x}_{t|t,\tth} &= \hat{x}_{t|t-1,\Theta} + L e_t - (\hat{x}_{t|t-1,\tth} + \tl (y_t - \tc \hat{x}_{t|t-1,\tth} ) ) \\
    &= \hat{x}_{t|t-1,\Theta} - \hat{x}_{t|t-1,\tth} + L e_t - \tl (e_t + C\hat{x}_{t|t-1,\Theta} - \tc \hat{x}_{t|t-1,\tth} ) \\
    &= \underbrace{\hat{x}_{t|t-1,\Theta} - \hat{x}_{t|t-1,\tth} }_{\Delta_t} + (L-\tl) e_t - \tl \left((C-\tc)\hat{x}_{t|t-1,\Theta} + \tc( \underbrace{\hat{x}_{t|t-1,\Theta} - \hat{x}_{t|t-1,\tth} }_{\Delta_t})\right) \\
    &= (I -\tl \tc) \Delta_t + (L - \tl) e_t + \tl (\tc - C) \hat{x}_{t|t-1,\Theta}
\end{align*}
Plugging $\hat{x}_{t|t,\Theta} - \hat{x}_{t|t,\tth}$ into (\ref{lastline}) gives the following,
\begin{align}
    \Delta_{t+1} &= (A + \delta_{\tth}) \left((I -\tl \tc) \Delta_t + (L - \tl) e_t + \tl (\tc - C) \hat{x}_{t|t-1,\Theta} \right) - \delta_{\tth} \hat{x}_{t|t,\Theta} \nonumber\\
    &= (A + \delta_{\tth}) (I -\tl \tc) \Delta_t + (A + \delta_{\tth}) (L - \tl) e_t + (A + \delta_{\tth}) \tl (\tc - C)\hat{x}_{t|t-1,\Theta} - \delta_{\tth} \hat{x}_{t|t,\Theta} \nonumber \\
    &= (A + \delta_{\tth}) (I -\tl \tc) \Delta_t + (A + \delta_{\tth}) (L - \tl) e_t + (A + \delta_{\tth}) \tl (\tc - C)\hat{x}_{t|t-1,\Theta} - \delta_{\tth} \hat{x}_{t|t-1,\Theta} - \delta_{\tth} L e_t  \nonumber \\
    &= (A + \delta_{\tth}) (I -\tl \tc) \Delta_t + (AL - A\tl +\delta_{\tth} L - \delta_{\tth} \tl -\delta_{\tth} L)e_t + \left( (A+\delta_{\tth})\tl(\tc-C) - \delta_{\tth} \right)\hat{x}_{t|t-1,\Theta} \nonumber \\
    &= (A + \delta_{\tth}) (I -\tl \tc) \Delta_t + (A(L -\tl) - \delta_{\tth} \tl )e_t + \left((A+\delta_{\tth})\tl(\tc-C) - \delta_{\tth} \right)\hat{x}_{t|t-1,\Theta} \nonumber \\
    &= \sum_{j=0}^t \left((A + \delta_{\tth}) (I -\tl \tc)\right)^{t-j} \left(A(L -\tl) - \delta_{\tth} \tl \right) e_j  \nonumber \\
    &\quad + \sum_{j=1}^t \left((A + \delta_{\tth}) (I -\tl \tc)\right)^{t-j} \left((A+\delta_{\tth})\tl(\tc-C) - \delta_{\tth} \right)\hat{x}_{j|j-1,\Theta}. \label{lastlinefordiff}
\end{align}
Now consider the decomposition given below for $\hat{x}_{j|j-1,\Theta}$:
\begin{align}
    \hat{x}_{j|j-1,\Theta} &= A \hat{x}_{j-1|j-1,\Theta} - B \tk \hat{x}_{j-1|j-1,\tth} \nonumber \\
    &= A \hat{x}_{j-1|j-1,\Theta} - B \tk \hat{x}_{j-1|j-1,\Theta} - B \tk (\hat{x}_{j-1|j-1,\tth} - \hat{x}_{j-1|j-1,\Theta})\nonumber  \\
    &= (A - B\tk) \hat{x}_{j-1|j-1,\Theta} + B \tk (\hat{x}_{j-1|j-1,\Theta} - \hat{x}_{j-1|j-1,\tth} )\nonumber  \\
    &= (A - B\tk) (\hat{x}_{j-1|j-2,\Theta} + L e_{j-1}) + B \tk ((I -\tl \tc) \Delta_{j-1} + (L - \tl) e_{j-1} + \tl (\tc - C) \hat{x}_{j-1|j-2,\Theta})\nonumber  \\
    &= (A - B\tk (I - \tl(\tc-C))) \hat{x}_{j-1|j-2,\Theta} + B \tk (I - \tl\tc) \Delta_{j-1} + ((A - B\tk) L + B \tk (L - \tl)) e_{j-1}\nonumber  \\
    &= \sum_{i=0}^{j-1} \left(A \!-\! B\tk \!+\! B\tk\tl(\tc\!-\!C)\right)^{j-i-1} \left( B\tk(I \!-\!\tl \tc) \Delta_{i} + ((A \!-\! B\tk) L \!+\! B \tk (L \!-\! \tl)) e_{i} \right). \label{lastlineforeqq}
\end{align}

For the brevity of representation, we define the following terms:
\begin{align*}
    \alpha &= (A \!+\! \delta_{\tth}) (I \!-\!\tl \tc),\\
    \beta &= (A\!+\!\delta_{\tth})\tl(\tc\!-\!C) \!-\! \delta_{\tth}, \\
    \kappa &= A(L -\tl) - \delta_{\tth} \tl, \\
    \gamma &= A \!-\! B\tk \!+\! B\tk\tl(\tc\!-\!C), \\
    \chi &= B\tk(I \!-\!\tl \tc), \\
    \xi &= (A \!-\! B\tk) L \!+\! B \tk (L \!-\! \tl).
\end{align*}
Finally, using (\ref{lastlinefordiff}) with the equality given in (\ref{lastlineforeqq}) and the given definitions above we get:
\begin{align}
    \Delta_{t+1} &= \sum_{j=0}^t \alpha ^{t-j} \kappa e_j +\sum_{j=1}^t \alpha^{t-j} \beta  \left(\sum_{i=0}^{j-1} \gamma^{j-i-1} \chi \Delta_{i} \right) +\sum_{j=1}^t \alpha^{t-j} \beta  \left(\sum_{i=0}^{j-1} \gamma^{j-i-1} \xi e_{i} \right). \label{term1}
\end{align}
We will bound first and third term in (\ref{term1}) separately and since each $\Delta_{t+1}$ has linear combination of $\Delta_i$ for $i\leq t$, we provide an inductive argument to bound $\Delta_{t+1}$ for all $t$. Since $\Texp > \max \{T_\alpha, T_\gamma \}$, the system parameter estimates are close enough to the underlying system parameters such that $1 > \sigma \geq \max \{ \|\alpha \|, \|\gamma \|  \}$. For the first term in (\ref{term1}), since $e_t$ is $(\| C\| \|\Sigma\|^{1/2} + \sigma_z)$-sub-Gaussian, using Lemma \ref{subgauss lemma} for all $T \geq t \geq \Texp$ with probability at least $1-\delta$ we have,
\begin{align}
    \sum_{j=0}^t \alpha^{t-j} \kappa e_j \leq \frac{\bar{\kappa}  }{1-\sigma} \left(\| C\| \|\Sigma\|^{1/2} + \sigma_z\right) \sqrt{2m\log(2mT/\delta)}, \label{firstterm}
\end{align}
where $\bar{\kappa} \geq \|\kappa \|$ under the event $\mathcal{E}$. 

For the third term, for all $T \geq t \geq \Texp$ with probability at least $1-\delta$ we have, 
\begin{align}
    &\sum_{j=1}^t \alpha^{t-j} \beta \left( \sum_{i=0}^{j-1} \gamma^{j-i-1} \xi e_i \right) \nonumber \\
    &= \alpha^{t-1} \beta \xi e_0 \!+\! \alpha^{t-2} \beta (\gamma \xi e_0 \!+\! \xi e_1) \!+\! \ldots \!+\! \beta(\gamma^{t-1} \xi e_0 \!+\! \ldots \!+\! \xi e_{t-1}) \nonumber \\
    &= \left(\alpha^{t-1} \beta \xi \!+\! \alpha^{t-2} \beta \gamma \xi \!+\! \beta  \gamma^{t-1} \xi \right) e_0 \nonumber \\
    &\quad + \left( \alpha^{t-2} \beta \xi + \ldots + \beta\gamma^{t-2}\xi \right) e_1 + \ldots + \beta \xi e_{t-1} \nonumber \\
    &\leq \bar{\beta} \bar{\xi} \left(\alpha^{t-1} \!+\! \alpha^{t-2}(\gamma \!+\! I) \!+\! \ldots \!+\! (\gamma^{t-1} \!+\! \gamma^{t-2} \!+\! \ldots \!+\! \gamma \!+\! I)\right) \!\! \left(\| C\| \|\Sigma\|^{1/2} \!+\! \sigma_z\right) \sqrt{2m\log(2mT/\delta)} \nonumber \\
    &\leq \bar{\beta} \bar{\xi} \left( \sum_{i=1}^{t-1} i \sigma^i + \sum_{i=0}^{t-1} \sigma^i \right) \left(\| C\| \|\Sigma\|^{1/2} + \sigma_z\right) \sqrt{2m\log(2mT/\delta)} \nonumber \\
    &\leq \frac{\bar{\beta} \bar{\xi} }{(1-\sigma)^2}\left(\| C\| \|\Sigma\|^{1/2} + \sigma_z\right) \sqrt{2m\log(2mT/\delta)}, \label{thirdterm}
\end{align}
where $\bar{\beta} \geq \|\beta \|$ and $\bar{\xi} \geq \|\xi \|$ under the event $\mathcal{E}$. 
Considering the second term, we will have an inductive argument for the boundedness of $\Delta$. Let $\bar{\Delta} \coloneqq 10 \left(\frac{\bar{\kappa} }{1-\sigma} + \frac{\bar{\beta} \bar{\xi} }{(1-\sigma)^2} \right)\left(\| C\| \|\Sigma\|^{1/2} + \sigma_z\right) \sqrt{2m\log(2mT/\delta)} $. We have that $\| \Delta_1\| \leq \bar{\Delta}$ with probability at least $1-\delta$. Assume that it holds for all $t$. Thus, $\bar{\Delta} \geq \max_{i\leq t} \| \Delta_i\|$.  Using the same arguments with the third term, for the second term, we get 
\begin{align*}
    \sum_{j=1}^t \alpha^{t-j} \beta \left( \sum_{i=0}^{j-1} \gamma^{j-i-1} \chi \Delta_i \right) &\leq \frac{\|\beta\| \bar{\chi} }{(1-\sigma)^2} \bar{\Delta}
\end{align*}
where $\bar{\chi} \geq \|\chi \|$ under the event $\mathcal{E}$. Combining with the bounds derived in (\ref{firstterm}) and (\ref{thirdterm}) for the  Thus,
\begin{align*}
    \Delta_{t+1} \leq \left(\frac{\bar{\kappa} }{1-\sigma} + \frac{\bar{\beta} \bar{\xi} }{(1-\sigma)^2} \right)\left(\| C\| \|\Sigma\|^{1/2} + \sigma_z\right) \sqrt{2m\log(2mT/\delta)} + \frac{\|\beta \| \bar{\chi} }{(1-\sigma)^2} \bar{\Delta}
\end{align*}
Recall that $\Texp > T_{\beta}$. Thus we have $\| \beta \| \leq \frac{9(1-\sigma)^2}{10 \Gamma \|B\|(1+ \zeta + \zeta \|C\|)}$. This shows that $\Delta_{t+1} \leq \bar{\Delta}$: 
\begin{align*}
    \Delta_{t+1} \leq \left(\left(\frac{\bar{\kappa} }{1-\sigma} \!+\! \frac{\bar{\beta} \bar{\xi} }{(1-\sigma)^2} \right) \!+\! 9 \left(\frac{\bar{\kappa} }{1-\sigma} \!+\! \frac{\bar{\beta} \bar{\xi} }{(1-\sigma)^2} \right) \right) \left(\| C\| \|\Sigma\|^{1/2} \!+\! \sigma_z\right) \sqrt{2m\log(2mT/\delta)} \leq \bar{\Delta}.
\end{align*}

Finally, for all $\Texp<t<T$, we have that with probability at least $1-\delta$,
\begin{align}
    \hat{x}_{t|t-1,\Theta} - \hat{x}_{t|t-1,\tth} \leq \bar{\Delta} = 10\left(\frac{\bar{\kappa} }{1-\sigma} + \frac{\bar{\beta} \bar{\xi} }{(1-\sigma)^2} \right)\left(\| C\| \|\Sigma\|^{1/2} + \sigma_z\right) \sqrt{2m\log(2mT/\delta)} \label{deltabound}
\end{align}
where
\begin{enumerate}
    \item $\bar{\kappa} = \Phi(A) \Delta L + 2\zeta (\beta_A + \Gamma \beta_B)$, 
    \item $\bar{\beta} = 2\zeta \beta_C (\Phi(A) + 2(\beta_A + \Gamma \beta_B)) + 2(\beta_A + \Gamma \beta_B) $, 
    \item $\bar{\xi} = \zeta (\rho + 2(\beta_A + \Gamma \beta_B)) + \|B\|\Gamma \Delta L $.
\end{enumerate} 
Then under the event $\mathcal{E}$, with probability $1-2\delta$,
\begin{align}
\| \hat{x}_{t|t,\tth}\| &= \left \| \sum_{i=1}^t \mathbf{M}^{t-i} \left(\tl C(x_{i} - \hat{x}_{i|i-1,\Theta}) + \tl C (\hat{x}_{i|i-1,\Theta} - \hat{x}_{i|i-1,\tth}) + \tl z_i \right)  \right \| \\
&\leq \max_{1\leq i\leq t}\left\| \tl C(x_{i} - \hat{x}_{i|i-1,\Theta}) + \tl C (\hat{x}_{i|i-1,\Theta} - \hat{x}_{i|i-1,\tth}) + \tl z_i \right \| \left( \sum_{i=1}^t \|\mathbf{M}\|^{t-i}  \right) \\
&\leq \frac{2}{1-\rho} \max_{1\leq i\leq t}\left\| \tl C(x_{i} - \hat{x}_{i|i-1,\Theta}) + \tl C (\hat{x}_{i|i-1,\Theta} - \hat{x}_{i|i-1,\tth}) + \tl z_i \right \| \\
&\leq \tilde{\mathcal{X}} \coloneqq \frac{2\zeta \left(\|C\| \bar{\Delta} + \left(\| C\| \|\Sigma\|^{1/2} + \sigma_z\right) \sqrt{2n\log(2nT/\delta)}\right) }{1-\rho}
\end{align}
where the last inequality follows from (\ref{excitationbound}) and (\ref{deltabound}). One can write $y_t$ as 
\begin{align*}
    y_t &= C (\ta - \tb\tk) \hat{x}_{t-1|t-1,\tth} + C(x_t - \hat{x}_{t|t-1,\tth}) + z_{t} \\
    &= C (\ta - \tb\tk) \hat{x}_{t-1|t-1,\tth} + C(x_t - \hat{x}_{t|t-1,\Theta} + \hat{x}_{t|t-1,\Theta} - \hat{x}_{t|t-1,\tth}) + z_{t}
\end{align*}
Following similar argument, with probability $1-2\delta$,
\begin{equation}
    \|y_t \| \leq \rho \| C\| \tilde{\mathcal{X}}  + \| C\| \bar{\Delta} + \left(\| C\| \|\Sigma\|^{1/2} + \sigma_z \right) \sqrt{2m\log(2mT/\delta)}
\end{equation}
for all $t\leq T$, which proves the lemma. 
\end{proof}

\section{Bellman Optimality Equation for \LQG, Proof of Lemma \ref{LQGBellman}}
\label{SuppBelmanOptimality}
For an average cost per stage problem in infinite state and control space like a \LQG control system $\Theta = \left(A,B,C\right)$ with regulating parameters $Q$ and $R$, using the optimal average cost per stage $J_*(\Theta)$ and guessing the correct differential(relative) cost, where $(A,B)$ is controllable, $(A,C)$ is  observable, $Q$ is positive semidefinite and $R$ is positive definite, one can verify that they satisfy Bellman optimality equation~\citep{bertsekas1995dynamic}. The lemma below shows the Bellman optimality equation for \LQG system $\Theta$, which will be critical in regret analysis.  \\

\begin{proof}
Define $\hat{\omega}_t = Ax_t - A\hat{x}_{t|t} + w_t$. Notice that $\hat{\omega}_t$ is independent of the policy used and depends only on the estimation error and noise in steady state. Also notice that $\mathbb{E}\left[\hat{\omega}_t \hat{\omega}_t^\top \right] = \Sigma$ where $\Sigma$ is the positive semidefinite solution to the algebraic Riccati equation:
\begin{equation}
\label{esti err cov riccati}
    \Sigma = A \bar{\Sigma} A^\top + \sigma_w^2 I, \qquad \bar{\Sigma} = \Sigma - \Sigma C^\top \left( C \Sigma C^\top + \sigma_z^2 I\right)^{-1} C \Sigma.
\end{equation}
Using these, for any given $y_t$ and $\hat{x}_{t|t-1}$ at time $t$, the optimum state estimation and the output at $t+1$ can be written as 
\begin{equation} \label{next time step decomposition}
    \hat{x}_{t|t} = \left( I - LC\right)\hat{x}_{t|t-1} + Ly_t, \enskip \hat{x}_{t+1|t,u} = A\hat{x}_{t|t} + Bu, \enskip y_{t+1,u} = CA \hat{x}_{t|t} + CBu + C\hat{\omega}_t + z_{t+1}
\end{equation}
where $L = \Sigma C^\top \left( C \tilde{\Sigma} C^\top + \sigma_z^2 \right)^{-1}$ is the steady-state Kalman filter gain for $\Theta$. Since the aim is to minimize average cost per stage of controlling $\Theta$, the optimal control input, $u = -K\hat{x}_{t|t}$, where $K = \left(R+B^{\top} P B\right)^{-1} B^{\top} P A $ is the steady-state LQR feedback gain for $\Theta$ and $P$ is the positive semidefinite solution to the following algebraic Riccati equation:
\begin{equation} 
    P = A^\top P A + C^\top Q C - A^\top P B \left( R + B^\top P B \right)^{-1} B^\top P A. 
\end{equation}
Recall that optimal average stage cost of \LQG is $J_*(\Theta) = \Tr(C^\top Q C \bar{\Sigma}) + \Tr(P(\Sigma - \bar{\Sigma})) + \Tr(\sigma_z^2 Q)$.
Suppose the differential cost $h$ is a quadratic function of $s_t$ where $s_t = [\hat{x}_{t|t-1}^\top~y_t^\top]^\top \in \R^{n+m}$, \textit{i.e.}
\[ h(s_t) = s_t^\top  
\left[
    \begin{array}{cc}{G_1} & {  G_2  } \\ {G_2^\top} & {  G_3  }  \end{array} \right] s_t = \hat{x}_{t|t-1}^\top G_1 \hat{x}_{t|t-1} + 2 \hat{x}_{t|t-1}^\top G_2 y_t + y_t^\top G_3 y_t. \] 
One needs to verify that there exists $G_1$, $G_2$, $G_3$ such that they satisfy Bellman optimality equation for the chosen differential cost:
\begin{align*}
    &J_*(\Theta) + \hat{x}_{t|t-1}^\top G_1 \hat{x}_{t|t-1} + 2 \hat{x}_{t|t-1}^\top G_2 y_t + y_t^\top G_3 y_t  = \\
    & y_t^\top Q y_t + \hat{x}_{t|t}^\top K^\top R K \hat{x}_{t|t} + \mathbb{E}\left[ \hat{x}_{t+1|t}^\top G_1 \hat{x}_{t+1|t} + 2 \hat{x}_{t+1|t}^\top G_2 y_{t+1} + y_{t+1}^\top G_3 y_{t+1} \right] 
\end{align*}
Using the fact that $\bar{\Sigma} = \Sigma - L \left( C \Sigma C^\top + \sigma_z^2 I \right) L^\top$, we can write the optimal average cost as $J_*(\Theta) = \Tr\left( \left(Q + L^\top P L - L^\top C^\top Q C L \right) \left(C \Sigma C^\top + \sigma_z^2 I \right)\right)$. Expanding the expectation given $\hat{x}_{t|t-1}, y_t$ and using \eqref{next time step decomposition} , we get 
\begin{align} \label{bellman decomp}
    \hat{x}_{t|t-1}^\top &G_1 \hat{x}_{t|t-1} + 2 \hat{x}_{t|t-1}^\top G_2 y_t + y_t^\top G_3 y_t \\
    &= y_t^\top Q y_t + \hat{x}_{t|t}^\top K^\top R K \hat{x}_{t|t} + \hat{x}_{t|t}^\top (A-BK)^\top G_1 (A-BK) \hat{x}_{t|t}\nonumber \\&+ 2\hat{x}_{t|t}^\top (A-BK)^\top G_2 C(A-BK)\hat{x}_{t|t} + \hat{x}_{t|t}^\top (A-BK)^\top C^\top G_3 C (A-BK)\hat{x}_{t|t} \nonumber\\ &+ \mathbb{E}\left[\hat{\omega}_t^\top C^\top G_3 C \hat{\omega}_t + z_{t+1}^\top G_3 z_{t+1} \right] \!-\! \Tr\left( \left(Q + L^\top P L - L^\top C^\top Q C L \right) \left(C \Sigma C^\top \!+\! \sigma_z^2 I \right)\right) \nonumber
\end{align}
Notice that $\mathbb{E}\left[\hat{\omega}_t^\top C^\top G_3 C \hat{\omega}_t + z_{t+1}^\top G_3 z_{t+1} \right] = \Tr\left( G_3 \left(C \Sigma C^\top + \sigma_z^2 I \right) \right)$. In order to match with the last term of \eqref{bellman decomp}, set $G_3 = Q + L^\top \left( P - C^\top Q C \right) L$. Inserting $G_3$ to \eqref{bellman decomp}, we get following 3 equations to solve for $G_1$ and $G_2$: \\
\textbf{1)} From quadratic terms of $y_t$:
\begin{align*}
    L^\top P L &- L^\top C^\top Q C L \\
    &= L^\top K^\top R K L + L^\top (A-BK)^\top G_1 (A-BK) L + 2 L^\top (A-BK)^\top G_2 C (A-BK) L \\&+ L^\top (A-BK)^\top C^\top \left( Q + L^\top P L - L^\top C^\top Q C L \right) C (A-BK) L
\end{align*}
\textbf{2)} From quadratic terms of $x_{t|t-1}$:
\begin{align*}
    G_1 &= (I-LC)^\top K^\top R K (I-LC) + (I-LC)^\top (A-BK)^\top G_1 (A-BK)(I-LC) \\
    &+ 2(I-LC)^\top (A-BK)^\top G_2 C (A-BK) (I-LC) \\ &+ (I-LC)^\top (A-BK)^\top C^\top \left( Q + L^\top P L - L^\top C^\top Q C L\right) C (A-BK) (I-LC)
\end{align*}
\textbf{3)} From bilinear terms of $x_{t|t-1}$ and $y_t$:
\begin{align*}
    G_2 &= (I-LC)^\top K^\top R K L + (I-LC)^\top (A-BK)^\top G_1 (A-BK) L\\ &+ 2(I-LC)^\top (A-BK)^\top G_2 C (A-BK)L \\
    &+ (I-LC)^\top (A-BK)^\top C^\top \left(Q + L^\top P L - L^\top C^\top Q C L \right) C (A-BK)L
\end{align*}
$G_1 = (I-LC)^\top \left( P - C^\top Q C  \right) (I-LC) $ and $ G_2 = (I-LC)^\top \left( P - C^\top Q C  \right) L$ satisfies all 3 equations. Thus one can write Bellman optimality equation as 
\small
\begin{align*}
    &J_*(\Theta)\!+\!\hat{x}_{t|t-1}^\top (I\!-\!LC)^\top\! \left( P\!-\! C^\top Q C  \right)(I\!-\!LC) \hat{x}_{t|t-1}\!\\
    &+\!2 \hat{x}_{t|t-1}^\top (I\!-\!LC)^\top \left( P\!-\! C^\top Q C  \right) L y_t\! +\! y_t^\top\!\left( Q \!+\!L^\top\!\left( P\!-\! C^\top Q C \right)\!L \right)\!y_t\!= \\
    & y_t^\top Q y_t + \hat{x}_{t|t}^\top K^\top R K \hat{x}_{t|t} + \mathbb{E}\left[ \hat{x}_{t+1|t}^\top (I-LC)^\top \left( P - C^\top Q C  \right) (I-LC) \hat{x}_{t+1|t} \right] \\ &+2 \mathbb{E}\left[ \hat{x}_{t+1|t}^\top (I-LC)^\top \left( P - C^\top Q C  \right) L y_{t+1} + y_{t+1}^\top \left( Q + L^\top \left( P - C^\top Q C \right) L \right) y_{t+1} \right] 
\end{align*}
\normalsize
Combining terms using \eqref{next time step decomposition} gives the \begin{align*}
    J_*(\Theta)\!+\!\hat{x}_{t|t}^\top\!\left( P \!-\!C^\top Q C\right)\!\hat{x}_{t|t}\!+\!y_t^\top Q y_t\!&=\! y_t^\top Q y_t + u_t^\top \!R u_t\! \\ &+\!\mathbb{E}\bigg[ \hat{x}_{t+1|t+1}^{\top}\!\left( P\!-\! C^\top Q C  \right)\!\hat{x}_{t+1|t+1}^{u}\!\!+\!y_{t+1}^{\top} Q y_{t+1}\!\bigg] 
\end{align*}
\end{proof}
\section{Regret Decomposition} \label{SuppRegret}
In this section, using the Bellman Optimality Equation for the optimistic system $\tilde{\Theta} = \left(\ta, \tb, \tc \right)$, we derive the regret decomposition of applying optimal policy $\tilde{\policy}_*\left(\tp, \tk, \tl \right)$ for $\tilde{\Theta}$ in the unknown system $\Theta = \left(A, B, C \right)$. Notice that, this is equivalent to provide the regret decomposition for a system that is obtained via similarity transformation $\mathbf{S}$, \textit{i.e.} $A' = \mathbf{S}^{-1} A \mathbf{S}$, $B' = \mathbf{S}^{-1} B$, $C' = C \mathbf{S}$. Therefore, without loss of generality we will assume that $\mathbf{S} = I$ in the regret decomposition and the concentration bounds that is going to be used in the regret analysis. 

First, for given $\hat{x}_{t|t-1}$ and $y_t$, define the following expressions for time step $t+1$ using the model specified as subscript,

\begin{align}
    \hat{x}_{t|t,\tth} &= \left( I - \tl\tc \right)\hat{x}_{t|t-1} + \tl y_t \label{firstdef} \\
    y_{t+1,\tth} &= \tc \left( \ta - \tb \tk \right) \hat{x}_{t|t,\tth} + \tc \ta \left( x_t - \hat{x}_{t|t,\tth} \right) + \tc w_t + z_{t+1} \label{seconddef}\\
    \hat{x}_{t+1|t+1,\tth} &= \left( \ta - \tb \tk \right) \hat{x}_{t|t,\tth} + \tl \tc \ta \left( x_t - \hat{x}_{t|t,\tth} \right) + \tl \tc w_t + \tl z_{t+1} \label{thirddef}\\
    y_{t+1, \Theta} &= CA \hat{x}_{t|t,\tth} - CB\tk \hat{x}_{t|t,\tth} + C w_t + CA(x_t - \hat{x}_{t|t,\tth}) + z_{t+1} \label{fourthdef} \\
    \hat{x}_{t+1|t+1,\Theta} &= (I -LC)(A\hat{x}_{t|t,\Theta} - B\tk\hat{x}_{t|t,\tth}) + L y_{t+1, \Theta} \\
    &= (I-LC)(A-B\tk)\hat{x}_{t|t,\tth} + (I-LC)A(\hat{x}_{t|t,\Theta} - \hat{x}_{t|t,\tth}) + L y_{t+1, \Theta} \\
    &= (I-LC)(A-B\tk)\hat{x}_{t|t,\tth} + LC(A-B\tk)\hat{x}_{t|t,\tth} + LC w_t + LCA(x_t - \hat{x}_{t|t,\tth}) \nonumber \\
    \qquad \qquad & \qquad + (I-LC)A(\hat{x}_{t|t,\Theta} - \hat{x}_{t|t,\tth}) + Lz_{t+1} \\
    &= (A\!-\!B\tk)\hat{x}_{t|t,\tth} \!+\! LCw_t \!+\! LCA(x_t \!-\! \hat{x}_{t|t,\tth}) \!+\! (I\!-\!LC)A(\hat{x}_{t|t,\Theta} \!-\! \hat{x}_{t|t,\tth}) \!+\! Lz_{t+1} \label{lastdef}
\end{align}
Notice that given $x_{t|t-1}$ and $y_t$ first and fourth terms in \eqref{lastdef} are deterministic. 
From Lemma~\ref{LQGBellman}, we get
\begin{align}
    &J_*(\tth) + \hat{x}_{t|t,\tth}^\top \left( \tp - \tc^\top Q \tc  \right) \hat{x}_{t|t,\tth} + y_t^\top Q y_t \nonumber \\ 
    &=  y_t^\top Q y_t + u_t^\top R u_t + \mathbb{E}\left[ \hat{x}_{t+1|t+1,\tth}^\top \left( \tp - \tc^\top Q \tc  \right)  \hat{x}_{t+1|t+1,\tth} + y_{t+1,\tth}^\top Q  y_{t+1,\tth} \Big | \hat{x}_{t|t-1}, y_t \right] \nonumber \\
    &=  \mathbb{E}\bigg[ \left( (\ta - \tb \tk ) \hat{x}_{t|t,\tth} + \tl\tc\ta(x_t - \hat{x}_{t|t,\tth}) + \tl\tc w_t + \tl z_{t+1} \right)^\top \left( \tp - \tc^\top Q \tc  \right) \nonumber \\
    & \qquad\qquad\qquad\qquad \times \left((\ta - \tb \tk ) \hat{x}_{t|t,\tth} + \tl\tc\ta(x_t - \hat{x}_{t|t,\tth}) + \tl\tc w_t + \tl z_{t+1} \right) \Big | \hat{x}_{t|t-1}, y_t \bigg]\nonumber
    \\& \qquad + y_t^\top Q y_t + u_t^\top R u_t + \mathbb{E}\bigg[ \left( \tc (\ta - \tb \tk ) \hat{x}_{t|t,\tth} + \tc \ta (x_t - \hat{x}_{t|t,\tth}) + \tc w_t + z_{t+1} \right)^\top Q \nonumber \\
    &\qquad\qquad\qquad\qquad \times \left(  \tc (\ta - \tb \tk ) \hat{x}_{t|t,\tth} + \tc \ta (x_t - \hat{x}_{t|t,\tth}) + \tc w_t + z_{t+1} \right)  \Big | \hat{x}_{t|t-1}, y_t  \bigg] \nonumber \\
    &= \hat{x}_{t|t,\tth}^\top \left(\ta - \tb \tk \right)^\top \left( \tp - \tc^\top Q \tc  \right) \left(\ta - \tb \tk \right) \hat{x}_{t|t,\tth} + \mathbb{E} \bigg[ w_t \tc^\top \tl^\top  \left( \tp - \tc^\top Q \tc  \right) \tl \tc w_t  \bigg]  \nonumber \\
    &\qquad +\mathbb{E}\bigg[ \left(x_t - \hat{x}_{t|t,\tth}\right)^\top \ta^\top \tc^\top \tl^\top  \left( \tp - \tc^\top Q \tc  \right) \tl\tc\ta \left(x_t - \hat{x}_{t|t,\tth}\right) \Big | \hat{x}_{t|t-1}, y_t \bigg] \nonumber \\
    &\qquad + \mathbb{E}\bigg[ z_{t+1}^\top \tl^\top  \left( \tp - \tc^\top Q \tc  \right) \tl z_{t+1}  \bigg] + y_t^\top Q y_t + u_t^\top R u_t \nonumber \\
    &\qquad + \hat{x}_{t|t,\tth}^\top \left(\ta - \tb \tk \right)^\top \tc^\top Q \tc \left(\ta - \tb \tk \right) \hat{x}_{t|t,\tth} + \mathbb{E} \bigg[ w_t \tc^\top Q \tc w_t  \bigg] \nonumber \\ 
    &\qquad + \mathbb{E} \bigg[ \left(x_t - \hat{x}_{t|t,\tth}\right)^\top \ta^\top \tc^\top Q \tc \ta \left(x_t - \hat{x}_{t|t,\tth}\right) \Big | \hat{x}_{t|t-1}, y_t  \bigg]  + \mathbb{E} \bigg[ z_{t+1}^\top Q z_{t+1}  \bigg] \label{insert}
\end{align}
Using equations \eqref{thirddef} and \eqref{lastdef}, we have the following expression for $\mathbb{E}\left[ z_{t+1}^\top \tl^\top  \left( \tp - \tc^\top Q \tc  \right) \tl z_{t+1} \right]$:

\begin{align*}
    &\mathbb{E}\left[z_{t+1}^\top (\tl - L + L)^\top \left( \tp - \tc^\top Q \tc  \right) (\tl - L + L) z_{t+1}  \right] \\
    &= \mathbb{E}\left[z_{t+1}^\top L^\top \left( \tp - \tc^\top Q \tc  \right) L z_{t+1}  \right] + 2 \mathbb{E}\left[z_{t+1}^\top L^\top \left( \tp - \tc^\top Q \tc  \right) (\tl - L) z_{t+1}  \right] \\ &\qquad +\mathbb{E}\left[z_{t+1}^\top (\tl - L)^\top \left( \tp - \tc^\top Q \tc  \right) (\tl - L) z_{t+1}  \right] \\
    &= \mathbb{E}\left[ \hat{x}_{t+1|t+1,\Theta}^\top \left( \tp - \tc^\top Q \tc  \right) \hat{x}_{t+1|t+1,\Theta} \Big | \hat{x}_{t|t-1}, y_t, u_t \right] - \mathbb{E}\left[w_t^\top C^\top L^\top \left( \tp - \tc^\top Q \tc  \right) L C  w_t  \right]  \\ &-\!\left((A\!-\!B\tk)\hat{x}_{t|t,\tth} \!+\! (I\!-\!LC)A(\hat{x}_{t|t,\Theta} \!-\! \hat{x}_{t|t,\tth}) \right)^\top\!\left( \tp \!-\! \tc^\top Q \tc  \right)\! \left((A\!-\!B\tk)\hat{x}_{t|t,\tth} \!+\! (I\!-\!LC)A(\hat{x}_{t|t,\Theta} \!-\! \hat{x}_{t|t,\tth}) \right)   \\ &- \mathbb{E}\left[\left(x_{t}-\hat{x}_{t|t,\tth}\right)^\top A^\top C^\top L^\top \left( \tp - \tc^\top Q \tc  \right) L C A \left(x_{t}-\hat{x}_{t|t,\tth}\right)  \Big | \hat{x}_{t|t-1}, y_t \right] \\
    &+ 2 \mathbb{E}\left[z_{t+1}^\top L^\top \left( \tp - \tc^\top Q \tc  \right) (\tl - L) z_{t+1}  \right] +\mathbb{E}\left[z_{t+1}^\top (\tl - L)^\top \left( \tp - \tc^\top Q \tc  \right) (\tl - L) z_{t+1}  \right].
\end{align*}

Similarly, using equations \eqref{seconddef} and \eqref{fourthdef}, we have the following expression for $\mathbb{E}\left[ z_{t+1}^\top Q z_{t+1} \right]$: 

\begin{align*}
    \mathbb{E}\left[z_{t+1}^\top Q z_{t+1} \right] &= \mathbb{E}\left[ y_{t+1,\Theta}^\top Q y_{t+1,\Theta} \Big | \hat{x}_{t|t-1}, y_t, u_t \right] - \hat{x}_{t|t,\tth}^\top (A-B\tk)^\top C^\top Q C (A-B\tk) \hat{x}_{t|t,\tth} \\ &- \mathbb{E}\left[w_t^\top C^\top Q C w_t \right] - \mathbb{E}\left[\left(x_{t}-\hat{x}_{t|t,\tth}\right)^\top A^\top C^\top Q C A \left(x_{t}-\hat{x}_{t|t,\tth}\right)  \Big | \hat{x}_{t|t-1}, y_t \right]
\end{align*}
\normalsize
Inserting these to the equality in \eqref{insert}, we get 

\begin{align*}
    &J_*(\tth) +  \hat{x}_{t|t,\tth}^\top \left( \tp - \tc^\top Q \tc  \right) \hat{x}_{t|t,\tth} + y_t^\top Q y_t   \\ 
    &=  y_t^\top Q y_t + u_t^\top R u_t +  \hat{x}_{t|t,\tth}^\top \left(\ta - \tb \tk \right)^\top \left( \tp - \tc^\top Q \tc  \right) \left(\ta - \tb \tk \right) \hat{x}_{t|t,\tth} 
    \\ &+  \mathbb{E}  \bigg[ w_t \tc^\top \tl^\top  \left( \tp - \tc^\top Q \tc  \right) \tl \tc w_t   \bigg] \\
    &+   \mathbb{E}  \bigg[ \left(x_t - \hat{x}_{t|t,\tth}\right)^\top \ta^\top \tc^\top \tl^\top  \left( \tp - \tc^\top Q \tc  \right) \tl\tc\ta \left(x_t - \hat{x}_{t|t,\tth}\right) \Big | \hat{x}_{t|t-1}, y_t   \bigg] \\
    &+ \mathbb{E}\left[ \hat{x}_{t+1|t+1,\Theta}^\top \left( \tp - \tc^\top Q \tc  \right) \hat{x}_{t+1|t+1,\Theta} \Big | \hat{x}_{t|t-1}, y_t, u_t \right]   - \mathbb{E}\left[w_t^\top C^\top L^\top \left( \tp - \tc^\top Q \tc  \right) L C  w_t  \right] \\
    &-  \mathbb{E}\left[\left(x_{t}-\hat{x}_{t|t,\tth}\right)^\top A^\top C^\top L^\top \left( \tp - \tc^\top Q \tc  \right) L C A \left(x_{t}-\hat{x}_{t|t,\tth}\right)  \Big | \hat{x}_{t|t-1}, y_t \right] \\
    &-\!\left((A\!-\!B\tk)\hat{x}_{t|t,\tth} \!+\! (I\!-\!LC)A(\hat{x}_{t|t,\Theta} \!-\! \hat{x}_{t|t,\tth}) \right)^\top\!\left( \tp \!-\! \tc^\top Q \tc  \right)\! \left((A\!-\!B\tk)\hat{x}_{t|t,\tth} \!+\! (I\!-\!LC)A(\hat{x}_{t|t,\Theta} \!-\! \hat{x}_{t|t,\tth}) \right)  \\
    &+  \hat{x}_{t|t,\tth}^\top \left(\ta - \tb \tk \right)^\top \tc^\top Q \tc \left(\ta - \tb \tk \right) \hat{x}_{t|t,\tth}  +  \mathbb{E}  \bigg[ w_t \tc^\top Q \tc w_t   \bigg]\\ 
    &+   \mathbb{E}  \bigg[ \left(x_t - \hat{x}_{t|t,\tth}\right)^\top \ta^\top \tc^\top Q \tc \ta \left(x_t - \hat{x}_{t|t,\tth}\right) \Big | \hat{x}_{t|t-1}, y_t    \bigg] \\
    &+  \mathbb{E}\left[ y_{t+1,\Theta}^\top Q y_{t+1,\Theta} \Big | \hat{x}_{t|t-1}, y_t, u_t \right] - \hat{x}_{t|t,\tth}^\top (A-B\tk)^\top C^\top Q C (A-B\tk) \hat{x}_{t|t,\tth} \\
    & -\mathbb{E}\left[w_t^\top C^\top Q C w_t \right]  -  \mathbb{E}\left[\left(x_{t}-\hat{x}_{t|t,\tth}\right)^\top A^\top C^\top Q C A \left(x_{t}-\hat{x}_{t|t,\tth}\right)  \Big | \hat{x}_{t|t-1}, y_t \right] \\
    &+ 2 \mathbb{E}\left[z_{t+1}^\top L^\top \left( \tp - \tc^\top Q \tc  \right) (\tl - L) z_{t+1}  \right] +\mathbb{E}\left[z_{t+1}^\top (\tl - L)^\top \left( \tp - \tc^\top Q \tc  \right) (\tl - L) z_{t+1}  \right] 
\end{align*}
Hence,
\begin{equation*}
    \sum_{t=0}^{T-\Texp}  J_*(\tth) \!+\! R_1 \!+\! R_2 = \sum_{t=0}^{T-\Texp} \left( y_t^\top Q y_t \!+\! u_t^\top R u_t \right) \!+\! R_3 \!+\! R_4 \!+\! R_5 \!+\! R_6 \!+\! R_7 \!+\! R_8 \!+\! R_9 \!+\! R_{10} \!+\! R_{11}
\end{equation*}
where 

\begin{align}
    R_1& \!=\!\! \sum_{t=0}^{T-\Texp} \left \{ \hat{x}_{t|t,\tth}^\top \left( \tp - \tc^\top Q \tc  \right) \hat{x}_{t|t,\tth} - \mathbb{E}\left[ \hat{x}_{t+1|t+1,\Theta}^\top \left( \tp - \tc^\top Q \tc  \right) \hat{x}_{t+1|t+1,\Theta} \Big | \hat{x}_{t|t-1}, y_t, u_t \right]  \right \} \\
    R_2& \!=\!\! \sum_{t=0}^{T-\Texp} \left \{ y_t^\top Q y_t - 
    \mathbb{E}\left[ y_{t+1,\Theta}^\top Q y_{t+1,\Theta} \Big | \hat{x}_{t|t-1}, y_t, u_t \right] \right \} \\
    R_3& \!=\!\! \sum_{t=0}^{T-\Texp} \left \{ \hat{x}_{t|t,\tth}^\top (\ta - \tb \tk )^\top \tc^\top Q \tc (\ta - \tb \tk) \hat{x}_{t|t,\tth} - \hat{x}_{t|t,\tth}^\top (A-B\tk)^\top C^\top Q C (A-B\tk) \hat{x}_{t|t,\tth} \right \} \\
    R_4& \!=\!\! \sum_{t=0}^{T-\Texp} \! \left\{ \hat{x}_{t|t,\tth}^\top (\ta \!-\! \tb \tk )^\top (\tp \!-\! \tc^\top Q \tc) (\ta \!-\! \tb \tk) \hat{x}_{t|t,\tth} \!-\! \hat{x}_{t|t,\tth}^\top (A\!-\!B\tk)^\top (\tp - \tc^\top Q \tc) (A\!-\!B\tk) \hat{x}_{t|t,\tth}
     \right \} \\
    R_5& \!=\! -\!\! \sum_{t=0}^{T-\Texp} \left \{ 2\hat{x}_{t|t,\tth}^\top (A-B\tk)^\top (\tp - \tc^\top Q \tc) (I-LC)A(\hat{x}_{t|t,\Theta} -\hat{x}_{t|t,\tth} ) \right \} \\
    R_6& \!=\! -\!\! \sum_{t=0}^{T-\Texp} \left \{ (\hat{x}_{t|t,\Theta} -\hat{x}_{t|t,\tth} )^\top A^\top (I-LC)^\top (\tp - \tc^\top Q \tc) (I-LC)A(\hat{x}_{t|t,\Theta} -\hat{x}_{t|t,\tth} ) \right \} \\
    R_7& \!=\!\! \sum_{t=0}^{T-\Texp} \left\{ \mathbb{E} \left[ w_t^\top \tc^\top Q \tc w_t   \right] - \mathbb{E}\left[w_t^\top C^\top Q C w_t \right] \right\}\\
    R_8& \!=\!\! \sum_{t=0}^{T-\Texp} \left \{ \mathbb{E}  \left[ w_t^\top \tc^\top \tl^\top  \left( \tp - \tc^\top Q \tc  \right) \tl \tc w_t   \right]   - \mathbb{E}\left[w_t^\top C^\top L^\top \left( \tp - \tc^\top Q \tc  \right) L C  w_t  \right] \right \} \\ \label{R9}
    R_9& \!=\!\! \sum_{t=0}^{T-\Texp} \bigg \{ \mathbb{E}  \bigg[ \left(x_t - \hat{x}_{t|t,\tth}\right)^\top \ta^\top \tc^\top Q \tc \ta \left(x_t - \hat{x}_{t|t,\tth}\right) \Big | \hat{x}_{t|t-1}, y_t   \bigg] \\ &\qquad \qquad-  \mathbb{E}\left[\left(x_{t}-\hat{x}_{t|t,\tth}\right)^\top A^\top C^\top Q C A \left(x_{t}-\hat{x}_{t|t,\tth}\right)  \Big | \hat{x}_{t|t-1}, y_t \right] \bigg \} \nonumber \\ \label{R10}
    R_{10}& \!=\!\! \sum_{t=0}^{T-\Texp} \bigg \{ \mathbb{E} \bigg[ \left(x_t - \hat{x}_{t|t,\tth}\right)^\top \ta^\top \tc^\top \tl^\top  \left( \tp - \tc^\top Q \tc  \right) \tl\tc\ta \left(x_t - \hat{x}_{t|t,\tth}\right) \Big | \hat{x}_{t|t-1}, y_t   \bigg] \\
    &\qquad \qquad - \mathbb{E}\left[\left(x_{t}-\hat{x}_{t|t,\tth}\right)^\top A^\top C^\top L^\top \left( \tp - \tc^\top Q \tc  \right) L C A \left(x_{t}-\hat{x}_{t|t,\tth}\right)  \Big | \hat{x}_{t|t-1}, y_t \right] \bigg \} \\
    R_{11}& \!=\!\! \sum_{t=0}^{T-\Texp} \bigg \{ 2 \mathbb{E}\left[z_{t+1}^\top L^\top\!\! \left( \tp \!-\! \tc^\top Q \tc  \right) (\tl \!-\! L) z_{t+1} \right] \!+\!\mathbb{E}\left[z_{t+1}^\top (\tl \!-\! L)^\top \!\! \left( \tp \!-\! \tc^\top Q \tc  \right) (\tl \!-\! L) z_{t+1}  \right]  \bigg \}
\end{align}

Thus, on event $\mathcal{E} \cap \mathcal{F} \cap \mathcal{G} $,
\begin{align*}
    &\sum_{t=0}^{T-\Texp}\!\!\!\! \left( y_t^\top Q y_t \!+\! u_t^\top R u_t \right) \!=\!\! \sum_{t=0}^{T-\Texp} \!\!\! J_*(\tth) \!+\! R_1 \!+\! R_2 \!-\! R_3 \!-\! R_4 \!-\! R_5 \!-\! R_6 \!-\! R_7 \!-\! R_8 \!-\! R_9 \!-\! R_{10} \!-\! R_{11} \\
    & \qquad \leq (T-\Texp) J_*(\Theta) \!+\! R_1 \!+\! R_2 \!-\! R_3 \!-\! R_4 \!-\! R_5 \!-\! R_6 \!-\! R_7 \!-\! R_8 \!-\! R_9 \!-\! R_{10} \!-\! R_{11} \!+\! T^{2/3}
\end{align*}
where the last inequality follows from the fact that $\tth$ is the optimistic parameter from the confidence sets such that $J_*(\tth) \leq J_*(\Theta) + \frac{1}{T^{1/3}}$ and on event $\mathcal{E}$, $\Theta \in \mathcal{C}$. Therefore, on event $\mathcal{E} \cap \mathcal{F} \cap \mathcal{G} $,
\begin{equation}
    \reg(T-\Texp) \leq R_1 + R_2 - R_3 - R_4 - R_5 - R_6 - R_7 - R_8 - R_9 - R_{10} \!-\! R_{11} + T^{2/3}.
\end{equation}
The following section contains the bounds on individual pieces. 
\section{Regret Upper Bound}
\label{SuppRegretTotal}
In this section we will provide bounds on each term in the regret decomposition. It will be useful to recall and denote the following bounds obtained after $\Texp>T_0$ time steps of exploration before starting commit phase of \Alg on the event of $\mathcal{E} \cap \mathcal{F} \cap \mathcal{G} $. Note that, without loss of generality, we obtain the regret analysis with the confidence sets with the similarity transformation $\mathbf{S} = I$, since any similarity transformation of the underlying system $\Theta$ will give a system with the same behavior. 
\small
\begin{align*}
    \|A - \ta \| &\leq \Delta A \coloneqq  \left( \frac{62n\|\mathbf{H}\| + 14n\sigma_n(\mathbf{H})}{\sigma_n^2(\mathbf{H})} \right) \frac{R_w+R_e+R_z}{\sigma_u\sqrt{\Texp-H+1}}, \\
    \|B - \tb \| &\leq \Delta B \coloneqq \frac{14n(R_w+R_e+R_z)}{\sigma_u\sqrt{(\Texp - H+1)\sigma_n(\mathbf{H})}}, \\
    \|C - \tc \| &\leq \Delta C \coloneqq \frac{14n(R_w+R_e+R_z)}{\sigma_u\sqrt{(\Texp - H+1)\sigma_n(\mathbf{H})}},\\
    \|\tsig  - \Sigma \| &\leq \Delta \Sigma \coloneqq \frac{\Phi(A)^2 (4\|C\|+ 2)\|\Sigma \|^2 + \sigma_z^2 (4\Phi(A) + 2)  \| \Sigma\| }{\sigma_z^2(1 - \upsilon^2)} \max\left\{\Delta A, \Delta C \right\},  \\
    \|\tl - L\| &\leq \Delta L \coloneqq \Delta C \bigg( \sigma_z^{-2} \| \sig\| +
    4\sigma_z^{-4} \| \sig\|^2
    + 6\sigma_z^{-4} \|C\| \| \sig\|^2 + 2\sigma_z^{-4} \|C\|^2 \| \sig\|^2   \bigg) \\ 
    &\qquad +\Delta \sig \bigg(\sigma_z^{-2}(\|C\|\!+\!2) \!+\! \sigma_z^{-4} \|C\|^3 (1\!+\!\| \sig\| ) \!+\! 10 \sigma_z^{-4} \|C\|^2 (1 \!+\! \| \sig \|) \!+\!  24 \sigma_z^{-4} \|C\| (1\!+\!\| \sig \|) \!+\! 16 \sigma_z^{-4} (1 \!+\! \| \sig \|)  \bigg) \\ \|\hat{x}_{t|t,\tth}\| &\leq \tilde{\mathcal{X}} \coloneqq \frac{2\zeta \left(\|C\| \bar{\Delta} + \left(\| C\| \|\Sigma\|^{1/2} + \sigma_z\right) \sqrt{2n\log(2nT/\delta)}\right) }{1-\rho} , \\
    \|y_t \| &\leq \mathcal{Y} \coloneqq \| C\| \tilde{\mathcal{X}}  + \| C\| \bar{\Delta} + \left(\| C\| \|\Sigma\|^{1/2} + \sigma_z \right) \sqrt{2m\log(2mT/\delta)} 
\end{align*}
\normalsize
for all $t \leq T$, where $\bar{\Delta} = 10\left(\frac{\bar{\kappa} }{1-\sigma} + \frac{\bar{\beta} \bar{\xi} }{(1-\sigma)^2} \right)\left(\| C\| \|\Sigma\|^{1/2} + \sigma_z\right) \sqrt{2m\log(2mT/\delta)}$, for $\bar{\kappa} = \Phi(A) \Delta L + 2\zeta (\beta_A + \Gamma \beta_B)$, $\bar{\beta} = 2\zeta \beta_C (\Phi(A) + 2(\beta_A + \Gamma \beta_B) + 2(\beta_A + \Gamma \beta_B) )$ and $\bar{\xi} = \zeta (\rho + 2(\beta_A + \Gamma \beta_B)) + \|B\|\Gamma \Delta L $. Notice that all the concentration results are $\tilde{\OO}\left(\frac{1}{\sqrt{\Texp}}\right)$ where $\tilde{\OO}(\cdot)$ hides the problem dependent constants and logarithm. All the theorems and lemmas in this section are given in $\tilde{\OO}(\cdot)$ notation. The exact expressions are given in the last lines of proofs.

\subsection[Bound on R1]{Bounding $R_1$ on the event of $\mathcal{E} \cap \mathcal{F} \cap \mathcal{G} $}
\begin{lemma}\label{lemmaR1}
Suppose Assumption~\ref{Stabilizable set} holds and system is explored for  $\Texp > T_0$ time steps. For any $\delta \in (0,1) $, given $\mathcal{E} \cap \mathcal{F} \cap \mathcal{G} $ holds, 
\begin{align*}
    &R_1 \!=\! \tilde{\OO}\left( \frac{T\!-\!\Texp}{\sqrt{\Texp}} \right) 
\end{align*}
with probability at least $1-\delta$. 
\end{lemma}

\begin{proof}
Define $\tilde{f}_t, f_t, \tilde{v}_t$ and $v_t$ such that 
\begin{align*}
    \tilde{f}_t &= \ta(I-\tl\tc)\hat{x}_{t-1|t-2,\tth} + \ta\tl y_{t-1} + \tb u_{t-1}, \\
    f_t &= A(I-LC)\hat{x}_{t-1|t-2,\tth} + AL y_{t-1} + B u_{t-1} \\
    \tilde{v}_t &= \tl\tc\ta(x_{t-1} - (I-\tl\tc)\hat{x}_{t-1|t-2,\tth} - \tl y_{t-1}) +  \tl\tc w_{t-1} + \tl z_{t}, \\
    v_t &= LCA(x_{t-1} - (I-LC)\hat{x}_{t-1|t-2,\tth} - L y_{t-1}) + LC w_{t-1} + L z_{t}
\end{align*}
$R_1$ is decomposed as follows, 
\begin{align*}
    R_1 &= \hat{x}_{0|0,\tth}^\top (\tp - \tc^\top Q \tc) \hat{x}_{0|0,\tth} - \mathbb{E}\left[ \hat{x}_{T+1|T+1,\Theta}^\top \left( \tp - \tc^\top Q \tc  \right) \hat{x}_{T+1|T+1,\Theta} \Big | \hat{x}_{T|T-1}, y_T, u_T \right] \\
    &+\!\!\sum_{t=1}^{T-\Texp} \left \{ \hat{x}_{t|t,\tth}^\top \left( \tp \!-\! \tc^\top Q \tc  \right) \hat{x}_{t|t,\tth} - \mathbb{E}\left[ \hat{x}_{t|t,\Theta}^\top \left( \tp - \tc^\top Q \tc  \right) \hat{x}_{t|t,\Theta} \Big | \hat{x}_{t-1|t-2}, y_{t-1}, u_{t-1} \right]  \right \}
\end{align*}

Since $\hat{x}_{0|0,\tth} =0$ and $\tp - \tc^\top Q \tc$ is positive semidefinite, the first line is bounded by zero. Using the definitions above, the remaining is decomposed as follows, 
\begin{align*}
    R_1 &\leq \sum_{t=1}^{T-\Texp} \left\{ \tilde{f}_t^\top \left( \tp - \tc^\top Q \tc  \right) \tilde{f}_t  - f_t^\top \left( \tp - \tc^\top Q \tc  \right) f_t\right\} + 2 \sum_{t=1}^T \left\{ \tilde{f}_t^\top \left( \tp - \tc^\top Q \tc  \right) \tilde{v}_t \right\} \\
    &\qquad \qquad + \sum_{t=1}^{T-\Texp} \left\{ \tilde{v}_t^\top \left( \tp - \tc^\top Q \tc  \right) \tilde{v}_t - \mathbb{E}\left[ v_t^\top \left( \tp - \tc^\top Q \tc  \right) v_t \Big | \hat{x}_{t-1|t-2}, y_{t-1}, u_{t-1} \right] \right \}
\end{align*}
Each term will be bounded separately: \\
\textbf{1) }Let $G_1 = 2 \sum_{t=1}^T \left\{ \tilde{f}_t^\top \left( \tp - \tc^\top Q \tc  \right) \tilde{v}_t \right\}$. Let $q_t^\top =  \tilde{f}_t^\top(\tp - \tc^\top Q \tc)$ and on the event of $\mathcal{E} \cap \mathcal{F} \cap \mathcal{G}$
\begin{equation*}
    \sum_{t=1}^{T-\Texp} \left\{ \tilde{f}_t^\top \left( \tp - \tc^\top Q \tc  \right) \tilde{v}_t \right\} = \sum_{t=1}^{T-\Texp} q_t^\top \tilde{v}_t = \sum_{t=1}^{T-\Texp} \sum_{i=1}^n q_{t,i} \tilde{v}_{t,i} = \sum_{i=1}^n \sum_{t=1}^{T-\Texp} q_{t,i} \tilde{v}_{t,i}
\end{equation*}
Observe that conditioned on current observations, $\tilde{v}_{t}$ is zero mean and entrywise $R'$-sub-Gaussian, where
\[R' = \zeta \left( \left((\Delta A + \Phi(A))\sqrt{\Delta\sig + \| \sig \|} + \sigma_w \right) \left(\| C\| + \Delta C \right)+ \sigma_z\right)\]. 

Let $M_{T,i} = \sum_{t=1}^{T-\Texp} q_{t,i} \tilde{v}_{t,i}$. By Theorem~\ref{selfnormalized}, the following holds with probability at least $1-\delta/(4n)$, for any $T\geq 0$ and $\lambda>0$,

\begin{equation*}
    M_{T,i}^2 \leq 2R'^2 \left(\lambda + \sum_{t=1}^{T-\Texp} q_{t,i}^2 \right) \log \left(\frac{4n}{\delta \sqrt{\lambda}} \left(\lambda + \sum_{t=1}^{T-\Texp} q_{t,i}^2 \right)^{1/2} \right)
\end{equation*}
Recalling Assumption~\ref{Stabilizable set}, on $\mathcal{E} \cap \mathcal{F} \cap \mathcal{G} $, 

\begin{align*}
\| q_t\| &\leq \rho (\| \tp\| + \|\tc^\top Q \tc \|)  \tilde{\mathcal{X}} \leq \rho \left(D + \| Q\| \left( \| C\| + \Delta C \right)^2  \right)  \tilde{\mathcal{X}}.
\end{align*} 
Thus, $q_{t,i} \leq \rho \left(D + \| Q\| \left( \| C\| + \Delta C \right)^2  \right) \tilde{\mathcal{X}} \coloneqq \mathcal{Q}_T$. Combining these results with union bound, with probability $1-\delta/4$ we have 
\begin{align}
     G_1 &\leq 2R' n  \sqrt{ (\lambda + \mathcal{Q}_T^2 (T-\Texp)) \log \left( \frac{n\sqrt{\lambda + (T-\Texp)\mathcal{Q}_T^2}}{\delta\sqrt{\lambda}}\right) } \nonumber \\
     &= \tilde{\OO}\left( \rho \zeta \tilde{\mathcal{X}}  \left(D \!+\! \| Q\| \left( \| C\| \!+\! \Delta C \right)^2  \right)  \left( \left((\Delta A \!+\! \Phi(A))\sqrt{\Delta\sig \!+\! \| \sig \|} \!+\! \sigma_w \right) \left(\| C\| \!+\! \Delta C \right)\!+\! \sigma_z\right) \sqrt{T \!-\!\Texp }\right) \nonumber
\end{align} 
where $\tilde{\OO}$ hides logarithmic terms. 

\noindent\textbf{2) }Now consider $\sum_{t=1}^{T-\Texp} \left\{ \tilde{v}_t^\top \left( \tp - \tc^\top Q \tc  \right) \tilde{v}_t - \mathbb{E}\left[ v_t^\top \left( \tp - \tc^\top Q \tc  \right) v_t \Big | \hat{x}_{t-1|t-2}, y_{t-1}, u_{t-1} \right] \right \}$. Adding and subtracting $\sum_{t=1}^{T-\Texp} \left\{ v_t^\top \left( \tp - \tc^\top Q \tc  \right) v_t \right \}$, we get following terms,
\begin{align*}
    &\sum_{t=1}^{T-\Texp} \left\{ v_t^\top \left( \tp - \tc^\top Q \tc  \right) v_t - \mathbb{E}\left[ v_t^\top \left( \tp - \tc^\top Q \tc  \right) v_t \Big | \hat{x}_{t-1|t-2}, y_{t-1}, u_{t-1} \right] \right \} \\ &\qquad \qquad \qquad + \sum_{t=1}^{T-\Texp} \left\{ \tilde{v}_t^\top \left( \tp - \tc^\top Q \tc  \right) \tilde{v}_t - v_t^\top \left( \tp - \tc^\top Q \tc  \right) v_t  \right \}
\end{align*}
We combine the last summation with $\sum_{t=1}^{T-\Texp} \left\{ \tilde{f}_t^\top \left( \tp - \tc^\top Q \tc  \right) \tilde{f}_t  - f_t^\top \left( \tp - \tc^\top Q \tc  \right) f_t\right\}$ and denote as $G_2$. We decompose $G_2$ as follows 
\begin{align*}
    G_2 &= \sum_{t=1}^{T-\Texp} \left\{ \tilde{v}_t^\top \left( \tp - \tc^\top Q \tc  \right) \tilde{v}_t + \tilde{f}_t^\top \left( \tp - \tc^\top Q \tc  \right) \tilde{f}_t  - f_t^\top \left( \tp - \tc^\top Q \tc  \right) f_t - v_t^\top \left( \tp - \tc^\top Q \tc  \right) v_t  \right \} \\
    &=\!\!\!\!\! \sum_{t=1}^{T-\Texp}\!\!\Big\{ (\tilde{v}_t\!-\!v_t)^\top \left( \tp\!-\!\tc^\top Q \tc  \right) \tilde{v}_t\!+\!v_t^\top \!\left( \tp\!-\!\tc^\top Q \tc  \right) (\tilde{v}_t\!-\!v_t)\!\\
    &\qquad \qquad \qquad+\! (\tilde{f}_t\!-\!f_t)^\top \left( \tp\!-\!\tc^\top Q \tc  \right) \tilde{f}_t\!+\!f_t^\top\!\left( \tp\!-\!\tc^\top Q \tc  \right) (\tilde{f}_t\!-\!f_t) \Big\} 
\end{align*}
Now notice that 
\begin{align*}
    \tilde{v}_t\!-\!v_t &= (\tl\tc\ta - LCA)x_t - (\tl\tc\ta - LCA)\hat{x}_{t-1|t-2,\tth} + (\tl\tc\ta\tl\tc - LCALC)\hat{x}_{t-1|t-2,\tth} \\
    &\qquad \qquad \qquad \qquad \qquad \qquad- (\tl\tc\ta\tl - LCAL)y_{t-1} + (\tl\tc -LC)w_{t-1} + (\tl - L)z_t \\
    \tilde{f}_t\!-\!f_t &= (\ta - A)\hat{x}_{t-1|t-2,\tth} - (\ta\tl\tc - ALC)\hat{x}_{t-1|t-2,\tth} + (\ta\tl - AL)y_{t-1} - (\tb -B)\tk \hat{x}_{t-1|t-1,\tth}
\end{align*}

Recall that $\hat{x}_{i-1|i-2,\tth} = (\ta - \tb \tk)\hat{x}_{i-2|i-2,\tth}$ and since $x_i - \hat{x}_{i-1|i-i,\tth}$ is $\|\tsig\|^{1/2}$-sub-Gaussian for all $i\leq T$, using Lemma~\ref{subgauss lemma}, we get $\|x_i - \hat{x}_{i-1|i-2,\tth} \| \leq  \sqrt{\|\sig\| + \|\Delta \sig\|}\sqrt{2n\log\left(\frac{8n(T-\Texp)}{\delta}\right)} $ for all $i\leq T$ with probability at least $1-\delta/4$. 
Similarly, $\|w_i\| \leq  \sigma_w\sqrt{2n\log\left(\frac{8n(T-\Texp)}{\delta}\right)} $, $\|z_i\| \leq  \sigma_z \sqrt{2m\log\left(\frac{8m(T-\Texp)}{\delta}\right)} $ for all $i\leq T$ with probability at least $1-\delta/4$ respectively. On the event of $\mathcal{E} \cap \mathcal{F} \cap \mathcal{G} $ consider the following decompositions: 

\begin{align*}
    \|\tl \tc \ta - LCA \| &\leq \Big(\Delta L\|C\| \Phi(A) + \zeta \Phi(A) \Delta C + \zeta \|C\| \Delta A + \zeta \Delta C \Delta A \\
    &\enskip + \Phi(A) \Delta L \Delta C + \Delta L \|C\| \Delta A + \Delta L \Delta C \Delta A \Big) \\
    \|\tl\tc\ta\tl\tc - LCALC \| &\leq \Big( 2 \Delta L \Phi(A) \|C\|^2 \| L\| + 2 \Delta C \Phi(A) \|L\|^2 \| C\| + \Delta A \|L\|^2 \| C\|^2 \\
    &\enskip+ 4 \Delta L \Delta C \Phi(A) \|L\| \|C\| + 2 \Delta L \Delta A  \| L\| \|C\|^2 \\
    &\enskip+ 2 \Delta A \Delta C \|L\|^2 \|C\| + \Delta L^2 \Phi(A) \|C\|^2 + \Delta C^2 \Phi(A) \|L\|^2 + 2 \Delta L^2 \Delta C \Phi(A) \|C\| \\
    &\enskip+ 2 \Delta C^2 \Delta L \Phi(A) \|L\| + 4 \Delta L \Delta C \Delta A \|L\| \|C\| + \Delta L^2 \Delta A \|C\|^2 + \Delta C^2 \Delta A \|L\|^2 \\
    &\enskip+ 2 \Delta L^2 \Delta C \Delta A \|C\| + 2 \Delta C^2 \Delta L \Delta A \|L\| + \Delta L^2 \Delta C^2 \Phi(A) + \Delta L^2 \Delta C^2 \Delta A \Big) \\
    \|\tl\tc\ta\tl - LCAL \| &\leq \Big( 2 \Delta L \Phi(A) \|C\| \| L\| + \Delta C \Phi(A) \|L\|^2 + \Delta A \|L\|^2 \| C\| + 2 \Delta L \Delta C \Phi(A) \|L\| \\
    &\enskip+ \Delta L \Delta A  \| L\| \|C\| + \Delta A \Delta C \|L\|^2 + \Delta L^2 \Phi(A) \|C\| + \Delta L^2 \Delta C \Phi(A) \\
    &\enskip+ 2 \Delta L \Delta C \Delta A \|L\| + \Delta L^2 \Delta A \|C\| + \Delta L^2 \Delta C \Delta A  \Big).
\end{align*}

Using the boundedness of $\hat{x}_{t|t,\tth}, y_t, x_t, w_t, z_t$ and the fact that estimation error in all system matrices is $\OO(1/\sqrt{\Texp})$, on the event of $\mathcal{E} \cap \mathcal{F} \cap \mathcal{G} $ with probability $1-\delta$ we have : 

\begin{align*}
    G_2 &\leq (T \!-\! \Texp) \left(D \!+\! \| Q\| \left( \| C\| \!+\! \Delta C \right)^2  \right) \left( 2 \|v_t\| \|\tilde{v}_t \!-\! v_t \| + \|\tilde{v}_t \!-\! v_t \|^2 + 2 \|f_t\| \|\tilde{f}_t \!-\! f_t \| + \|\tilde{f}_t \!-\! f_t \|^2 \right) \\
    &= \tilde{\OO}\left(\frac{T-\Texp}{\sqrt{\Texp}}\right)
\end{align*}

where the last line is obtained by observing that the dominating term has only one concentration result. 

\noindent\textbf{3) }Now we focus on the remaining term in $R_1$, 
\begin{equation*}
    G_3 = \sum_{t=1}^{T-\Texp} \left\{ v_t^\top \left( \tp - \tc^\top Q \tc  \right) v_t - \mathbb{E}\left[ v_t^\top \left( \tp - \tc^\top Q \tc  \right) v_t \Big | \hat{x}_{t-1|t-2}, y_{t-1}, u_{t-1} \right] \right \}
\end{equation*}
Combining previous bounds on the components of $v_t$, with probability $1-\delta/2$ we have 
\begin{align}
    \|v_t\| &\leq \Omega \coloneqq \zeta \| C\| \Phi(A) \|\sig\|^{1/2}\sqrt{2n\log\left(\frac{8n(T-\Texp)}{\delta}\right)} \label{R1eq} \\
    & \qquad\qquad \qquad + \zeta \|C\| \sigma_w \sqrt{2n\log\left(\frac{8n(T-\Texp)}{\delta}\right)}  + \zeta \sigma_z \sqrt{2m\log\left(\frac{8m(T-\Texp)}{\delta}\right)} \nonumber
\end{align}
On the event of $\mathcal{E} \cap \mathcal{F} \cap \mathcal{G} $, let $V_t = v_t^\top \left( \tp - \tc^\top Q \tc  \right) v_t - \mathbb{E}\left[ v_t^\top \left( \tp - \tc^\top Q \tc  \right) v_t \Big | \hat{x}_{t-1|t-2}, y_{t-1}, u_{t-1} \right]$ and define its truncated version $\bar{V}_t = V_t \mathbbm{1}_{V_t \leq  2\left(D + \| Q\| \left( \| C\| + \Delta C \right)^2  \right) \Omega} $. Thus $G_3 = \sum_{t=1}^{T-\Texp} V_t$ and define $\bar{G}_3 =\sum_{t=1}^{T-\Texp} \bar{V}_t $. By Lemma~\ref{basicprob}, 

\begin{align}
    &\Pr\left( G_3 > 2\left(D + \| Q\| \left( \| C\| + \Delta C \right)^2  \right) \Omega \sqrt{2(T-\Texp) \log\frac{2}{\delta}} \right) \\
    &\leq\!\Pr\!\left(\max_{1\leq t\leq T - \Texp}\!\!\!\!\!V_t\!\geq\! 2\left(D\!+\!\| Q\| \left( \| C\|\!+\!\Delta C \right)^2  \right) \Omega \right)\!\\
    &\qquad \qquad +\!\Pr\left( \bar{G}_3\!>\!2\left(D\!+\!\| Q\| \left( \| C\|\!+ \Delta C \right)^2  \right) \Omega \sqrt{2(T\!-\!\Texp) \log\frac{2}{\delta}} \right).
\end{align}
From \eqref{R1eq} and Theorem~\ref{azuma}, each term on the right hand side is bounded by $\delta/2$. Thus, with probability $1-\delta$,

\begin{equation*}
    G_3 \leq 2\left(D + \| Q\| \left( \| C\| + \Delta C \right)^2  \right) \Omega \sqrt{2(T-\Texp) \log\frac{2}{\delta}} = \tilde{\OO}\left(\sqrt{T-\Texp} \right).
\end{equation*}
Since $R_1 = G_1 + G_2 + G_3$, the dominating term is $G_2$ which proves the lemma. 
\end{proof}

\subsection[Bound on R2]{Bounding $R_2$ on the event of $\mathcal{E} \cap \mathcal{F} \cap \mathcal{G} $}
\begin{lemma}
Suppose Assumption~\ref{Stabilizable set} holds and system is explored for  $\Texp > T_0$ time steps. For any $\delta \in (0,1) $, given $\mathcal{E} \cap \mathcal{F} \cap \mathcal{G} $ holds,
\begin{align*}
    \sum_{t=0}^{T-\Texp} \left \{ y_t^\top Q y_t - \mathbb{E}\left[ y_{t+1,\Theta}^\top Q y_{t+1,\Theta} \Big | \hat{x}_{t|t-1}, y_t, u_t \right] \right \} = \tilde{\OO}\left( \sqrt{T-\Texp} \right) 
\end{align*}
with probability at least $1-\delta$.
\end{lemma}

\begin{proof}
Define $f_t = CA(I-LC)\hat{x}_{t-1|t-2} + CAL y_{t-1} + CB u_{t-1}$, $v_t = CA(x_{t-1} - (I-LC)\hat{x}_{t-1|t-2} - L y_{t-1}) +  C w_{t-1} + z_{t}$. $R_2$ can be written as follows,
\begin{align*}
    R_2 &= y_0^\top Q y_0 - \mathbb{E}\left[ y_{T+1,\Theta}^\top Q y_{T+1,\Theta} \Big | \hat{x}_{T|T-1}, y_T, u_T \right] \!+\! \sum_{t=1}^{T-\Texp} \left \{y_t^\top Q y_t - \mathbb{E}\left[ y_t^\top Q y_t \Big | \hat{x}_{t-1|t-2}, y_{t-1}, u_{t-1} \right]  \right \}
\end{align*}
Since $y_{0} = 0$ and $Q$ is positive semidefinite, the first line is bounded by zero. Using the definitions above, the remaining can be decomposed as follows, 
\begin{align*}
    R_2 &\leq  2 \sum_{t=1}^{T-\Texp} \left\{ f_t^\top Q v_t \right\} + \sum_{t=1}^{T-\Texp} \left\{ v_t^\top Q v_t - \mathbb{E}\left[ v_t^\top Q v_t \Big | \hat{x}_{t-1|t-2}, y_{t-1}, u_{t-1} \right] \right \}
\end{align*}

Notice that these terms are all defined under true model. Each term will be bounded separately. Let $G_1 = 2 \sum_{t=1}^{T-\Texp} \left\{ f_t^\top Q v_t \right\}$. Let $q_t^\top =  f_t^\top Q$ and on the event of $\mathcal{E} \cap \mathcal{F} \cap \mathcal{G}$
\begin{equation*}
    \sum_{t=1}^{T-\Texp} \left\{ f_t^\top Q v_t \right\} = \sum_{t=1}^{T-\Texp} q_t^\top v_t = \sum_{t=1}^{T-\Texp} \sum_{i=1}^n q_{t,i} v_{t,i} = \sum_{i=1}^n \sum_{t=1}^{T-\Texp} q_{t,i} v_{t,i}
\end{equation*}
Observe that conditioned on current observations, $v_{t}$ is zero mean and entrywise $R''$-sub-Gaussian, where
\[R'' = (\|C\| \Phi(A) \| \sig \|^{1/2} + \sigma_w \|C \|  + \sigma_z).\] 

Let $M_{T,i} = \sum_{t=1}^{T-\Texp} q_{t,i} v_{t,i}$. By Theorem~\ref{selfnormalized}, the following holds with probability at least $1-\delta/(2n)$, for any $T\geq 0$ and $\lambda>0$,

\begin{equation*}
    M_{T,i}^2 \leq 2R'^2 \left(\lambda + \sum_{t=1}^{T-\Texp} q_{t,i}^2 \right) \log \left(\frac{2n}{\delta \sqrt{\lambda}} \left(\lambda + \sum_{t=1}^{T-\Texp} q_{t,i}^2 \right)^{1/2} \right)
\end{equation*}
Recalling Assumption~\ref{Stabilizable set}, on $\mathcal{E} \cap \mathcal{F} \cap \mathcal{G} $, 

\begin{align*}
\| q_t\| &\leq \|Q\|\|C\| \left( \Phi(A) \|\hat{x}_{t|t,\Theta} -\hat{x}_{t|t,\tth} \|  + \Phi(A) \| \hat{x}_{t|t,\tth} \| + \Gamma \|B\|  \|\hat{x}_{t|t,\tth} \| \right) \\
&\leq \mathcal{Q}_T' \! \coloneqq \! \|Q\|\|C\|\left( \left(\!\| \tl\!-\!L\|\|C\|\!+\!\|L \| \|\tc\!-\! C\|\!+\!\|\tl\!-\!L\| \|\tc\!-\!C\|\!\right) \Phi(A) \rho \tilde{\mathcal{X}}\!+\!\|\tl\!-\!L \| \mathcal{Y} \!+\! (\Phi(A)\!+\! \Gamma \|B\|) \tilde{\mathcal{X}}  \right)
\end{align*} 
Thus, $q_{t,i} \leq \mathcal{Q}_T'$. Combining these results with union bound, with probability $1-\delta/2$ we have 
\begin{align}
     G_1 &\leq 2R'' n  \sqrt{ (\lambda + \mathcal{Q}_T'^2 (T-\Texp)) \log \left( \frac{n\sqrt{\lambda + (T-\Texp)\mathcal{Q}_T'^2}}{\delta\sqrt{\lambda}}\right) } = \tilde{\OO}\left(\sqrt{T -\Texp }\right) \nonumber 
\end{align} 
where $\tilde{\OO}$ hides logarithmic terms. 

From the fact that $v_t$ is $R''$-sub-Gaussian, with probability $1-\delta/4$ we have 
\begin{align}
    \|v_t\| &\leq \Omega' \coloneqq (\|C\| \Phi(A) \| \sig \|^{1/2} + \sigma_w \|C \|  + \sigma_z) \sqrt{2n\log\left(\frac{8n(T-\Texp)}{\delta}\right)}\label{R2eq}
\end{align}
On the event of $\mathcal{E} \cap \mathcal{F} \cap \mathcal{G} $, let $V_t = v_t^\top Q v_t - \mathbb{E}\left[ v_t^\top Q v_t \Big | \hat{x}_{t-1|t-2}, y_{t-1}, u_{t-1} \right]$ and define its truncated version $\bar{V}_t = V_t \mathbbm{1}_{V_t \leq  2\|Q\| \Omega'} $. Thus $G_2 = \sum_{t=1}^{T-\Texp} V_t$ and define $\bar{G}_2 =\sum_{t=1}^{T-\Texp} \bar{V}_t $. By Lemma~\ref{basicprob}, 

\begin{align}
    \Pr\left( G_2 > 2\|Q\|\Omega' \sqrt{2(T-\Texp) \log\frac{4}{\delta}} \right) &\leq\!\Pr\!\left(\max_{1\leq t\leq T - \Texp}\!\!\!\!\!V_t\!\geq\! 2\|Q\|\Omega' \right)\!\\
    &\qquad \qquad+\!\Pr\left( \bar{G}_3\!>\!2\|Q\|\Omega' \sqrt{2(T\!-\!\Texp) \log\frac{4}{\delta}} \right). \nonumber 
\end{align}
From \eqref{R1eq} and Theorem~\ref{azuma}, each term on the right hand side is bounded by $\delta/4$. Thus, with probability $1-\delta/2$,

\begin{equation*}
    G_2 \leq 2\|Q\| \Omega' \sqrt{2(T-\Texp) \log\frac{4}{\delta}} = \tilde{\OO}\left(\sqrt{T-\Texp} \right).
\end{equation*}

Since $R_2 = G_1 + G_2 $ the statement of lemma is obtained. 
\end{proof}
\subsection[Bound on R3]{Bounding $|R_3|$ on the event of $\mathcal{E} \cap \mathcal{F} \cap \mathcal{G} $}
\begin{lemma}
Suppose Assumption~\ref{Stabilizable set} holds and system is explored for  $\Texp > T_0$ time steps. Given $\mathcal{E} \cap \mathcal{F} \cap \mathcal{G} $ holds, 
\begin{align*}
   \left| R_3 \right| = \tilde{\OO}\left( \frac{T-\Texp}{\sqrt{\Texp}}\right).
\end{align*}
\end{lemma}
\begin{proof}
Let $\tilde{\Upsilon} = [\tc \ta, \enskip -\tc \tb \tk]$ and $\Upsilon = [CA, \enskip -CB\tk]$. 
Then, on the event of $\mathcal{E} \cap \mathcal{F} \cap \mathcal{G}$,
\begin{align}
    &\left\|\left(\tilde{\Upsilon} - \Upsilon \right) \hat{x}_{t|t,\tth} \right\| \leq \|(\tilde{\Upsilon} - \Upsilon)\| \|\hat{x}_{t|t,\tth} \|  \nonumber \\
    &\leq \left(\| \tc \ta -CA\| + \| \tk\| \|CB -\tc\tb \| \right) \|\hat{x}_{t|t,\tth} \| \nonumber \\
    &\leq \big( \Phi(A) \|\tc-C\| + \|\ta -A \|\|\tc - C \| + \|\ta -A \| \|C\| \\
    &\qquad +\Gamma (\|\tc-C\| \|B \| + \|\tb -B \| \| \tc - C\| + \|\tb -B \| \|C\| ) \big) \tilde{\mathcal{X}} \nonumber \\
    &\leq \bigg( \Phi(A) \Delta C + \Delta A \Delta C + \Delta A \|C\| + \Gamma \big(\Delta C \|B\| + \Delta B \Delta C + \Delta B \|C\| \big) \bigg) \tilde{\mathcal{X}} = \tilde{\OO}\left(\frac{1}{\sqrt{\Texp}}\right) \label{lastdecomposeR3}
\end{align}
This fact helps us prove the statement of lemma as follows, 
\begin{align}
    &\mathbbm{1}_{\mathcal{E} \cap \mathcal{F} \cap \mathcal{G}} |R_3| \nonumber \\
    &\leq  \mathbbm{1}_{\mathcal{E} \cap \mathcal{F} \cap \mathcal{G}} \sum_{t=0}^{T-\Texp} \left| \left\|Q^{1/2} \tilde{\Upsilon} \hat{x}_{t|t,\tth} \right\|^2 - \left\|Q^{1/2} \Upsilon \hat{x}_{t|t,\tth} \right\|^2 \right| \label{R3_1}\\
    &\leq \mathbbm{1}_{\mathcal{E} \cap \mathcal{F} \cap \mathcal{G}}\!\! \left( \sum_{t=0}^{T-\Texp}\!\!\!\!\!\left( \left\|Q^{1/2} \tilde{\Upsilon} \hat{x}_{t|t,\tth} \right\| \!-\! \left\|Q^{1/2} \Upsilon \hat{x}_{t|t,\tth} \right\| \right)^2 \right)^{1/2}\!\!\!\! \left(\! \sum_{t=0}^{T-\Texp}\!\!\!\!\! \left( \left\|Q^{1/2} \tilde{\Upsilon} \hat{x}_{t|t,\tth} \right\| 
    \!+\! \left\|Q^{1/2} \Upsilon \hat{x}_{t|t,\tth} \right\| \right)^2 \!\!\right)^{1/2} \label{R3_2} \\
    &\leq \mathbbm{1}_{\mathcal{E} \cap \mathcal{F} \cap \mathcal{G}} \left( \sum_{t=0}^{T-\Texp}  \left\|Q^{1/2} \left(\tilde{\Upsilon} - \Upsilon \right) \hat{x}_{t|t,\tth} \right\|^2 \right)^{1/2} \left( \sum_{t=0}^{T-\Texp} \left( \left\|Q^{1/2} \tilde{\Upsilon} \hat{x}_{t|t,\tth} \right\| 
    + \left\|Q^{1/2} \Upsilon \hat{x}_{t|t,\tth} \right\| \right)^2 \right)^{1/2} \label{R3_3} \\
    &\leq  (T-\Texp) \|Q\| \tilde{\mathcal{X}} \Bigg( \left\|\left(\tilde{\Upsilon} - \Upsilon\right) \hat{x}_{t|t,\tth} \right\|\left(\|\tilde{\Upsilon} \| + \|\Upsilon \| \right) \Bigg)  \\
    &= \tilde{\OO}\left( \frac{T-\Texp}{\sqrt{\Texp}}\right) \label{R3result}
\end{align}
where \eqref{R3_1} follows from triangle inequality, \eqref{R3_2} is due to Cauchy Schwarz and \eqref{R3_3} is again triangle inequality. Finally, in \eqref{R3result}, we use \eqref{lastdecomposeR3} and the boundedness of $G$ and $\tilde{G}$, which translates to boundedness of $\Upsilon$ and $\tilde{\Upsilon}$.

\end{proof}

\subsection[Bound on R4]{Bounding $|R_4|$ on the event of $\mathcal{E} \cap \mathcal{F} \cap \mathcal{G} $}
\begin{lemma}
Suppose Assumption~\ref{Stabilizable set} holds and system is explored for  $\Texp > T_0$ time steps. Given $\mathcal{E} \cap \mathcal{F} \cap \mathcal{G} $ holds,
\begin{align*}
  \left |R_4 \right| =  \tilde{\OO} \left( \frac{T - \Texp}{\sqrt{\Texp}} \right).
\end{align*}
\end{lemma}

\begin{proof}
Similar to $R_3$, let $\tilde{\Upsilon} = [\ta, \enskip -\tb \tk]$ and $\Upsilon = [A, \enskip -B\tk]$. 
Then, on the event of $\mathcal{E} \cap \mathcal{F} \cap \mathcal{G}$,
\begin{align}
    \left\|\left(\tilde{\Upsilon} - \Upsilon \right) \hat{x}_{t|t,\tth} \right\| &\leq \|(\tilde{\Upsilon} - \Upsilon)\| \|\hat{x}_{t|t,\tth} \|  \nonumber \\
    &\leq \left(\| \ta -A\| + \| \tk\| \|B -\tb \| \right) \|\hat{x}_{t|t,\tth} \| \nonumber \\
    &\leq \left( \|\ta -A \| + \Gamma  \|\tb -B \|  \right) \tilde{\mathcal{X}} \nonumber \\
    &\leq \bigg( \Delta A + \Gamma \Delta B \bigg) \tilde{\mathcal{X}} = \tilde{\OO}\left(\frac{1}{\sqrt{\Texp}}\right) \label{lastdecomposeR4}
\end{align}
This fact helps us prove the statement of lemma as follows, 

\begin{align}
    &\mathbbm{1}_{\mathcal{E} \cap \mathcal{F} \cap \mathcal{G} } |R_4| \nonumber \\
    &\leq  \mathbbm{1}_{\mathcal{E} \cap \mathcal{F} \cap \mathcal{G} } \sum_{t=0}^{T-\Texp} \left| \left\|(\tp - \tc^\top Q \tc)^{1/2} \tilde{\Upsilon} \hat{x}_{t|t,\tth} \right\|^2 - \left\|(\tp - \tc^\top Q \tc)^{1/2} \Upsilon \hat{x}_{t|t,\tth} \right\|^2 \right| \label{R4_1} \\
    &\leq \mathbbm{1}_{\mathcal{E} \cap \mathcal{F} \cap \mathcal{G} } \left( \sum_{t=0}^{T-\Texp} \left( \left\|(\tp - \tc^\top Q \tc)^{1/2} \tilde{\Upsilon} \hat{x}_{t|t,\tth} \right\| - \left\|(\tp - \tc^\top Q \tc)^{1/2} \Upsilon \hat{x}_{t|t,\tth} \right\| \right)^2 \right)^{1/2} \label{R4_2} \\
    &\qquad \qquad \times \left( \sum_{t=0}^{T-\Texp} \left( \left\|(\tp - \tc^\top Q \tc)^{1/2} \tilde{\Upsilon} \hat{x}_{t|t,\tth} \right\| 
    + \left\|(\tp - \tc^\top Q \tc)^{1/2} \Upsilon \hat{x}_{t|t,\tth} \right\| \right)^2 \right)^{1/2} \nonumber \\
    &\leq \mathbbm{1}_{\mathcal{E} \cap \mathcal{F} \cap \mathcal{G} } \left( \sum_{t=0}^{T-\Texp}  \left\|(\tp - \tc^\top Q \tc)^{1/2} \left(\tilde{\Upsilon} - \Upsilon \right)\hat{x}_{t|t,\tth} \right\|^2 \right)^{1/2} \label{R4_3}\\
    &\qquad \qquad \times \left( \sum_{t=0}^{T-\Texp} \left( \left\|(\tp - \tc^\top Q \tc)^{1/2} \tilde{\Upsilon} \hat{x}_{t|t,\tth} \right\| 
    + \left\|(\tp - \tc^\top Q \tc)^{1/2} \Upsilon \hat{x}_{t|t,\tth} \right\| \right)^2 \right)^{1/2} \nonumber \\
    &\leq (T-\Texp) \|\tp - \tc^\top Q \tc\| \tilde{\mathcal{X}} \Bigg( \left\|\left(\tilde{\Upsilon} - \Upsilon\right) \hat{x}_{t|t,\tth} \right\|\left(\|\tilde{\Upsilon} \| + \|\Upsilon \| \right) \Bigg)   \\
    &\leq (T-\Texp) \left(D + \| Q\| \left( \| C\| + \Delta C \right)^2  \right) \tilde{\mathcal{X}} \Bigg( \left\|\left(\tilde{\Upsilon} - \Upsilon\right) \hat{x}_{t|t,\tth} \right\|\left(\|\tilde{\Upsilon} \| + \|\Upsilon \| \right) \Bigg)\\
    &= \tilde{\OO}\left( \frac{T-\Texp}{\sqrt{\Texp}}\right)  \label{R4result}
\end{align}
where \eqref{R4_1} follows from triangle inequality, \eqref{R4_2} is due to Cauchy Schwarz and \eqref{R4_3} is again triangle inequality. Finally, in \eqref{R4result}, we use \eqref{lastdecomposeR4} and the boundedness of $G$ and $\tilde{G}$, which translates to boundedness of $\Upsilon$ and $\tilde{\Upsilon}$.
\end{proof}

\subsection[Bound on R5]{Bounding $|R_5|$ on the event of $\mathcal{E} \cap \mathcal{F} \cap \mathcal{G} $}
\begin{lemma}
Suppose Assumption~\ref{Stabilizable set} holds and system is explored for  $\Texp > T_0$ time steps. Given $\mathcal{E} \cap \mathcal{F} \cap \mathcal{G} $ holds, 
\begin{align*}
  \Bigg |  & \sum_{t=0}^{T-\Texp} \left \{ 2\hat{x}_{t|t,\tth}^\top (A-B\tk)^\top (\tp - \tc^\top Q \tc) (I-LC)A(\hat{x}_{t|t,\Theta} -\hat{x}_{t|t,\tth} ) \right \} \Bigg | = \tilde{\OO}\left( \frac{T-\Texp}{\sqrt{\Texp}}\right).
\end{align*}
\end{lemma}

\begin{proof}
\begin{align*}
    &\mathbbm{1}_{\mathcal{E} \cap \mathcal{F} \cap \mathcal{G} }|R_5| \\
    &= 2\left| \sum_{t=0}^{T-\Texp} \left \{ \hat{x}_{t|t,\tth}^\top (A-B\tk)^\top (\tp - \tc^\top Q \tc) (I-LC)A(\tl\tc\hat{x}_{t|t-1,\tth} -LC\hat{x}_{t|t-1,\tth} + L y_t -\tl y_t ) \right \}  \right| \\
    &\leq 2\left| \sum_{t=0}^{T-\Texp} \left \{ \hat{x}_{t|t,\tth}^\top (A-B\tk)^\top (\tp - \tc^\top Q \tc) (I-LC)A(\tl\tc-LC)\hat{x}_{t|t-1,\tth} \right\} \right |
    \\&\qquad \qquad + 2\left| \sum_{t=0}^{T-\Texp} \left \{ \hat{x}_{t|t,\tth}^\top (A-B\tk)^\top (\tp - \tc^\top Q \tc) (I-LC)A (L-\tl) y_t  \right \}  \right|\\
    &\leq\!2(T-\Texp)\!\Bigg( \tilde{\mathcal{X}}\left(\Phi(A)+ \Gamma \|B\|\right) \|\tp\!-\!\tc^\top Q \tc \| \|A - LCA\| \times \\
    &\qquad \qquad \qquad \qquad \qquad \left(\| \tl\!-\!L \| \left(\|y_t \|\!+\!(\|C \|\!+\!\|\tc\!-\!C\|) \tilde{\mathcal{X}} \right)\!+\! \zeta \| \tc\!-\!C\| \tilde{\mathcal{X}} \right)  \Bigg) \\ 
    &\leq 2(T-\Texp) \Big(\tilde{\mathcal{X}}\left(\Phi(A)+ \Gamma \|B\|\right)\left(D + \| Q\| \left( \| C\| + \Delta C \right)^2  \right) \left( \Phi(A) (1 + \zeta \| C\|)\right) \times \\
    &\qquad \qquad \qquad \qquad \qquad \left(\Delta L \left(\mathcal{Y} +  (\|C\|+\Delta C) \tilde{\mathcal{X}} \right) + \Delta C \zeta  \tilde{\mathcal{X}} \right)  \Big)
\end{align*}
where we used the bounds listed in the beginning of the section, \textit{i.e.} concentration results and the boundedness property of $\hat{x}_{t|t,\tth}$ and $y_t$. Using the fact that all the concentration results are $\tilde{\OO}\left(\frac{1}{\sqrt{\Texp}} \right)$ we obtain the the statement of lemma. 
\end{proof}

\subsection[Bound on R6]{Bounding $|R_6|$ on the event of $\mathcal{E} \cap \mathcal{F} \cap \mathcal{G} $}
\begin{lemma}
Suppose Assumption~\ref{Stabilizable set} holds and system is explored for  $\Texp > T_0$ time steps. Given $\mathcal{E} \cap \mathcal{F} \cap \mathcal{G} $ holds, 
\begin{align*}
  \left|  R_6  \right| = \tilde{\OO}\left( \frac{T-\Texp}{\Texp}\right).
\end{align*}
\end{lemma}
\begin{proof}
\begin{align*}
    &\mathbbm{1}_{\mathcal{E} \cap \mathcal{F} \cap \mathcal{G} }|R_6| \\
    &= \left| \sum_{t=0}^{T-\Texp} \left \{ \left(\hat{x}_{t|t,\Theta} -\hat{x}_{t|t,\tth} \right)^\top A^\top (I-LC)^\top (\tp - \tc^\top Q \tc) (I-LC)A(\hat{x}_{t|t,\Theta} -\hat{x}_{t|t,\tth} ) \right \}  \right| \\
    &\leq (T-\Texp) \left\|(\tl\tc -LC)\hat{x}_{t|t-1,\tth} + (L-\tl)y_t \right\|^2 \|\tp - \tc^\top Q \tc \| \left( \Phi(A) (1 + \zeta \| C\|)\right)^2 \\
    &\leq (T-\Texp)\!\!\left(\!\left(\!\| \tl\!-\!L\|\|C\|\!+\!\|L \| \|\tc\!-\! C\|\!+\!\|\tl\!-\!L\| \|\tc\!-\!C\|\!\right) \rho \tilde{\mathcal{X}}\!+\!\|\tl\!-\!L \| \mathcal{Y} \right)^2 \times \\
    &\qquad \qquad \left(\!D\!+\!\| Q\| \left( \| C\|\!+\!\Delta C \right)^2  \right) \left( \Phi(A) (1 + \zeta \| C\|)\right)^2 \\
    &\leq (T - \Texp) \left( \left(\Delta L \|C\| + \zeta \Delta C\!+\!\Delta L \Delta C \right) \rho \tilde{\mathcal{X}}\!+\!\Delta L \mathcal{Y} \right)^2 \left(D\!+\!\| Q\| \left( \| C\|\!+\!\Delta C \right)^2  \right) \left( \Phi(A) (1 + \zeta \| C\|)\right)^2 \\
    &= \tilde{\OO}\left( \frac{T-\Texp}{\Texp}\right)
\end{align*}
where we used the bounds listed in the beginning of the section, \textit{i.e.} concentration results and the boundedness property of $\hat{x}_{t|t,\tth}$ and $y_t$. Using the fact that all the concentration results are $\tilde{\OO}\left(\frac{1}{\sqrt{\Texp}} \right)$ and each concentration term is squared in the final expression, we get the statement of lemma. 
\end{proof}

\subsection[Bound on R7]{Bounding $|R_7|$ on the event of $\mathcal{E} \cap \mathcal{F} \cap \mathcal{G} $}

\begin{lemma}
Given $\mathcal{E} \cap \mathcal{F} \cap \mathcal{G} $ holds 
\begin{align*}
  \Bigg |  &\sum_{t=0}^{T-\Texp} \left\{ \mathbb{E} \left[ w_t \tc^\top Q \tc w_t   \right] - \mathbb{E}\left[w_t^\top C^\top Q C w_t \right] \right\}  \Bigg | = \tilde{\OO}\left( \frac{T-\Texp}{\sqrt{\Texp}}\right).
\end{align*}
\end{lemma}
\begin{proof}
\begin{align*}
    \mathbbm{1}_{\mathcal{E} \cap \mathcal{F}}|R_7| &= \mathbbm{1}_{\mathcal{E} \cap \mathcal{F} \cap \mathcal{G} } \sum_{t=0}^{T-\Texp} \left | \mathbb{E}  \bigg[w_t^\top \left( \tc^\top Q \tc -  C^\top Q C  \right) w_t \bigg] \right | \\
    &= \mathbbm{1}_{\mathcal{E} \cap \mathcal{F} \cap \mathcal{G} } \sum_{t=0}^{T-\Texp} \left | \Tr \left( \left(\tc^\top Q \tc -  C^\top Q C \right)   \mathbb{E}\bigg[w_t w_t^\top \bigg]\right) \right | \\
    &\leq n \sigma_w^2 \mathbbm{1}_{\mathcal{E} \cap \mathcal{F} \cap \mathcal{G} } \sum_{t=0}^{T-\Texp}  \left\| \tc^\top Q \tc -  C^\top Q C \right\| \\
    &\leq  n (T-\Texp) \sigma_w^2 \|Q\|  \left( \Delta C^2 + 2 \| C \| \Delta C \right) = \tilde{\OO}\left( \frac{T-\Texp}{\sqrt{\Texp}}\right).
\end{align*}
\end{proof}
\subsection[Bound on R8]{Bounding $|R_8|$ on the event of $\mathcal{E} \cap \mathcal{F} \cap \mathcal{G} $}
\begin{lemma}
Suppose Assumption~\ref{Stabilizable set} holds and system is explored for  $\Texp > T_0$ time steps. Given $\mathcal{E} \cap \mathcal{F} \cap \mathcal{G} $ holds, 
\begin{align*}
  \Bigg |  \sum_{t=0}^{T-\Texp} \left \{ \mathbb{E}  \left[ w_t^\top \tc^\top \tl^\top  \left( \tp - \tc^\top Q \tc  \right) \tl \tc w_t   \right]   - \mathbb{E}\left[w_t^\top C^\top L^\top \left( \tp - \tc^\top Q \tc  \right) L C  w_t  \right] \right \}   \Bigg | = \tilde{\OO}\left( \frac{T-\Texp}{\sqrt{\Texp}}\right)
\end{align*}
\end{lemma}
\begin{proof}
\begin{align*}
    &\mathbbm{1}_{\mathcal{E} \cap \mathcal{F} \cap \mathcal{G} }|R_8| \\
    &= \mathbbm{1}_{\mathcal{E} \cap \mathcal{F} \cap \mathcal{G} } \sum_{t=0}^{T-\Texp} \left | \mathbb{E}  \bigg[w_t^\top \left( \tc^\top \tl^\top  \left( \tp - \tc^\top Q \tc  \right) \tl \tc -  C^\top L^\top \left( \tp - \tc^\top Q \tc  \right) L C  \right) w_t \bigg] \right | \\
    &= \mathbbm{1}_{\mathcal{E} \cap \mathcal{F} \cap \mathcal{G} } \sum_{t=0}^{T-\Texp} \left | \Tr \left( \left( \tc^\top \tl^\top  \left( \tp - \tc^\top Q \tc  \right) \tl \tc -  C^\top L^\top \left( \tp - \tc^\top Q \tc  \right) L C  \right) \mathbb{E}\bigg[w_t w_t^\top \bigg]\right) \right | \\
    &\leq n \sigma_w^2 \mathbbm{1}_{\mathcal{E} \cap \mathcal{F} \cap \mathcal{G} } \sum_{t=0}^{T-\Texp}  \left\| \tc^\top \tl^\top  \left( \tp - \tc^\top Q \tc  \right) \tl \tc -  C^\top L^\top \left( \tp - \tc^\top Q \tc  \right) L C \right\| \\
    &\leq n (T-\Texp) \sigma_w^2 \mathbbm{1}_{\mathcal{E} \cap \mathcal{F} \cap \mathcal{G} } \|\tp - \tc^\top Q \tc \| \left( \|\tl\tc - LC \|^2 + 2 \|LC\| \|\tl\tc - LC \|  \right) \\
    &\leq n (T-\Texp) \sigma_w^2 \left(D + \| Q\| \left( \| C\| + \Delta C \right)^2  \right)\mathbbm{1}_{\mathcal{E} \cap \mathcal{F} \cap \mathcal{G} } \Bigg( \left(\!\| \tl\!-\!L\|\|C\|\!+\!\|L \| \|\tc\!-\! C\|\!+\!\|\tl\!-\!L\| \|\tc\!-\!C\|\!\right)^2 \\
    &\qquad + 2 \|LC\| \left(\!\| \tl\!-\!L\|\|C\|\!+\!\|L \| \|\tc\!-\! C\|\!+\!\|\tl\!-\!L\| \|\tc\!-\!C\|\!\right)  \Bigg) \\
    &\leq\!n (T\!-\!\Texp) \sigma_w^2 \!\left(D\!+\!\| Q\| \left( \| C\|\! +\!\Delta C \right)^2\right)\!\! \Big(\Delta L \|C\|\!+\! \zeta \Delta C\!+\!\Delta L \Delta C \Big)  \Big(2\zeta\|C\|\!+\!\Delta L \|C\|\!+\!\zeta \Delta C\!+\!\Delta L \Delta C \Big) \\
    &= \tilde{\OO}\left( \frac{T-\Texp}{\sqrt{\Texp}}\right)
\end{align*}
Since the dominating term has only one concentration result. 
\end{proof}

\subsection[Bounds on R9 R10]{Bounding $|R_9|$ and $|R_{10}|$ on the event of $\mathcal{E} \cap \mathcal{F} \cap \mathcal{G} $}

\begin{lemma} \label{lemmaR10}
Let $R_9$ and $R_{10}$ be as defined in equations \eqref{R9} and \eqref{R10}, respectively. Suppose Assumption~\ref{Stabilizable set} holds and system is explored for  $\Texp > T_0$ time steps. Given $\mathcal{E} \cap \mathcal{F} \cap \mathcal{G} $ holds, 
\begin{align*}
  \left | R_9 \right | = \tilde{\OO}\left( \frac{T-\Texp}{\sqrt{\Texp}}\right), \qquad  \left | R_{10} \right | = \tilde{\OO}\left( \frac{T-\Texp}{\sqrt{\Texp}}\right)
\end{align*}
\end{lemma}
\begin{proof}
Observe that $x_{t+1} - \hat{x}_{t+1|t, \tth}$ has the following dynamics, 
\[
x_{t+1} - \hat{x}_{t+1|t, \tth} = \ta(I-\tl\tc)(x_{t} - \hat{x}_{t|t-1, \tth}) + w_t - \tl z_t
\]
Thus, the estimation error propagates according to a linear system, with closed-loop dynamics $\ta-\ta\tl\tc$, driven by the process $w_t - LC z_t$, which is iid zero mean and covariance $W + LZL^\top$. Additionally, from Assumption~\ref{Stabilizable set}, $\ta - \ta\tl\tc$ is stable. Recall that $\tsig $ is the steady state error covariance matrix of state estimation of the system with optimistic parameters $\tth$. Thus, we have that $\left\|\tsig  \right\| \geq \left\|\mathbb{E}\bigg[ \left(x_t - \hat{x}_{t|t,\tth}\right)\left(x_t - \hat{x}_{t|t,\tth}\right)^\top  \Big | \hat{x}_{t|t-1}, y_t   \bigg] \right\| $ for all $t\geq 0$. From this we get 
\begin{align*}
    &\mathbbm{1}_{\mathcal{E} \cap \mathcal{F} \cap \mathcal{G} }|R_9| \\
    &= \mathbbm{1}_{\mathcal{E} \cap \mathcal{F} \cap \mathcal{G} } \sum_{t=0}^{T-\Texp} \left | \mathbb{E}  \bigg[ \left(x_t - \hat{x}_{t|t,\tth}\right)^\top \left( \ta^\top \tc^\top Q \tc \ta - A^\top C^\top Q C A \right)\left(x_t - \hat{x}_{t|t,\tth}\right) \Big | \hat{x}_{t|t-1}, y_t   \bigg] \right | \\
    &= \mathbbm{1}_{\mathcal{E} \cap \mathcal{F} \cap \mathcal{G} } \sum_{t=0}^{T-\Texp} \left | \Tr \left(  \left( \ta^\top \tc^\top Q \tc \ta - A^\top C^\top Q C A \right) \mathbb{E}\bigg[\left(x_t - \hat{x}_{t|t,\tth}\right)\left(x_t - \hat{x}_{t|t,\tth}\right)^\top \Big | \hat{x}_{t|t-1}, y_t \bigg] \right) \right | \\
    &\leq n \mathbbm{1}_{\mathcal{E} \cap \mathcal{F} \cap \mathcal{G} } \sum_{t=0}^{T-\Texp}  \left\| \ta^\top \tc^\top Q \tc \ta - A^\top C^\top Q C A \right\| \left\|\tsig  \right\| \\
    &\leq  n (T-\Texp)\|Q\|  \mathbbm{1}_{\mathcal{E} \cap \mathcal{F} \cap \mathcal{G} }  \left( \|\tc\ta - CA \|^2 + 2 \| C \| \| \tc\ta -CA \|\right)(\|\sig\| + \Delta \Sigma) \\
    &\leq  n (T-\Texp) \|Q\|(\|\sig\| + \Delta \Sigma) \bigg(\!\!\Big(\Phi(A)\Delta C\!+\! \Delta A \Delta C\!+\! \|C\|\Delta A \Big)  \Big(2\|C\|\!+\!\Phi(A)\Delta C\!+\! \Delta A \Delta C\!+\! \|C\|\Delta A \Big)\!\! \bigg) \\
    &= \tilde{\OO}\left( \frac{T-\Texp}{\sqrt{\Texp}}\right)
\end{align*}
Since the dominating term has only one concentration result. Similar result for $R_{10}$ is obtained as follows

\begin{align*}
    &\mathbbm{1}_{\mathcal{E} \cap \mathcal{F} \cap \mathcal{G} }|R_{10}|\\
    &= \mathbbm{1}_{\mathcal{E} \cap \mathcal{F} \cap \mathcal{G} } \sum_{t=0}^{T-\Texp} \Bigg | \Tr \bigg(  \left( \ta^\top \tc^\top \tl^\top  \left( \tp - \tc^\top Q \tc  \right) \tl\tc\ta - A^\top C^\top L^\top \left( \tp - \tc^\top Q \tc  \right) L C A \right) \times \\
    &\qquad \qquad \qquad \qquad \qquad \qquad \qquad \qquad \qquad \qquad  \mathbb{E}\bigg[\left(x_t - \hat{x}_{t|t,\tth}\right)\left(x_t - \hat{x}_{t|t,\tth}\right)^\top \Big | \hat{x}_{t|t-1}, y_t \bigg] \bigg) \Bigg | \\
    &\leq n \mathbbm{1}_{\mathcal{E} \cap \mathcal{F} \cap \mathcal{G} } \sum_{t=0}^{T-\Texp}  \left\| \ta^\top \tc^\top \tl^\top  \left( \tp - \tc^\top Q \tc  \right) \tl\tc\ta - A^\top C^\top L^\top \left( \tp - \tc^\top Q \tc  \right) L C A \right\| \left\|\tsig  \right\| \\
    &\leq  n (T-\Texp)  \mathbbm{1}_{\mathcal{E} \cap \mathcal{F} \cap \mathcal{G} }  \|\tp - \tc^\top Q \tc \| \left( \|\tl\tc\ta - LCA \|^2 + 2 \|LCA\| \|\tl\tc\ta - LCA \|  \right)(\|\sig\| + \Delta \Sigma) \\
    &\leq  n (T-\Texp) \left(D\!+\!\| Q\| \left( \| C\|\! +\!\Delta C \right)^2\right)(\|\sig\| + \Delta \Sigma) \left(\Delta K'  (2\zeta\|C\| + \Delta K' ) \right) 
\end{align*}
where $\Delta K' \coloneqq \Big(\Delta L\Phi(A)\|C\| + \zeta \Phi(A) \Delta C + \zeta \|C\| \Delta A + \zeta \Delta C \Delta A + \Phi(A) \Delta L \Delta C + \Delta L \|C\| \Delta A + \Delta L \Delta C \Delta A \Big) $. Similar to $R_9$, the dominating term has only one concentration result. Thus, $\mathbbm{1}_{\mathcal{E} \cap \mathcal{F} \cap \mathcal{G} }|R_{10}| = \tilde{\OO}\left( \frac{T-\Texp}{\sqrt{\Texp}}\right)$

\end{proof}

\subsection[Bound on R11]{Bounding $|R_{11}|$ on the event of $\mathcal{E} \cap \mathcal{F} \cap \mathcal{G} $}

\begin{lemma}\label{lemmaR11}
Suppose Assumption~\ref{Stabilizable set} holds and system is explored for $\Texp > T_0$ time steps. Given $\mathcal{E} \cap \mathcal{F} \cap \mathcal{G} $ holds, 
\begin{align*}
  \left |  R_{11}  \right | = \tilde{\OO}\left( \frac{T-\Texp}{\sqrt{\Texp}}\right).
\end{align*}
\end{lemma}
\begin{proof}
\begin{align*}
    &\mathbbm{1}_{\mathcal{E} \cap \mathcal{F} \cap \mathcal{G} }\Bigg |  \sum_{t=0}^{T-\Texp} 2 \mathbb{E}\left[z_{t+1}^\top L^\top\!\! \left( \tp \!-\! \tc^\top Q \tc  \right) (\tl \!-\! L) z_{t+1} \right]\Bigg | \\   
    &= \mathbbm{1}_{\mathcal{E} \cap \mathcal{F} \cap \mathcal{G} }  \sum_{t=0}^{T-\Texp} 2 \left | \mathbb{E}  \left[z_{t+1}^\top L^\top\!\! \left( \tp \!-\! \tc^\top Q \tc  \right) (\tl \!-\! L) z_{t+1} \right] \right | \\
    &= \mathbbm{1}_{\mathcal{E} \cap \mathcal{F} \cap \mathcal{G} } \sum_{t=0}^{T-\Texp} 2 \left | \Tr \left( \left(L^\top\!\! \left( \tp \!-\! \tc^\top Q \tc  \right) (\tl \!-\! L)\right) \mathbb{E}\bigg[z_{t+1} z_{t+1}^\top \bigg]\right) \right | \\
    &\leq 2m \sigma_z^2 \mathbbm{1}_{\mathcal{E} \cap \mathcal{F} \cap \mathcal{G} } \sum_{t=0}^{T-\Texp}  \left\| L^\top\!\! \left( \tp \!-\! \tc^\top Q \tc  \right) (\tl \!-\! L) \right\| \\
    &\leq 2m \sigma_z^2 (T\!-\!\Texp) \!\left(D\!+\!\| Q\| \left( \| C\|\! +\!\Delta C \right)^2\right) \zeta \Delta L = \tilde{\OO}\left( \frac{T-\Texp}{\sqrt{\Texp}}\right)
\end{align*}
Since the second term in $R_{11}$ has two concentration terms, \textit{i.e.} $\|\tl - L\|$, the one given above is the dominating term which proves the lemma. 

\end{proof}

Combining these individual bounds with the result from Section~\ref{SuppRegretExplore}, we obtain the total regret upper bound of \Alg: 

\begin{theorem}[Regret Upper Bound of \Alg]\label{Sup:totalregret}
Given a \LQG $\Theta = (A,B,C)$, and regulating parameters $Q$ and $R$, suppose Assumptions~\ref{Stable}-\ref{Stabilizable set} hold. For any exploration duration $\Texp> T_0$, 
if \Alg interacts with the system $\Theta$ for $T$ steps in total such that $T>\Texp$, with probability at least $1- 10\delta$, the regret of \Alg is bounded as follows,

\begin{equation}
    \reg(T) = \tilde{\OO}\left( \Texp + \frac{T-\Texp}{\sqrt{\Texp}} + T^{2/3} \right).
\end{equation}
\end{theorem}
\begin{proof}
Combining Lemmas~\ref{lemmaR1}-\ref{lemmaR11}, on the event of $\mathcal{E} \cap \mathcal{F} \cap \mathcal{G} $, with probability at least $1-\delta$, we have 

\begin{align*}
    R_1\!&=\!\tilde{\OO}\left(\! \frac{T\!-\!\Texp}{\sqrt{\Texp}}\!\right), R_2 =  \tilde{\OO}\left( \sqrt{T\!-\!\Texp} \right), |R_3|\!=\!\tilde{\OO}\left(\! \frac{T\!-\!\Texp}{\sqrt{\Texp}}\!\right), |R_4|\!=\!\tilde{\OO}\left(\! \frac{T\!-\!\Texp}{\sqrt{\Texp}}\!\right), \\
    |R_5| &= \tilde{\OO}\left( \frac{T\!-\!\Texp}{\sqrt{\Texp}} \right), |R_6|\!=\!\tilde{\OO}\left(\! \frac{T\!-\!\Texp}{\sqrt{\Texp}}\!\right), |R_7|\!=\!\tilde{\OO}\left(\! \frac{T\!-\!\Texp}{\sqrt{\Texp}}\!\right),
    |R_8|\!=\!\tilde{\OO}\left(\! \frac{T\!-\!\Texp}{\sqrt{\Texp}}\!\right), \\
    &\qquad \qquad \qquad |R_9|\!=\!\tilde{\OO}\left(\! \frac{T\!-\!\Texp}{\sqrt{\Texp}}\!\right),
    |R_{10}|\!=\!\tilde{\OO}\left(\! \frac{T\!-\!\Texp}{\sqrt{\Texp}}\!\right), |R_{11}|\!=\!\tilde{\OO}\left(\! \frac{T\!-\!\Texp}{\sqrt{\Texp}}\!\right).
\end{align*}

Thus, on the event $\mathcal{E} \cap \mathcal{F} \cap \mathcal{G} $, with probability at least $1-\delta$,
\begin{align*}
\reg(T-\Texp)\! &\leq \!R_1 \!+\! R_2 \!+\! |R_3| \!+\! |R_4| \!+\! |R_5| \!+\! |R_6| + |R_7| \!+\! |R_8| \!+\! |R_9| \!+\! |R_{10}| \!+\! |R_{11}|  \!+\! T^{2/3} \\
&= \tilde{\OO}\left( \Texp + \frac{T-\Texp}{\sqrt{\Texp}} + T^{2/3} \right) .
\end{align*}
Recall that the event $\mathcal{E} \cap \mathcal{F} \cap \mathcal{G} $ holds with probability at least $1-9\delta$. Combining these results, we prove the theorem. 
\end{proof}

\begin{corollary}
If the total interaction time $T$ is long enough, \textit{i.e.} having exploration duration of $\Texp = T^{2/3}$ satisfies that $\Texp > T_0$, then with high probability \Alg obtains regret upper bounded by $\tilde{\OO}\left(T^{2/3}\right)$. 
\end{corollary}
\begin{proof}
$T\geq T_0^{3/2}$  time steps with $\Texp = T^{2/3}$ is
Interacting with the system at least for $T\geq T_0^{3/2}$ allows us to set $\Texp = T^{2/3}$. Inserting this $\Texp$ to Theorem~\ref{Sup:totalregret},  we obtain $\reg(T-\Texp) = \tilde{\OO}\left( T^{2/3} + \frac{T-T^{2/3}}{\sqrt{T^{2/3}}} \right) + T^{2/3} = \tilde{\OO}\left( T^{2/3} + T^{2/3}-T^{1/3} \right) + T^{2/3} = \tilde{\OO}\left(T^{2/3}\right) $. Note that $\tilde{\OO}(\sqrt{\Texp})$ term of regret obtained in explore phase doesn't effect the total regret result with the given choice of $\Texp$.
\end{proof}

\paragraph{Relation to~\citet{mania2019certainty}}
Note that the result in Theorem~\ref{Sup:totalregret} is the first sublinear end-to-end regret result for $\LQG$ control in the literature. The closest related result to this is the sensitivity analysis of Riccati equation in~\citet{mania2019certainty}. In this work, the authors show that if the estimated system parameters are \textit{very close} to the underlying system parameters, the control designed by using the estimated model, results in a suboptimality gap which is quadratic with respect to the estimation error. It is worth noting that this result is for expected cost difference between optimal controller and designed controller, rather than the actual acquired cost. 

\section{Technical Lemmas and Theorems}
\label{Technical}
\begin{theorem}[Partial Random Circulant Matrices~\citep{krahmer2014suprema}]\label{circ thm}

Let $\mathbf{C}\in\R^{d\times d}$ be a circulant matrix where the first row is distributed as $\mathcal{N}(0,I)$. Given $s\geq 1$, set $m_0=c_0s\log^2(2s)\log^2(2d)$ for some absolute constant $c_0>0$. Pick an $m\times s$ submatrix $\mathbf{S}$ of $\mathbf{C}$. With probability at least $1-(2d)^{-\log(2d)\log^2(2s)}$, $\mathbf{S}$ satisfies
\[
\|\frac{1}{m}\mathbf{S}^\top \mathbf{S}-I\|\leq \max\{\sqrt{\frac{m_0}{m}},\frac{m_0}{m}\}.
\]
\end{theorem}

\begin{theorem}[Azuma's inequality] \label{azuma}
Assume that $(X_s; s\geq 0)$ is a supermartingale and $|X_s - X_{s-1}| \leq c_s$ almost surely. Then for all $t>0$ and all $\epsilon>0$,

\begin{equation*}
    \Pr(|X_t - X_0|\geq \epsilon) \leq 2 \exp\left(\frac{-\epsilon^2}{2\sum_{s=1}^t c_s^2}\right)
\end{equation*}
\end{theorem}

\begin{theorem}[Self-normalized bound for vector-valued martingales~\citep{abbasi2011improved}]
\label{selfnormalized}
Let $\left(\mathcal{F}_{t} ; k \geq\right.$
$0)$ be a filtration, $\left(m_{k} ; k \geq 0\right)$ be an $\mathbb{R}^{d}$-valued stochastic process adapted to $\left(\mathcal{F}_{k}\right),\left(\eta_{k} ; k \geq 1\right)$
be a real-valued martingale difference process adapted to $\left(\mathcal{F}_{k}\right) .$ Assume that $\eta_{k}$ is conditionally sub-Gaussian with constant $R$. Consider the martingale
\begin{equation*}
S_{t}=\sum_{k=1}^{t} \eta_{k} m_{k-1}
\end{equation*}
and the matrix-valued processes
\begin{equation*}
V_{t}=\sum_{k=1}^{t} m_{k-1} m_{k-1}^{\top}, \quad \overline{V}_{t}=V+V_{t}, \quad t \geq 0
\end{equation*}
Then for any $0<\delta<1$, with probability $1-\delta$
\begin{equation*}
\forall t \geq 0, \quad\left\|S_{t}\right\|^2_{V_{t}^{-1}} \leq 2 R^{2} \log \left(\frac{\operatorname{det}\left(\overline{V}_{t}\right)^{1 / 2} \operatorname{det}(V)^{-1 / 2}}{\delta}\right)
\end{equation*}

\end{theorem}

\begin{theorem}[Gordon's theorem for Gaussian matrices~\citep{vershynin2010introduction}]. \label{Gordon} Let $A$ be an $N \times n$ matrix
whose entries are independent standard normal random variables. Then
\begin{equation*}
\sqrt{N}-\sqrt{n} \leq \mathbb{E} \sigma_{\min }(A) \leq \mathbb{E} \sigma_{\max }(A) \leq \sqrt{N}+\sqrt{n}
\end{equation*}
\end{theorem}

\begin{lemma}[Sub-Gaussian Martingale Concentration~\citep{simchowitz2018learning}]\label{sig sub}
Let $\{\mathcal{F}_t\}_{t\geq 1}$ be a filtration, $\{Z_t,W_t\}_{t\geq 1}$ be real valued processes adapted to $\mathcal{F}_t,\mathcal{F}_{t+1}$ respectively (i.e.~$Z_t\in\mathcal{F}_t,W_t\in\mathcal{F}_{t+1}$). Suppose $W_t \big | \mathcal{F}_t$ is a $\sigma^2$-sub-Gaussian random variable with mean zero. Then
\[
\Pr(\{\sum_{t=1}^T Z_t W_t\geq \alpha\}\bigcap \{\sum_{t=1}^T Z_t^2\leq \beta\})\leq \exp(-\frac{\alpha^2}{2\sigma^2\beta})
\]
\end{lemma}
This lemma implies that $\sum_{t=1}^T Z_tW_t$ can essentially be treated as an inner product between a deterministic sequence $Z_t$ and an i.i.d.~sub-Gaussian sequence $W_t$.

\begin{lemma} [Covering bound~\citep{oymak2018non}]\label{cover bound} 
Given matrices $A\!\in\!\R^{n_1\!\times\! N},B\!\in\!\R^{N\times n_2}$, let $M\!=\!AB$. Let $\mathcal{C}_1$ be a $1/4$-cover of the unit sphere $\mathcal{S}^{n_1-1}$ and $\mathcal{C}_2$ be a $1/4$-cover of the unit sphere in the row space of $B$ (which is at most $\min\{N,n_2\}$ dimensional). Suppose for all $a\in\mathcal{C}_1,b\in\mathcal{C}_2$, we have that $a^\top M b \leq \gamma$. Then, $\|M\|\leq 2\gamma$.
\end{lemma}

\begin{lemma}[Norm of a sub-Gaussian vector \citep{abbasi2011regret}]\label{subgauss lemma}
Let $v\in \R^d$ be a entry-wise $R$-sub-Gaussian random variable. Then with probability $1-\delta$, $\|v\| \leq R\sqrt{2d\log(2d/\delta)}$.
\end{lemma}

\begin{lemma}[\citep{abbasi2011regret}]\label{basicprob}
Let $X_1, \ldots, X_t$ be random variables. Let $a \in \R$. Let $S_t = \sum_{s=1}^t X_s$ and $\tilde{S}_t = \sum_{s=1}^t \mathbbm{1}_{X_s \leq a}X_s$ where $\mathbbm{1}_{X_s \leq a}X_s$ denotes the truncated version of $X_s$. Then it holds that 
\begin{equation*}
    \Pr(S_t > x) \leq \Pr(\max_{1\leq s\leq t} X_s \geq a) + \Pr(\tilde{S}_t > x).
\end{equation*}
\end{lemma}

\begin{lemma}[Gaussian concentration Ineq. for Lipschitz func.~\citep{ledoux2013probability}] \label{gausss_lip}

Let $f(\cdot) : \mathbb{R}^{p} \rightarrow \mathbb{R}$ be an $L-$
Lipschitz function and $\mathbf{g} \sim \mathcal{N}\left(0, \mathbf{I}_{p}\right) .$ Then,
\begin{equation*}
\mathbb{P}(|f(\mathbf{g})-\mathbb{E}[f(\mathbf{g})]| \geq t) \leq 2 \exp \left(-\frac{t^{2}}{2 L^{2}}\right)
\end{equation*}
\end{lemma}

\begin{lemma}[\citep{mania2019certainty}] \label{normwoodbury} Let $M$ and $N$ be two positive semidefinite matrices of the same dimension. Then, $\|N(I + MN)^{-1}\| \leq \|N\|$. 

\end{lemma}

\end{document}